\documentclass[a4paper, 12pt]{article}
\usepackage{amsmath, amssymb, amsfonts}  % Math stuff
\usepackage{cases}
\usepackage{fullpage}
\usepackage{graphicx}    

\usepackage{bm}
\usepackage{bbm}
\usepackage{dsfont}
\usepackage{xcolor}
\usepackage{enumerate}
\usepackage{booktabs}
\usepackage[mathscr]{eucal}
\usepackage{subfigure}
\usepackage{caption}
\usepackage{subcaption}
\usepackage{theorem}
\usepackage{longtable}
\usepackage{pifont}  % for \ding{51} (✓) and \ding{55} (✗)

\usepackage{amsmath, amssymb, amsfonts}  % Math stuff

\usepackage[mathscr]{eucal}
\usepackage{subfigure}
\usepackage{theorem}

\usepackage{algorithm}

\newtheorem{theorem}{Theorem}

\newtheorem{proposition}{Proposition}
\newtheorem{lemma}{Lemma}

\newtheorem{assumption}{Assumption}

{\theorembodyfont{\upshape}
\newtheorem{remark}{Remark}

}

\def\qed{\hfill \vrule height 5pt width 5pt depth 0pt \medskip}
\newcommand{\proof}{\noindent {\bf Proof. }}

\def\diag{{\rm diag}}
\def\sign{{\rm sgn}}

\def\rank{{\rm rank}}

\newcommand{\bp}{\bm{p}}
\newcommand{\bq}{\bm{q}}
\newcommand{\br}{\bm{r}}
\newcommand{\bs}{\bm{s}}
\newcommand{\bu}{\bm{u}}
\newcommand{\bv}{\bm{v}}
\newcommand{\bx}{\bm{x}}
\newcommand{\by}{\bm{y}}
\newcommand{\bw}{\bm{w}}
\newcommand{\bz}{\bm{z}}

\newcommand{\bK}{{F_K}}
\newcommand{\bQ}{{F_Q}}
\newcommand{\bV}{{F_V}}

\newcommand{\beq}{\begin{equation}}
\newcommand{\eeq}{\end{equation}}
\newcommand{\beqa}{\begin{eqnarray}}
\newcommand{\eeqa}{\end{eqnarray}}
\newcommand{\beqan}{\begin{eqnarray*}}
\newcommand{\eeqan}{\end{eqnarray*}}

\newcommand{\pde}[2]{ \frac{\partial #1}{\partial #2} }

\newcommand{\bite}{\begin{itemize}}
\newcommand{\eite}{\end{itemize}}
\newcommand{\benu}{\begin{enumerate}}
\newcommand{\eenu}{\end{enumerate}}

\definecolor{darkpastelgreen}{rgb}{0.01, 0.75, 0.24}

\begin{document}

%  \title{\LARGE \bf Transformer's Self Attention as Multiagent Dynamics on the Sphere}
\title{\LARGE \bf Analogies between Transformer Layers and Power Method}
%\title{\LARGE \bf Self Attention Dynamics on the Sphere: stability analysis}

\author{
Chenglong Li and Claudio Altafini
% \thanks{
% C. Altafini is with the Division of Automatic Control, Dept. of Electrical Engineering,
% Link\"oping University, SE-58183, Link\"oping, Sweden.
% email: {\tt\small claudio.altafini@liu.se}.
% Work supported in part by the Swedish Research Council (grant n. 2024-04772)
% and by the ELLIIT framework program at Link\"oping University.
% }
}
%% 

% $\mathbbm{1} $ $ \mathds{1}$

\maketitle
%\thispagestyle{empty}
%\pagestyle{empty}
%

%%%%%%%%%%%%%%%%%%%%%%%%%%%%%%%%%%%%%%%%%%%%%%%%%%%%%%%%%%%%%%%%%%%%%%%%%%%%%%%%
\begin{abstract}
In the paper we show that there is an analogy between the operations occurring in a layer of a transformer (projections and layer normalizations, disregarding the feedforward neural network) and a step in the power method. Coherently with this analogy, we show that passing through a layer the tokens tend to be tilted towards the principal eigenvector of a matrix which is the product of the output and value weight matrices of that layer. 
In the special case of a transformer with shared weights (i.e., in which all layers  have identical weights) then the alignment with this principal eigenvector is particularly evident empirically, and can also be shown analytically. 
The analogy also suggests a method to steer the output of the transformer towards an arbitrary desired direction in token space.
\end{abstract}

\section{Introduction}

The power method is a simple and efficient numerical method to compute the principal eigenvector of a matrix, say $F$. It works by iteratively multiplying a vector $ \bx(t)$ with $F$ and then normalizing the resulting new vector: $ \bx(t+1) = \frac{F\bx(t)}{\| F\bx (t) \|} $, with $ t$ a step counter (i.e., time). If $F$ has a strictly dominant real eigenvalue, say $ \lambda_1 $, then the process initialized in a random vector generically converges to the associated eigenvector, say  $ \bv_1$, i.e.,  $ \bx(t) \to \bv_1 $ as the number of steps $ t$ grows. 

In this paper we show that the operations occurring in a transformer have analogies with the power method, provided we think of a transformer as the unfolding in time of a discrete-time dynamical system in which time is the layer index. 
Such interpretation is by now well-known \cite{dutta2021redesigning,sander2022sinkformers,lu2019understanding} and has been exploited in several papers \cite{abellaconsensus,geshkovski2023emergence,geshkovski2024dynamic,geshkovski2025mathematical,wu2024role}.
Most of the literature viewing a transformer as a dynamical system focuses on the self-attention mechanism only, viewed as a continuous-time time-invariant system, often limited to query, key and value matrices having a special form, and even more often dealing with a single head. 
The asymptotic behavior of such dynamical system is investigated for instance in \cite{abella2024asymptotic,abellaconsensus,altafini2025multistability,geshkovski2025mathematical}, where it is shown that the generic behavior of a transformer (in the infinite depth limit) is to converge to a consensus point, in which all tokens become equal. 
This phenomenon is predicted by the mathematical models and it is also known experimentally under the names of  oversmoothing, rank collapse, or token-uniformity \cite{dong2021attention,dovonon2024setting,noci2022signal,nguyen2023mitigating,scholkemper2024residual,shi2022revisiting,zhai2023stabilizing}.

If consensus (i.e., token oversmoothing in depth) has been confirmed by several empirical and theoretical studies, much less is known regarding what is the vector to which all tokens converge to, and what interpretation does this vector have in practical transformers. 
In was shown in \cite{altafini2025multistability} that for the continuous-time single-head self-attention model, all eigenvectors of the value matrix are equilibria of the system, and that most frequently convergence occurs towards the principal eigenvector of such matrix. 

In this paper these properties are investigated in discrete-time and for multihead systems. We obtain complete analytical characterizations of the convergence to the principal eigenvector for special cases (full attention, symmetric value matrices) and numerical results for more general cases (causal attention, non-symmetric weight matrices). 
More importantly, we show that for off-the-shelf transformers like GPT2 \cite{gpt2}, GPTNeo \cite{gptneo}, BERT \cite{bert} and ALBERT (``A Lite BERT'', a variant of BERT with shared weights across all its layers \cite{lan2019albert}), the theoretical results have predictive power and the aforementioned analogy with the power method makes sense, provided we use the $L_2 $ norm and disregard the feedforward neural network (FFN) downstream the attention block in each layer. 

In particular, a transformer like ALBERT, in which the query, key, value and output weight matrices are identical in all layers, corresponds to a time-invariant discrete-time dynamical system like the one we study analytically, and, in fact, we confirm empirically that (if we use the $L_2 $ norm and disregard the FFN) the consensus value is the one predicted by the model: all tokens converge to the principal eigenvector $ \bv_1 $ associated to a ``layer matrix'' $F$ corresponding to the product of value and output weight matrices. 
For such ``universal transformer'', each layer can be seen as a ``perturbed'' projection of the tokens on the eigenspaces of a matrix like $F$, followed by a layer normalization, where the perturbation is given by the attention matrix. 
When the tokens converge to consensus, then a full attention matrix becomes the uniform matrix, and the aforementioned perturbation vanishes as the tokens are mapped through the successive identical layers of the transformer. 
The projection action along $ \bv_1 $ becomes more and more evident as the number of layers increases. 
For such class of transformers, the analogy with the power method is particularly evident.
% This will be shown below for the ALBERT transformer.

When the weight matrices differ from layer to layer, as in most commonly used transformers, then the dynamical system becomes time-varying (or layer-varying) and the projection effect is less evident, but the basic principle remains. Each layer can still be seen as one step of the power method, but since the associated matrix $F$ changes from layer to layer, i.e., $ F=F_t $ with $ t$ the layer index, perfect alignment along the corresponding dominant eigenvector $ \bv_{1,t}$ is not achieved, because also $ \bv_{1,t}$ changes from layer to layer. 
Nevertheless what can be shown is that if we look at statistics over many tokens (in order to factor out the ``perturbation'' due to the attention coefficients) and disregard the FFN, then with the $L_2$ normalization at each layer the projection of the tokens along the current $ \bv_{1,t}$ typically increases as the tokens pass through the $ t$-th layer, i.e., in probability the correlation between tokens and $ \bv_{1,t}$ before the $t$-th layer is typically lower than the same correlation after the layer.
At each layer, the matrix $ F_t$ for which this occurs is a product between the output weight matrix and the concatenation of the value weight matrices of the various heads. 
This increase in alignment is precisely the spirit of the power method: the dominant direction is amplified with respect to all other directions. 
We show that this is occurring in nearly all layers of real transformers like GPT2, GPTNeo and BERT, even though the alignment with this dominant direction remains typically weak, because the layer matrices of these transformers have very small spectral gaps between principal and second eigenvalue.

The insight we acquire by the analogy with the power method suggests however a simple method for obtaining a very strong alignment with a desired eigenspace, and thereby for steering the tokens towards an arbitrary direction in token space. Given a transformer with layer-varying weights, it is enough to modify the last layer by adding a properly chosen rank-1 perturbation to $ F_t $. If this perturbation is sufficiently large, it creates a spectral gap in the modified layer matrix, and the tokens align themselves with the new principal eigenvector induced by this rank-1 perturbation. This eigenvector can be chosen by the designer, meaning that the tokens can be steered essentially in any arbitrary direction in the space of token vector representations.

Steering the output of a large language model toward a desired token alignment can be done at various levels: by intervening on the prompts \cite{shin2020autoprompt,NEURIPS2023_91edff07,schulhoff2024prompt}, at the embedding level (e.g. ``soft-prompting'' \cite{lester2021powerscaleparameterefficientprompt}), at the activation level \cite{dathathri2019plug,turner2023steering} or on the weights. The latter approach, which is closest in spirit to our work, can be attained e.g. by fine-tuning, weight editing, or iterative refinement via self-feedback \cite{hu2021loralowrankadaptationlarge,ilharco2022editing,zelikman2022starbootstrappingreasoningreasoning}. None of these methods, however, specifies how to mechanistically modify the layers of a transformer in order to achieve a desired alignment, as we do here. Our methodology suggests an alternative principled way to steer the output of a large language model.

\section{Model formulation}
\label{sec:modelformulation}
Following e.g. \cite{abella2024asymptotic,geshkovski2025mathematical}, we aim at representing a transformer as a discrete-time dynamical system with time given by the layer index $ t= 1, \ldots, T$, where $ T$ is the number of layers. 
Using a column vector representation for the tokens, $ \bx_i (t) \in \mathbb{R}^d$, $ i=1, \ldots, n $, the update law for each token on a layer of a multihead transformer, including the skip connection and layer normalization, and excluding the FFN can be compactly written as:
\begin{equation}
\label{eq:update-multihead1}
\bx_i(t+1) = \text{LayerNorm}\left( \bx_i (t) +F_{O,t} \sum_{h=1}^H \bar{F}_{V,t}^{(h)}  \sum_{j=1}^n A_{t, ij}^{(h)}(\bx(t)) \bx_j (t) \right), \quad \begin{array}{l} i=1, \ldots, n, \\ t=1, \ldots, T,\end{array}
\end{equation}
where $ A_t^{(h)}(\bx(t))$ is the attention matrix of each head $ h = 1, \ldots, H$, (a function of all tokens $ \bx(t)$ and of the query and key weight matrices) and the matrices $ F_{O,t} $ and $  \bar{F}_{V,t}^{(h)} $ are obtained from the learned output weight matrices and value weight matrices of each head.
In particular, if $ W_{O,t} $ and $ W_{V,t}^{(h)}$ are the learned output and value weight matrices in the standard transformer formulation \cite{vaswani2017attention}, then 
\[
F_{O,t} = (W_{O,t})^T, \qquad 
 \bar{F}_{V,t}^{(1)} = \begin{bmatrix} (W_{V,t}^{(1)})^T \\ 0 \\ \vdots \\ 0  \end{bmatrix}, \ldots,  \bar{F}_{V,t}^{(H)} = \begin{bmatrix} 0 \\ 0 \\ \vdots \\ (W_{V,t}^{(H)})^T   \end{bmatrix},
 \]
where the transpose originates from the column-token convention we use. See a detailed formulation in Section~\ref{sec:standardW} of the Supplementary Materials.

If we disregard attention (which varies from token to token) the matrix that lumps together the contribution of the $H$ heads occurring in the transformer layer is the following ``layer matrix''
\[
F_t = F_{O,t} F_{V,t} \quad \text{with} \quad  F_{V,t} =  \sum_{h=1}^H \bar{F}_{V,t}^{(h)} .
\] 
Let $ \lambda_{1,t}, \ldots, \lambda_{d,t} $ be the eigenvalues of $ F_t $, ordered as $ {\rm Re}(\lambda_{1,t}) \geq \ldots\geq  {\rm Re}(\lambda_{d,t}) $, and let $ \bv_{1,t}, \ldots, \bv_{d,t} $ be the associated eigenvectors, normalized s.t. $ \| \bv_{j,t}\|=1$.
 Then $ \bv_{1,t}  $ is the principal eigenvector of $ F_t$ (i.e., the right eigenvector associated to $ \lambda_{1,t}$). 

The dynamical system \eqref{eq:update-multihead1} is time-varying, i.e., $ F_t $ changes from layer to layer. 
When we disregard the FFN part as we do in \eqref{eq:update-multihead1}, it is known from the literature that the model \eqref{eq:update-multihead1} asymptotically (i.e., when $ T\to \infty$) tends to a consensus equilibrium point, in which all tokens are equal \cite{abella2024asymptotic,altafini2025multistability,geshkovski2025mathematical}.

\subsection{Shared weight transformers: Alignment with principal eigenvector}
\label{sec:shared-param}

If in a transformer all weight matrices are identical across layers, i.e., in \eqref{eq:update-multihead1} $ \forall \, t_1, t_2 \in \{ 1, \ldots, T\}$
\beq
 \bar{F}_{V,t_1}^{(h)} =\bar{F}_{V,t_2}^{(h)} , \quad  h\in \{1, \ldots, H \} , \quad \text{and} \;\; F_{O,t_1} =F_{O,t_2} ,
\label{eq:weight-shared1}
\eeq
then we can drop the index $ t$ in all the matrices in \eqref{eq:weight-shared1} and  the dynamical system~\eqref{eq:update-multihead1} becomes autonomous (i.e., time-invariant).

Consider $ F= F_O F_V= F_O \sum_{h=1}^H \bar{F}_{V}^{(h)}  $ and  denote $ \lambda_1, \ldots, \lambda_d $ its eigenvalues, ordered as $ {\rm Re}(\lambda_1) \geq \ldots\geq  {\rm Re}(\lambda_d) $, with $ \bv_1, \ldots, \bv_d $ the associated eigenvectors.
Several cases are possible, depending on the principal eigenvalue $ \lambda_1 $:
\begin{description}
\item[Case 1:] $ \lambda_1 $ is real and strictly dominant: $ \lambda_1 >  {\rm Re}(\lambda_i)  $, $ i=2, \ldots, d$;
\item[Case 2:]  $ \lambda_1 $ is complex and strictly dominant: $ {\rm Re}(\lambda_1) >  {\rm Re}(\lambda_i)  $, $ i=3, \ldots, d$;
\item[Case 3:]  $ \lambda_1 $ (real or complex) is not strictly dominant.
\end{description}

\subsubsection{Case 1: real strictly dominant principal eigenvalue}

For the case of $ \text{LayerNorm}(\cdot) $ which is the $ L_2 $ norm, the system~\eqref{eq:update-multihead1} can be reformulated as 
\begin{equation}
\label{eq:update-multihead-auton}
\bx_i(t+1) = \frac{\bx_i (t) +\eta F_{O} \sum_{h=1}^H \bar{F}_{V}^{(h)}  \sum_j A_{ij}^{(h)}(\bx(t)) \bx_j (t)  }{\| \bx_i (t) +\eta F_{O} \sum_{h=1}^H \bar{F}_{V}^{(h)}  \sum_j A_{ ij}^{(h)} (\bx(t)) \bx_j  (t) \|}, \qquad \begin{array}{l} i=1, \ldots, n, \\ t=1, \ldots, T\end{array}
\end{equation}
where we have added an extra parameter $ \eta$ (``step size''), which balances the importance of the self-attention part against the skip connection \cite{noci2022signal}.  
Under the special case of $ F $ which is symmetric, the asymptotic behavior of the system~\eqref{eq:update-multihead-auton} is described in the following theorem.

\begin{theorem}
\label{thm:symm-case1}
Consider the system~\eqref{eq:update-multihead-auton}. 
Assume the attention matrices $ A^{(h)}$, $ h=1, \ldots, H$, are full and the matrix $ F = F_O \sum_{h=1}^H \bar{F}_{V}^{(h)} $ is symmetric with $ \lambda_1 > \lambda_i$, $ i=2,\ldots, d$, $ \lambda_1 >0$. Assume further that  $ \eta<  \frac{2}{ |\lambda_1 + \lambda_d |}  $ whenever $ - \lambda_d > \lambda_1 $. Then the consensus point $ \bx_i =\bv_1 $ $\forall \, i=1, \ldots, n$ (or $ \bx_i =-\bv_1 $ $\forall \, i=1, \ldots, n$) is a locally asymptotically stable equilibrium point for the system~\eqref{eq:update-multihead-auton} when $ T\to \infty$. 
\end{theorem}
The proof of this theorem is in Section~\ref{sec:ext-multi-head} in the Supplementary Materials. It is a consequence of Theorem~\ref{thm:stab-main-self} (also in the Supplementary Materials). 

What the theorem says is that when sufficiently many identical layers are available (i.e., sufficiently many time steps in the dynamical system~\eqref{eq:update-multihead-auton}), then all tokens align themselves with the principal eigenvalue of $ F$, $ \bv_1 $. 
This is similar to the behavior occurring in the power method. 
The reason for this form of convergence is that when the token $ \bx_i $ approaches a consensus point, each (full) attention matrix approaches a uniform matrix, $ A_{ij}^{(h)}(\bx) \to \frac{1}{n}$, and can be taken out of the summation in \eqref{eq:update-multihead-auton} \cite{noci2022signal}. What is left is a special case of a dynamical system which in Section~\ref{sec:single-head-analysis} of the Supplementary Materials we call the discrete-time multigent Oja flow \cite{oja1985stochastic}, whose continuous-time single-agent version is well-known for discovering the principal eigenvector of a matrix \cite{yoshizawa2001convergence}, just like the power method does in discrete-time. The matrix in question is our $F$.

Roughly speaking, the behavior of \eqref{eq:update-multihead-auton} can be described as follows. If we split a token $\bx_i(t) $ into components associated to the eigenvectors of $F$ (the spectral decomposition can be done rigorously only if $F$ has real eigenvalues, but let us not focus on this detail here), then at each step of \eqref{eq:update-multihead-auton} (disregarding the residual connection and the attention coefficients), the component along $ \bv_j$ acquires a scalar amplification/attenuation factor equal to the corresponding $ \lambda_j$. Hence if $ \lambda_1 > \lambda_j$, $j=2, \ldots, n$, the component along $ \bv_1 $ gets amplified the most in comparison with the other components. 
Since this projection is followed by a normalization, the resulting normalized vector $ \bx_i(t+1) $ becomes more aligned with $ \bv_1$ than $ \bx_i(t)$, because the projection dominates with respect to the remaining eigenvectors of $ F$. 
The projection + normalization mechanism is what occurs in a power method.

In the single head case discussed in the Supplementary Materials, other consensus-like equilibria, denoted bipartite consensus points, could in principle be locally asymptotically stable, corresponding to the remaining eigenvectors of the matrix $F$, $ \bv_2, \ldots, \bv_d$, see Lemma~\ref{lem:self-equil1} in the Supplementary Materials. The name bipartite consensus is due to the splitting of the $n$ tokens into two groups, the first group converging to $ \bv_k$, the other to the antipodal point $ - \bv_k$, with $ k=2, \ldots, d$ \cite{altafini2025multistability}. Also more exotic types of multi-partite clustered equilibria are theoretically possible for the single-head model, see Lemma~\ref{lem:self-equil1}. 
In the simulations we carried out for the multi-head model \eqref{eq:update-multihead-auton} we could not observe any of these extra locally asymptotically stable equilibria (see Section~\ref{sec:extra-examples} in the Supplementary Materials), nor were they visible in the experiments with the ALBERT transformer described below. 

While we could prove convergence to $ \bv_1 $ only locally, in the full attention case and for a symmetric $ F$ (and for $ \eta $ small enough), convergence seems to occur under more general conditions (for causal attention, for any $F$ with a strictly dominant eigenvalue, for any $ \eta $, and for nearly any initial token vector). 
We show below that numerically this is the case for the ALBERT transformer.

\subsubsection{Case 2: complex conjugate principal eigenvector}
When the principal eigenvalue is $ \lambda_{1,2} = \alpha \pm i \beta $, and $ {\rm Re}(\lambda_{1,2}) >  {\rm Re}(\lambda_i)  $, $ i=3, \ldots, d$, then the power method fails, because the vector $ \bx $ keeps rotating in the 2D subspace generated by the complex conjugate eigenvector pair $ \bv_1 $ and $ \bv_2 $. The same behavior occurs in \eqref{eq:update-multihead-auton} when $ T \to \infty $: the $ \bx_i $ keep rotating in the 2D subspace formed by $ {\rm Re} (\bv_1 ) $ and $ {\rm Im}(\bv_1) $. 
Denoting $ P = [{\rm Re} (\bv_1 ) \; {\rm Im}(\bv_1) ]$, the projection onto $ {\rm span}(P) $ is  $ \bx_{i,{\rm proj}} = P P^\dagger \bx_i $. 
The quantity that we need to monitor is then the correlation between $ \bx_i $ and its projection $ \bx_{i,{\rm proj}}$. 
This situation actually happens for some of the variants of ALBERT analyzed below.

\subsubsection{Case 3: no strictly dominant eigenvalue}
This case, due to $ \lambda_1 $ having multiplicity $ > 1 $, is not generic, hence it almost never occurs in real learned weight matrices. However, what can happen (and indeed we encounter in some variants of the ALBERT transformer) is $ \lambda_2 \simeq \lambda_1 $. The consequence is that convergence of the iterate \eqref{eq:update-multihead-auton} may occur very slowly for some initial conditions $ \bx_i(0)$, as the two eigenspaces $ {\rm span} (\bv_1)$ and $ {\rm span} (\bv_2)$ are difficult to separate (a large $ T $ is required). 
The quantity to monitor in this case is the correlation between $ \bx_i $ and its projection $ \bx_{i,{\rm proj}}= P P^\dagger \bx_i $ with $ P = [ \bv_1 \; \bv_2]$. 

\subsection{Transformers with layer-varying weights: Alignment with the principal eigenvector at the current step}
\label{sec:varying-param}
In nearly all transformers, the (query, key, value and output) weight matrices change from layer to layer. 
In these cases, we need to consider the time-varying dynamical system \eqref{eq:update-multihead1} for which the ``layer matrix'' $ F_t $ changes from layer to layer. 
Since also the associated principal eigenvalue and eigenvector pair $ \lambda_{1,t} $ and $ \bv_{1,t}$ change from layer to layer, we cannot expect an incrementally growing alignment of the tokens with $ \bv_{1,t}$ as $ t$ grows. 
Nevertheless, we can check if at a step $ t$ the tokens still tend to align themselves with the direction of $ \bv_{1,t}$ when they pass through the $t$-th layer of the transformer. 
Because of the attention coefficients, we cannot expect this to happen to all tokens at all layers, but for each layer $t$ we can compute the average correlation increment 
\beq
\psi(\bv_{1,t}) = \mathbb{E}\left( \left|\mathrm{corr} (\bx_i{(t+1)},\bv_{1,t})\right|-\left|\mathrm{corr}(\bx_i{(t)},\bv_{1,t})\right| \right) 
\label{eq:corr-increm1}
\eeq
and the probability of increment 
\beq
\rho(\bv_{1,t})= \mathbb{P}\!\left(\left|\mathrm{corr}(\bx_i{(t+1)},\bv_{1,t})\right|>\left|\mathrm{corr}(\bx_i{(t)},\bv_{1,t})\right|\right)
\label{eq:prob-increm1}
\eeq
over a sufficiently large number of token vectors.
Values $\rho(\bv_{1,t}) > 0.5$ and $\psi (\bv_{1,t})>0$ imply that the tokens tend to become more aligned with $\bv_{1,t} $ when passing through layer $t$, indicating that the layer behaves analogously to a single step of the power method applied to $F_t$.

When the principal eigenvalue is complex as in Case~2 (and possibly in Case~3), then instead of \eqref{eq:corr-increm1} and \eqref{eq:prob-increm1} for each token $ \bx_i$ we must compute
\begin{align*}
\psi(\bx_{i,{\rm proj}}) & = \mathbb{E}\left( \left| \mathrm{corr}(\bx_i{(t+1)},\bx_{i,{\rm proj}}(t+1))\right|-\left|\mathrm{corr}(\bx_i{(t)},\bx_{i,{\rm proj}} (t))\right| \right) 
%\label{eq:corr-increm2} 
\\
\rho(\bx_{i,{\rm proj}}) & = \mathbb{P}\!\left(\left|\mathrm{corr}(\bx_i{(t+1)},\bx_{i,{\rm proj}}(t+1))\right|>\left|\mathrm{corr}(\bx_i{(t)},\bx_{i,{\rm proj}}(t))\right|\right)
%\label{eq:prob-increm2}
\end{align*}
Hence we denote
\beq
\psi_t = \begin{cases} 
\psi(\bv_{1,t}) & \text{if $ \bv_{1,t} $ real} \\
\psi(\bx_{i,{\rm proj}}) & \text{if $ \bv_{1,t} $ complex} 
\end{cases}
\quad  \text{and} \quad 
\rho_t = \begin{cases} 
\rho(\bv_{1,t}) & \text{if $ \bv_{1,t} $ real} \\
\rho(\bx_{i,{\rm proj}}) & \text{if $ \bv_{1,t} $ complex.} 
\end{cases}
\label{eq:psi_t}
\eeq
We also denote 
\beq
\label{eq:gamma_t}
\gamma_t = \begin{cases}
\mathbb{E}\left( \left| \mathrm{corr}(\bx_i{(t+1)}, \bv_{1,t}) \right|  \right) & \text{ if $ \bv_{1,t} $ real} \\
\mathbb{E}\left( \left| \mathrm{corr}(\bx_i{(t+1)},\bx_{i,{\rm proj}}(t+1))\right| \right)  & \text{if $ \bv_{1,t} $ complex} 
\end{cases}
\eeq
the average correlation between the tokens and the principal eigenvector at the output of each layer. 

\subsection{Imposing an alignment in a transformer with layer-varying weights}
\label{sec:steering-token}

For transformers with layer-varying weights, the alignment effect with the principal eigenvector we obtain (i.e., $ \gamma_t $) is often rather limited. 
Furthermore also the total alignment with the principal eigenvector of the overall layer matrix product $ \Pi_{t=1}^T F_t $ is minimal, if not insignificant (see experiments below).

It is however possible to steer the tokens towards a desired alignment by modifying only the last layer $ t= T$ of the transformer. 
Recall that the layer matrix $ F_T $ is given by the product $ F_T =F_{O,T} F_{V,T}$ with $ F_{V,T} =\sum_{h=1}^H \bar{F}_{V,T}^{(h)}$ the concatenation of the value matrices of the $H$ heads.
Assume that $ F_{V,T} $ is invertible (generically this is true; if not, pseudoinversion or least-squares regression can be used instead). Then replacing the output matrix $ F_{O,T}$ with $ F_{O,T, {\rm mod}} = F_{O,T} + \sigma \bw_r \bw_\ell^T F_{V,T}^{-1}$ with $ \sigma \gg |\lambda_{1,T} | $ ($ \lambda_{1,T}$ is the principal eigenvalue of the layer matrix $ F_{T}$), the new layer matrix is $ F_{T,{\rm mod}} = F_T + \sigma \bw_r \bw_\ell^T$ where the rank-1 addition can be designed at will. 
If the spectral gap between $ \sigma $ and $ \lambda_{1,T} $ is sufficiently large, then the rank-1 term $ \sigma \bw_r \bw_\ell^T$ determines the new principal eigenpair of $ F_{T,{\rm mod}}$: the principal eigenvalue is (approximately) $ \sigma$ and the corresponding principal eigenvector (approximately) $ \bw_r$.
When in $ F_{T,{\rm mod}}$ the spectral gap between $ \lambda_{1,T,{\rm mod}}\approx \sigma $ and $ \lambda_{2,T,{\rm mod}}$ is of at least an order of magnitude, this novel principal direction dominates the projection and leads to a very high correlation between $ \bv_{1,T,{\rm mod}} (\approx \bw_r )$ and the tokens.
Since $ \bw_r $ can be chosen freely, the rank-1 modification provides a way to steer the tokens towards a desired direction.

Notice that if $ \bw_r $ and $\bw_\ell $ correspond to right and left eigenvector of an existing eigenvalue of $F_T $, say $ \lambda_j$, then the spectrum of $ F_{T,mod}$ is unchanged except for $ \lambda_{T,j} $ which becomes $ \lambda_{T,j} + \sigma $.

\section{Experiments}
In this section, we perform experiments with both layer-invariant (ALBERT) and layer-varying (BERT, GPT2, and GPTNeo) real transformers. See Supplementary Materials~\ref{sec:supply_experimental_details} for details.

\subsection{Application to ALBERT}
ALBERT \cite{lan2019albert} is a transformer in which all weights are identical across layers, i.e., \eqref{eq:weight-shared1} holds. The 4 versions of ALBERT we consider in this study are described in Table~\ref{tab:model_comparison} in the Supplementary Materials.
If we disregard the FFN part of a layer and use $ L_2 $ normalization, then the dynamical model \eqref{eq:update-multihead-auton} captures the transformation of the tokens through the layers. 
Each of the four transformers quickly converges to consensus, see Fig.~\ref{fig:Albert2} in the Supplementary Materials. 

When we apply the considerations in the previous sections to these 4 ALBERT transformers, we notice the following 
\bite
\item ALBERT-base: The distribution of the eigenvalues of $ F$ is shown in Fig.~\ref{fig:Albert1}(a). The principal eigenvalue $ \lambda_1  = 3.8624,$ is real, positive, simple and s.t. $ \lambda_1 > |\lambda_i| $ $ \forall \, i$, hence technically we are in Case~1.
However, $ \lambda_2 =  3.7738 $ is very close in magnitude to $ \lambda_1 $, and furthermore, the correlation between the two associated eigenvectors is extremely high: $ {\rm corr} (\bv_1, \bv_2 ) =0.9799$, meaning that the two eigenspaces $ {\rm span}\{\bv_1 \} $ and $ {\rm span}\{\bv_2 \} $ are essentially collinear. This means that {\em de facto} we are in Case~3. In fact, on the time horizon given by the number of layers $ T=12$, we often see that $ \bx_i(T) \in {\rm span} \{\bv_1, \bv_2 \} $ rather than being strictly aligned with $ \bv_1 $. 
%Consensus is always achieved in this time horizon, see Fig.~XX.
The alignment of $ \bx_i (T)$ with $ \bv_1 $ over $ 10^4$ different initial conditions $ \bx_i(0)$ is shown in Fig~\ref{fig:Albert1}(b). Both alignment and anti-alignment (i.e., alignment with $ -\bv_1$) seem to be overrepresented in the distribution, but not by much. If instead we look at $ \bx_{i,{\rm proj}} = P P^\dagger \bx_i $ where $ P = [ \bv_1 \, \bv_2]$, then a high correlation between $ \bx_i $ and this 2D projection $ \bx_{i,{\rm proj}} $ is clearly dominating the histogram. If instead of stopping at $ T=12$ we artificially prolong the sequence of steps in \eqref{eq:update-multihead-auton} to $ T=60$, then alignment with $ \bv_1 $ is typically obtained, see Fig.~\ref{fig:Albert4} in the Supplementary Materials for one example. 

\vspace{1cm}

% \begin{figure*}[ht]
%      \centering    
%       \subfigure[]{
%         \includegraphics[trim=0.1cm 0cm 0cm 0cm, clip=true,width=0.31\textwidth]{figures/fig_albert_base_v2_eigF.eps}} 
%       \subfigure[]{
%       \begin{minipage}{0.66\textwidth}
%       %\vspace{-3.5cm} 
%         \includegraphics[trim=0.5cm 6cm 1.5cm 0cm, clip=true,width=0.5\textwidth]{figures/fig_albert_base_v2_corr3_multiple.eps}
%         \includegraphics[trim=0.5cm 6cm 1.5cm 0cm, clip=true,width=0.5\textwidth]{figures/fig_albert_large_v2_corr3_multiple.eps}\\
%         \includegraphics[trim=0.5cm 6cm 1.5cm 0cm, clip=true,width=0.5\textwidth]{figures/fig_albert_xlarge_v2_corr3_multiple.eps}
%         \includegraphics[trim=0.5cm 6cm 1.5cm 0cm, clip=true,width=0.5\textwidth]{figures/fig_albert_xxlarge_v2_corr3_multiple.eps}
%         \end{minipage}}
%         \caption{Eigenvalues and token alignment for the four ALBERT transformers. (a): The histograms show the distribution of the real part of the eigenvalues of $F$ in each of the four ALBERT transformers. In the first two $ \lambda_1 $ is real, in the remaining two $ \lambda_1 $ is complex. (b): The histograms show the alignment of the tokens at the end of a transformer, $ \bx_i(T)$, with the principal eigenvector $ \bv_1 $ and with the 2D subspace determined by $ \bv_1 $ and $ \bv_2 $ for $ \lambda_1 $ real, and by real and imaginary part of $ \bv_1$ when $ \lambda_1 $ is complex. }
%         \label{fig:Albert1}
% \end{figure*}
\begin{figure*}[ht]
     \centering    
      \subfigure[]{
        \includegraphics[trim=0.1cm 0cm 0cm 0cm, clip=true,width=0.31\textwidth]{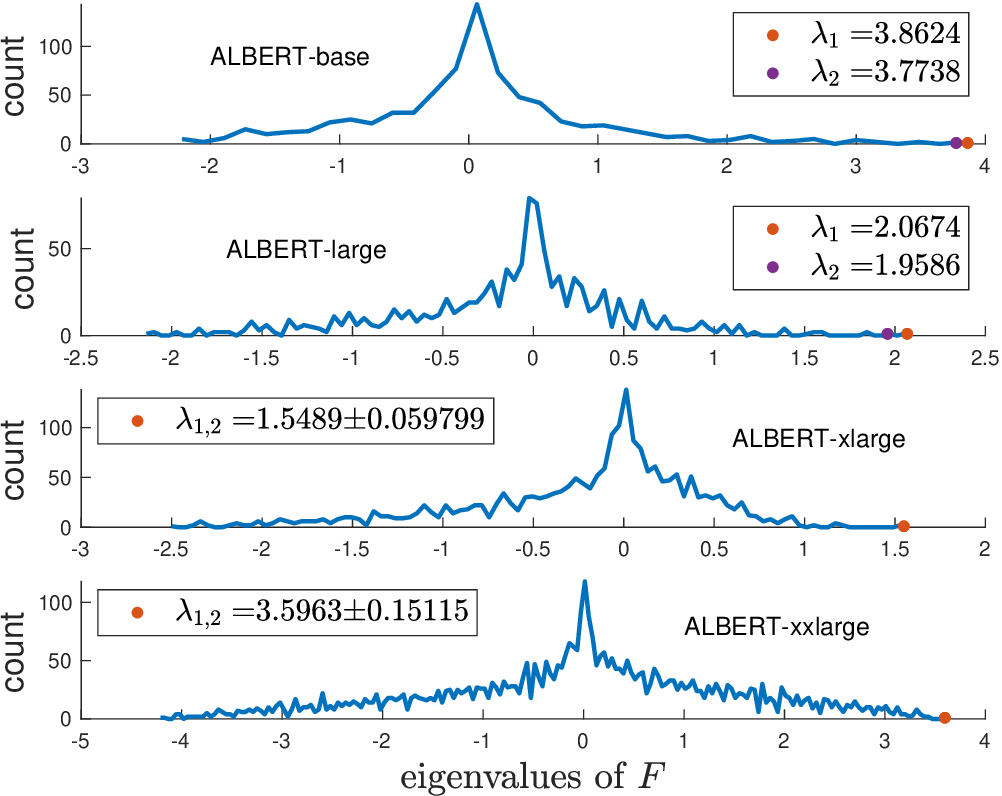}} 
      \subfigure[]{
      \begin{minipage}[b]{0.66\textwidth}
        \includegraphics[trim=0.5cm 6cm 1.5cm 0cm, clip=true,width=0.48\textwidth]{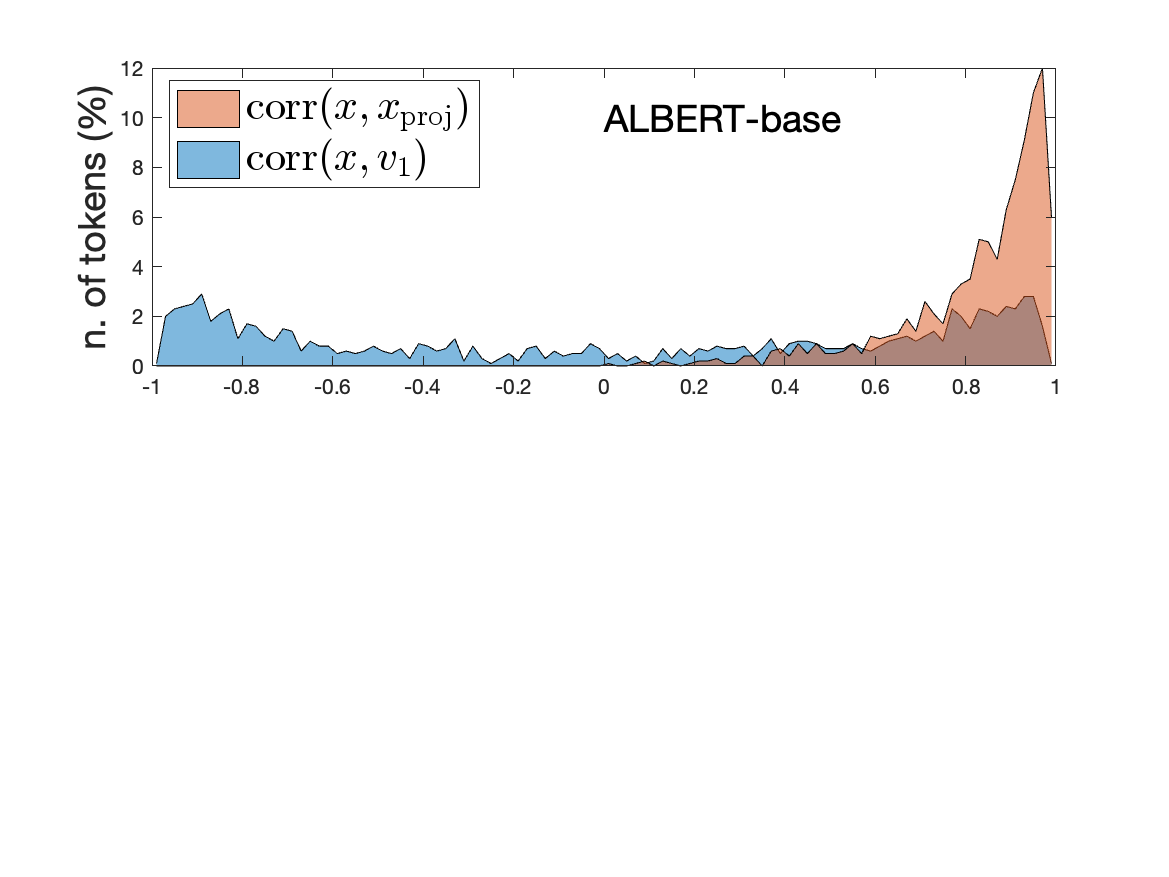}
        \includegraphics[trim=0.5cm 6cm 1.5cm 0cm, clip=true,width=0.48\textwidth]{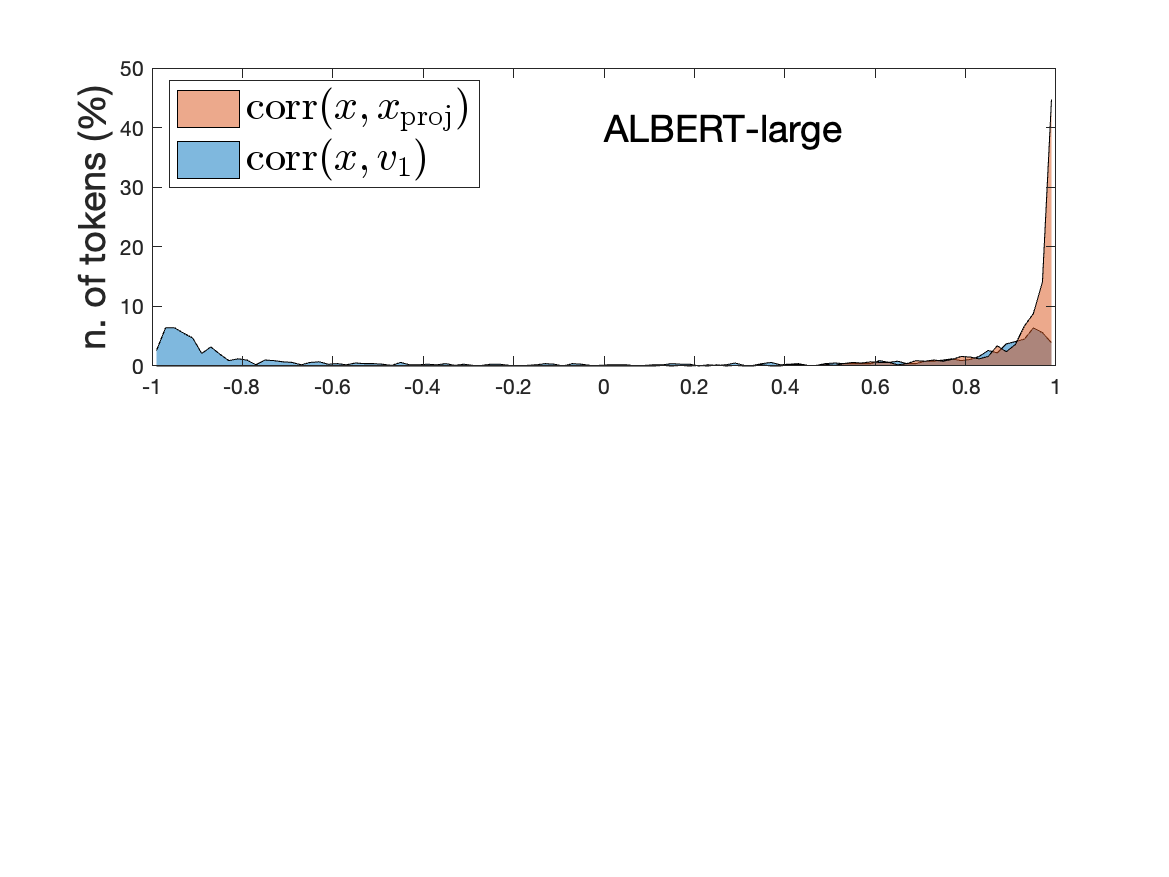}\\
        \includegraphics[trim=0.5cm 6cm 1.5cm 0cm, clip=true,width=0.48\textwidth]{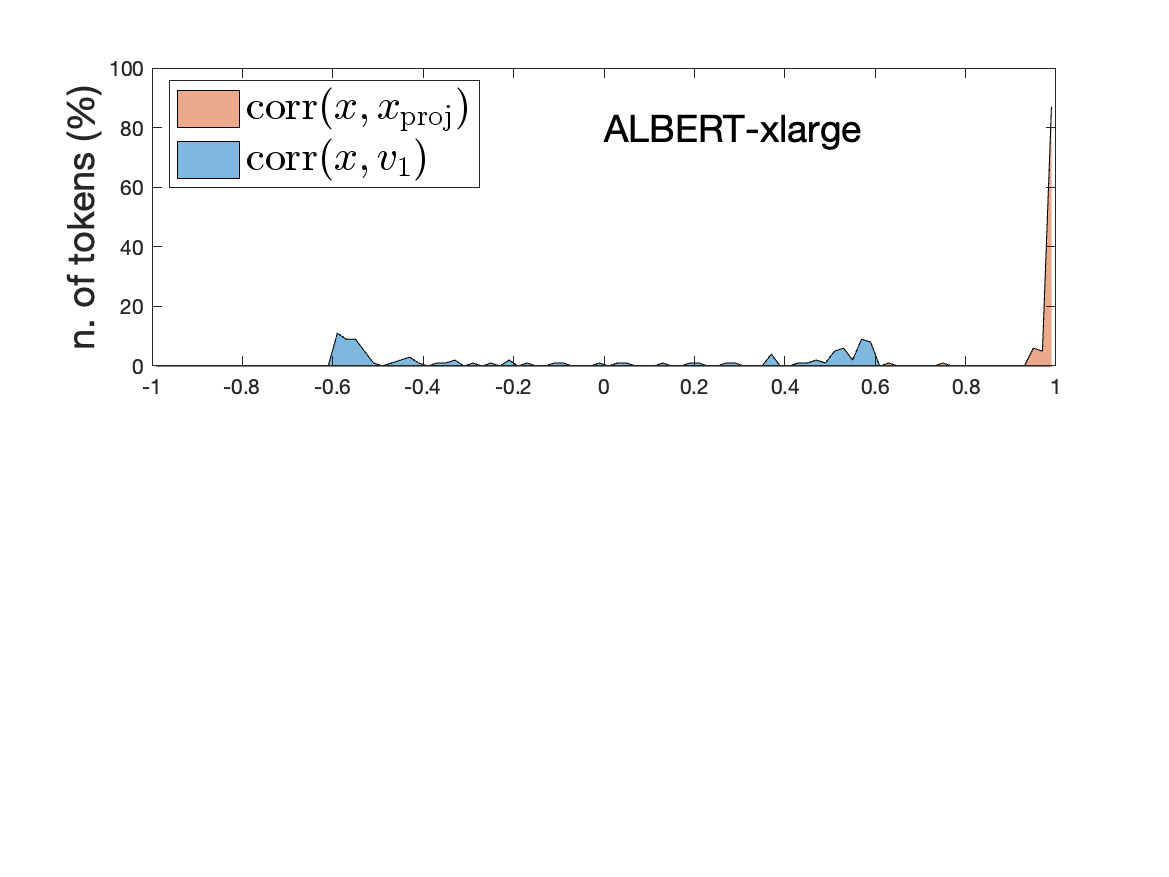}
        \includegraphics[trim=0.5cm 6cm 1.5cm 0cm, clip=true,width=0.48\textwidth]{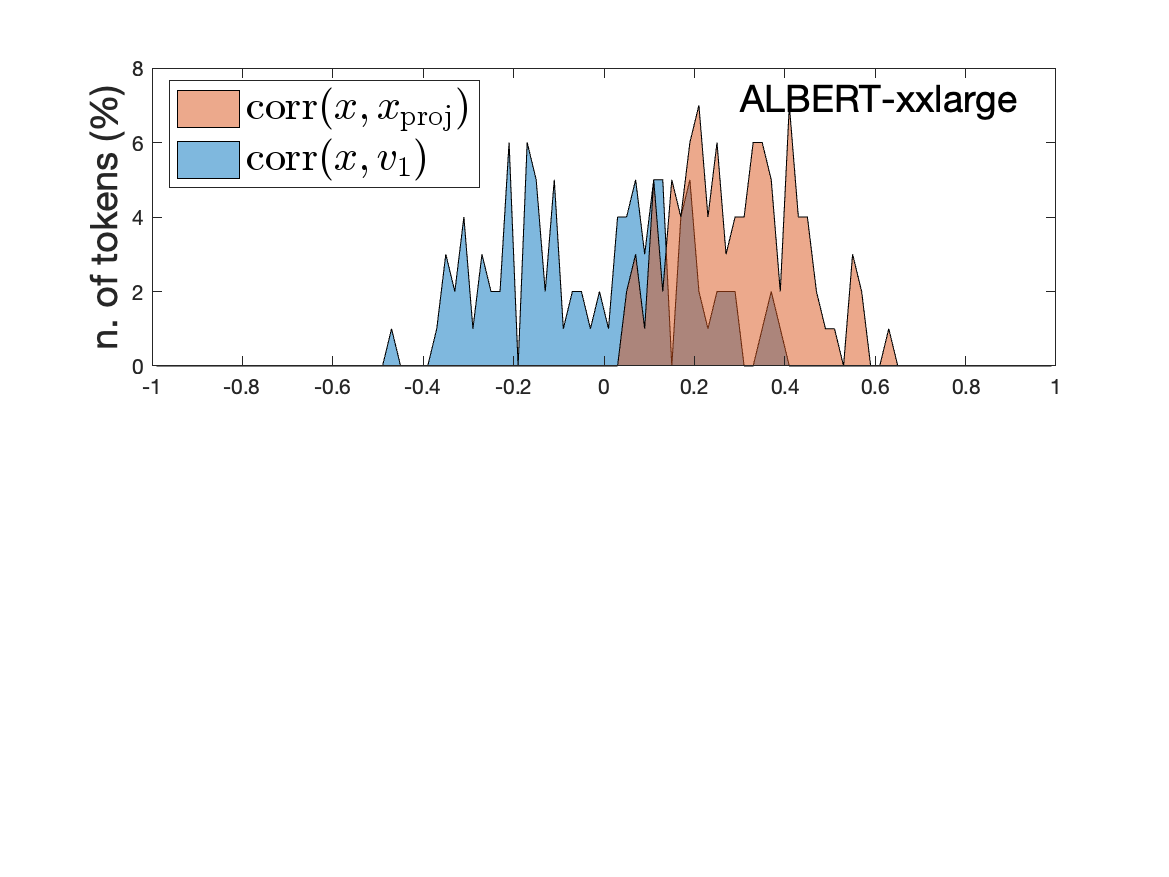}
        \end{minipage}}
        \caption{Eigenvalues and token alignment for the four ALBERT transformers. (a): The histograms show the distribution of the real part of the eigenvalues of $F$ in each of the four ALBERT transformers. In the first two $ \lambda_1 $ is real, in the remaining two $ \lambda_1 $ is complex. (b): The histograms show the alignment of the tokens at the end of a transformer, $ \bx_i(T)$, with the principal eigenvector $ \bv_1 $ and with the 2D subspace determined by $ \bv_1 $ and $ \bv_2 $ for $ \lambda_1 $ real, and by real and imaginary part of $ \bv_1$ when $ \lambda_1 $ is complex. }
        \label{fig:Albert1}
\end{figure*}

\item ALBERT-large: The principal eigenvalue $ \lambda_1  = 2.0674$ is real, positive, simple (even though not equal to the spectral radius of $F$: $ \lambda_1 < |\lambda_d| $, where $ \lambda_d = -2.1455  $ is the most negative eigenvalue).
Also in this case, $ \lambda_1 \approx \lambda_2 = 1.9586 $, and the correlation $ {\rm corr} (\bv_1, \bv_2 ) =0.8943$, meaning that the two eigenspaces $ {\rm span}\{\bv_1 \} $ and $ {\rm span}\{\bv_2 \} $ are almost collinear. Also here in the $ T=24 $ time horizon given by the number of layers sometimes $ \bx(T) \in {\rm span} \{\bv_1, \bv_2 \} $ rather than being strictly aligned with $ \bv_1 $, even though the distribution over $10^4$ initial $ \bx_i(0)$ is much more bimodal than for ALBERT-base, see Fig.~\ref{fig:Albert1}(b), because ALBERT-large is twice deeper.

\item  ALBERT-xlarge: The principal eigenvalue is complex, $ \lambda_1 = 1.5489 \pm 0.0598i $, and therefore it is expected that $ \bx_i $ keeps rotating in the 2D subspace generated by the complex conjugate eigenvector pair $ \bv_1 $ and $ \bv_2 $. In fact, as shown in Fig.~\ref{fig:Albert1}(b), $ \bx_i $ nearly always belongs to $ {\rm span}(P)$, where $  P =[{\rm Re} (\bv_1 ) \; {\rm Im}(\bv_1) ] $ even when the correlation between $ \bx_i $ and $ \bv_1 $ is low. 
%Notice further that $ | \lambda_1| < -\lambda_d $ where the most negative eigenvalue is $ \lambda_d =-2.5169$. 

\item ALBERT-xxlarge: the principal eigenvalue is complex, $ \lambda_1 = 3.5963 \pm 0.1512i$, so also this transformer is in Case~2 and it is expected that $ \bx_i$ converges to a limit cycle in the 2D subspace generated by the complex conjugate eigenvector pair $ \bv_1 $ and $ \bv_2 $.
However, for a transformer of this size ($ d=  4096$), $ T=12 $ layers are not really enough to achieve alignment with $ {\rm span}({\rm Re} (\bv_1 ),\; {\rm Im}(\bv_1)) $, see Fig.~\ref{fig:Albert1}(b). 
Also in this case, if we artificially prolong the transformer by adding extra identical layers the alignment improves, see Fig.~\ref{fig:Albert4} in the Supplementary Materials. 

\eite

For the 4 versions of ALBERT, we could never observe the system \eqref{eq:update-multihead-auton} reaching a bipartite consensus equilibrium, neither aligned with $ \bv_1 $ nor with any other eigenvector of $F$. 
This could be simply due to the fact that the domain of attraction of the bipartite equilibria is very small, hence never encountered by our initial conditions $ \bx_i(0)$. Another possible reason is that the concatenated structure of the multihead never renders the bipartite equilibria locally asymptotically stable, unlike in the single head case (see Theorem~\ref{thm:stab-main-self} in the Supplementary Materials).

\subsection{Application to transformers with layer-varying weights}
When we apply our considerations to transformers like BERT, GPT2 and GPTNeo in which the weight matrices change from layer to layer, then only the level of alignment with the current $ \bv_{1,t} $ can be checked, and it is not expected to increase over the layers, as $ \bv_{1, t-1} $ and $ \bv_{1,t}$ are typically significantly different in direction. 
Denote $ \psi_t^{\rm mean} = \sum_{k=1}^d \psi(\bv_k)/d $ the mean of the increment \eqref{eq:corr-increm1} over all eigenvectors of $ F_t $ (with obvious corrections for the case of complex eigenvalues, in the spirit of \eqref{eq:psi_t}). Similarly, denote $ \rho_t^{\rm mean} = \sum_{k=1}^d \rho(\bv_k)/d $ the average probability of increment over all eigendirections, and $\gamma_t^{\rm mean} = \sum_{k=1}^d \gamma(\bv_k)/d$ the mean of the correlation over all eigenvectors of $F_t$.
In Fig.~\ref{fig:psi_rho_selected} we compare $ \gamma_t $ with $ \gamma_t^{\rm mean} $, $ \psi_t $ with $ \psi_t^{\rm mean} $, and $ \rho_t $ with $ \rho_t^{\rm mean} $ for BERT-base, GPT2-medium and GPTNeo-2.7B (other 6 transformer models in the BERT, GPT2 and GPTNeo families are shown in Fig.~\ref{fig:psi_single_all} in the Supplementary Materials). As can be seen in the middle panels, $ \psi_t $ is almost always positive and larger than $ \psi_t^{\rm mean} $. Therefore $ \rho_t > 0.5 $ for almost all layers as shown in the lower panels of Fig.~\ref{fig:psi_rho_selected}.

\begin{figure}[!htbp]
  \centering
  % Row 1: psi_combined 
      \subfigure[]{\includegraphics[width=0.32\textwidth]{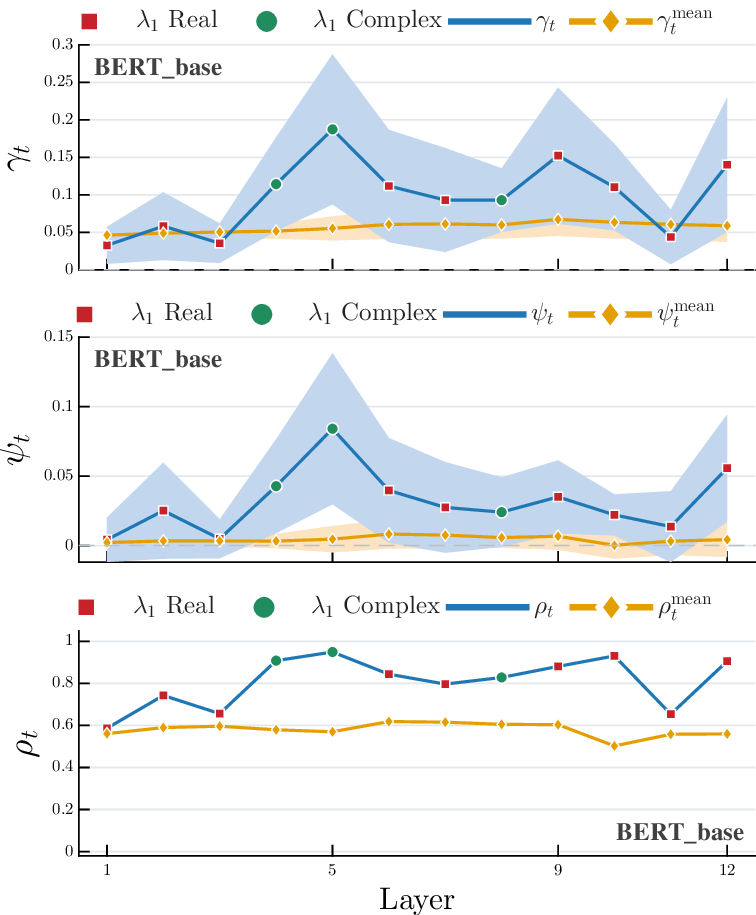}}
        \subfigure[]{\includegraphics[width=0.32\textwidth]{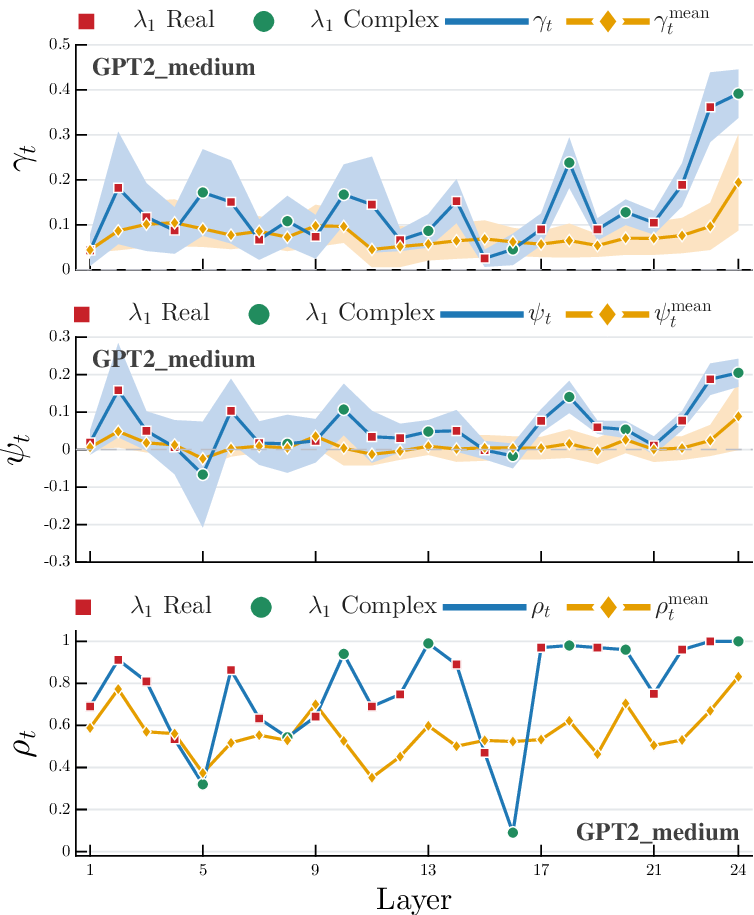}}        
      \subfigure[]{ \includegraphics[width=0.32\textwidth]{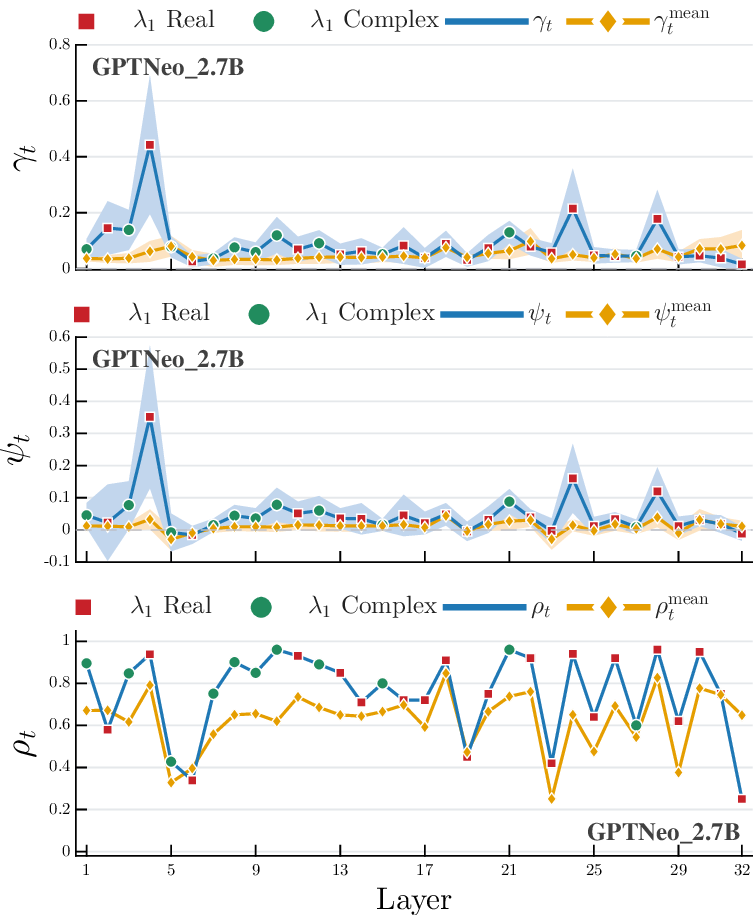}}
  % \caption{Values of $\psi_t$  and $ \rho_t $ across layers for BERT-base, and GPT2-medium and GPTNeo-2.7B. Top: $\psi_t$ and $ \psi_t^{\rm mean}$ and associated standard deviations over $ 10^4 $ tokens are shown; Bottom: corresponding probabilities $\rho_t$ and $ \rho_t^{\rm mean}$. See Fig.~\ref{fig:psi_single_all} in the Supplementary Material for the remaining models.}
  \caption{Values of $\gamma_t$ (top row), $\psi_t$ (middle row) and $\rho_t$ (bottom row) across layers for BERT-base, GPT2-medium and GPTNeo-2.7B. Top: $\gamma_t$ and $\gamma_t^{\rm mean}$ and associated standard deviations over $10^4$ tokens are shown; Middle: $\psi_t$ and $\psi_t^{\rm mean}$ and associated standard deviations over $10^4$ tokens; Bottom: corresponding probabilities $\rho_t$ and $\rho_t^{\rm mean}$. See Fig.~\ref{fig:psi_single_all} in the Supplementary Material for the remaining models.}
  \label{fig:psi_rho_selected}
\end{figure}

While an increment in correlation is visible, and confirms the idea that the principal eigenvector carry information on the inner behavior of a transformer, the absolute values of correlation $ \gamma_t$ remain low for all layer-varying transformer families, and do not grow with the layer index (nor with the depth of the transformer), see top row in Fig.~\ref{fig:psi_rho_selected}.
This is also due to the fact that the spectral gap in the layers of a transformer (i.e., the difference $ \lambda_{1,t} - \lambda_{2,t}$, with obvious corrections when the eigenvalues are complex) is normally minimal, see Fig.~\ref{fig:all_models_re_lambda} in the Supplementary Materials. 

For layer-varying transformers, one can wonder if a principal direction of alignment emerges as the tokens pass through the layers. 
Since in our framework what decided the behavior of a layer is the layer matrix $ F_t$, one can ask if the cumulative product of such matrices $ \Pi_{\tau=1}^t F_\tau $ has some sort of convergence property, i.e., if conditions in the style of Oseledets multiplicative ergodic theorem for random matrices may lead to the identification of Lyapunov exponents and associated directions in token space. 
As shown in Fig.~\ref{fig:pi_f_all} in the Supplementary Materials, this does not seem to be the case for our transformers. 
Call $ \bv_{1,t}^{\rm cum} $ the principal eigenvector of the cumulative product $F_t^{\rm cum} = \Pi_{\tau=1}^t F_\tau $ as $ t$ varies from $ 1$ to $T$. 
The absolute correlation  $ \left|\mathrm{corr}(\bx_i{(t+1)},\bv_{1,t}^{\rm cum} )\right|$ shows no clear trend in the majority of models when compared with the average correlation $\sum_{j=1}^d \left|\mathrm{corr}(\bx_i{(\tau+1)},\bv_{j,t}^{\rm cum} )\right| / d $.

\subsection{Steering the tokens towards a desired direction}
\label{sec:steer-tokens}

If we apply the rank-1 perturbation described in Section~\ref{sec:steering-token} to the final layer of a layer-varying transformer, then it is easy to achieve a high correlation $ \gamma_T = \left|\mathrm{corr} (\bx_i{(T)},\bv_{1,T,{\rm mod} })\right| $ at the end of a transformer. This, combined with the consensus argument, implies that all tokens get similarly projected onto the subspace $ {\rm span}(\bv_{1,T,{\rm mod}}) \approx {\rm span}(\bw_r) $ which can be chosen by the designer. Such a correlation is shown in Fig.~\ref{fig:modified_fo_some} for BERT-base, GPT2-medium and GPTNeo-2.7B (and in Fig.~\ref{fig:modified_fo_all} in the Supplementary Materials for the other 6 transformer models).
In these examples, the spectral gap $ \lambda_{1,T,{\rm mod}}- {\rm Re}(\lambda_{2,T,{\rm mod}}) $ is on the order  of $10 - 10^3$, depending on the model, and allows us to achieve a near-perfect alignment with $ \bv_{1,T,{\rm mod}}$.
% \rednote{Chenglong: I would prefer a higher order (can we say order 50? or ``order'' is always multiple of 10 [i.e. here 100]) but yielding a c  orrelation of 1 exactly..}

% \rednote{Chenglong: Different models need different gap sizes to achieve a perfect alignment with $\bv_{1,T,mod}$. For example, BERT-base gets a correlation of 0.98 with a gap of 100 (order 10). GPT2-xl reaches about 0.98 at a gap of 300 (order 30). But for GPTNeo-125M, we need a much larger gap of $10^4$(order $10^3$) to get roughly 0.94. Would it be reasonable to say: In these examples the spectral ratio $ \lambda_{1,T,{\rm mod}} /  {\rm Re}(\lambda_{2,T,{\rm mod}}) $ is of order $10^2$ and allows to achieve a perfect alignment with $ \bv_{1,T,{\rm mod}}$.}

% \bluenote{CA: We can write like this, but the spectral gap is the difference, not the ratio, so it would be better a statement in that direction. We can stick to the original sentence (and plot) you produced last week. We can write 
% "In these examples the spectral gap $ \lambda_{1,T,{\rm mod}}- {\rm Re}(\lambda_{2,T,{\rm mod}}) $ is of order $10\div10^3$, depending on the model, and allows...."
% }

\begin{figure}[!htbp]
  \centering
  \includegraphics[width=0.32\textwidth]{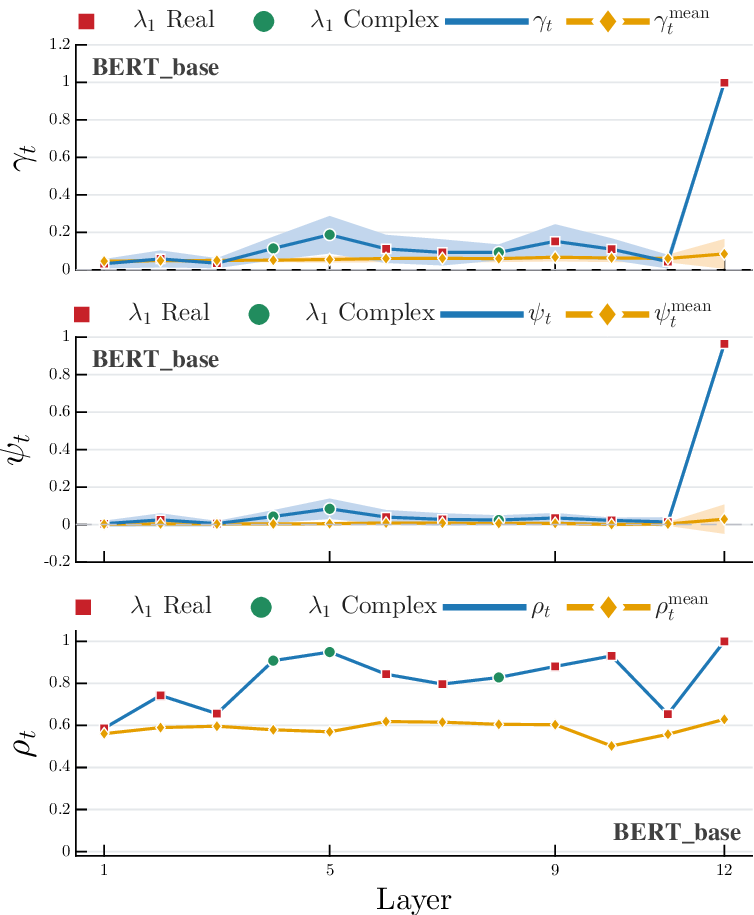}\hfill
  \includegraphics[width=0.32\textwidth]{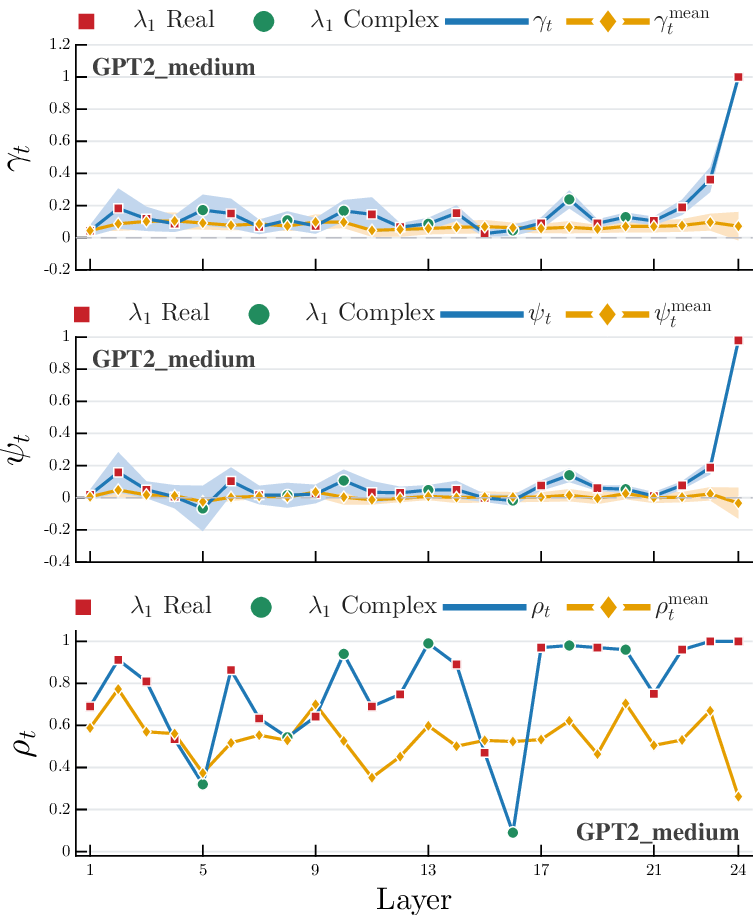}\hfill
  \includegraphics[width=0.32\textwidth]{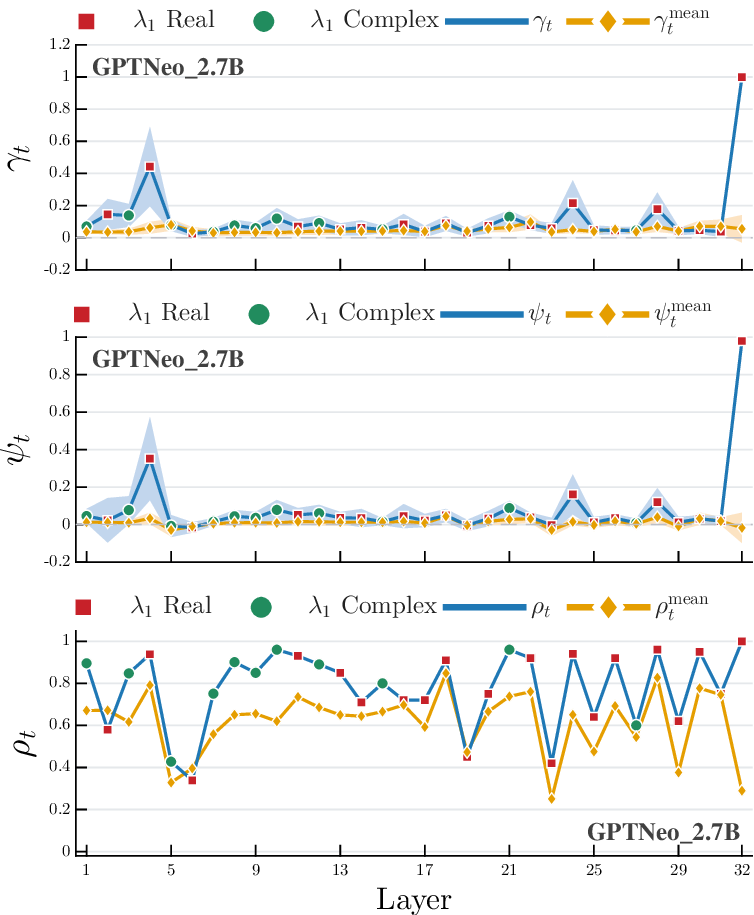}

  \caption{Values of $ \gamma_t $ (top row), $\psi_t$ (middle row) and $ \rho_t $ (bottom row) across layers for BERT-base, and GPT2-medium and GPTNeo-2.7B, with the last layer output matrix $ F_{O,T,{\rm mod}}$ modified as described in Section~\ref{sec:steer-tokens}. See Fig.~\ref{fig:modified_fo_all} in the Supplementary Materials for the other 6 models. Perfect alignment with $ \bv_{1,T,{\rm mod}}$ is always achieved.}
  \label{fig:modified_fo_some}
\end{figure}

\section{Discussion}

The main contribution of this paper is to show that in a transformer the behavior of each layer (in which we ignore the effect of the FFN) has some resemblance with a step of the power method algorithm, notably in the tendency of token vectors to tilt towards the principal eigenvector of what we call the layer matrix $ F_t$. 
This behavior is clearly visible in a ``universal'' transformer with all identical weight matrices across all layers, because in this case the principal eigenvector stays constant across layers and the steps in the power method accumulate, hence convergence is macroscopically observable. 
In a transformer with layer-varying weight matrices, instead, the principal eigenvector keeps changing from layer to layer, hence no accumulation of power method steps can occur. 
Nevertheless, each layer still represents a projection step as confirmed by the experimental data. 
What matters in the amplitude of this projection is the spectral gap between the first and second eigenvalue of the layer matrix, and this is sometimes visible in the data: for instance one of the layers achieving the largest $ \gamma_t $ is in GPTNeo-2.7B, at $ t=4$, see Fig.~\ref{fig:psi_rho_selected} right panel, which happens to correspond to one of the few cases in which the spectral gap in $ F_t $ is sizable ($ \lambda_{1,4} -\lambda_{2,4}>5$, see Fig.~\ref{fig:all_models_re_lambda} in the Supplementary Materials).

The other main contribution of the paper builds on this observation.
Since the spectral gap in $ F_t $ matters in the projection, if we impose a sufficiently large spectral gap in the final layer of a transformer then we can get a strong projection  along the corresponding eigenvector. 
We show above that this is achievable in a constructive way: a user-specified rank-1 modification in the output matrix of the final layer allows to choose the direction of the projection in an arbitrary way. If only a mild steering towards a desired direction is needed, this can be accomplished by choosing a smaller $ \sigma$. 

% In practice, rather than a strong projection and complete alignment, one may wish to implement a mild steering towards a desirable direction. 
In a transformer, token oversmoothing in depth (i.e., consensus) is an undesired effect. 
Consensus can be slowed down by the skip connection (for us, by modulating $ \eta$ \cite{dong2021attention,noci2022signal}) but in a model like \eqref{eq:update-multihead-auton} it cannot be eliminated, because it is an asymptotic property (infinite depth).
%and the step size $ \eta $ is constant \rednote{[\cite{noci2022signal} claims that there are time-varying $ \eta$ for which consensus can be avoided: check this]. Chenglong:Noci's $\eta$ is not time-varying but depth-dependent, with $\eta = 1/\sqrt{L}$ at every layer. This holds only at initialization; once training begins, $\eta$ will inevitably change.}
%In particular, in a model like \eqref{eq:update-multihead1}, modulating the importance of self-attention and skip connection through $ \eta$ does not suppress consensus but renders the time-scale of convergence too slow for the number of layers available in a real transformer. 
%We can observe this fact quite well in ALBERT-xxlarge, cf Fig.~\ref{fig:Albert3} and Fig.~\ref{fig:Albert4} in the Supplementary Materials.
Apart from using skip connections, methods to tame oversmoothing include modifying the attention matrix \cite{dovonon2024setting}, for instance scaling the softmax operation inside the attention block with a scalar parameter (inverse temperature) \cite{geshkovski2025mathematical,giorlandino2025two} or removing the principal (uniform) mode \cite{saada2024mind}.
The idea of \cite{saada2024mind} of suppressing the principal mode in the attention matrix relies on the same principle as our token steering method: exploit the presence of a significant spectral gap between the principal eigenvalue and the rest of the spectrum. 
The aim and its implementation are however completely different: we aim at creating a spectral gap in the value-output projection map of the last layer of the transformer $ F_t $, while \cite{saada2024mind} aims at removing the uniform behavior inside the attention matrix.

In a model like \eqref{eq:update-multihead1} or \eqref{eq:update-multihead-auton}, the $L_2$ layer normalization operation performed at every layer effectively renders the manifold in which each token is evolving a sphere $ \mathbb{S}^{d-1} \subset \mathbb{R}^{d}$, see Section~\ref{sec:single-head-analysis} in the Supplementary Materials. 
The intersection of a 1-dimensional eigenspace $ {\rm span}(\bv_1 )$ with $ \mathbb{S}^{d-1}$ consists of two points $ \bv_1 $ (normalized to $ \| \bv_1 \| =1 $) and $ - \bv_1 $. 
As can be seen e.g. in Fig.~\ref{fig:Albert1}(b), the alignment of the tokens can be towards $ \bv_1 $ or towards $ - \bv_1 $, and the absolute value in the correlations of formulas \eqref{eq:corr-increm1}-\eqref{eq:gamma_t} is meant to account for both cases. 
% it is a priori difficult to predict which of the two points in the unit sphere an initial condition will converge to. 
% Consequently, when computing correlations between tokens and principal eigenvectors (formulas \eqref{eq:corr-increm1}-\eqref{eq:gamma_t}) we need to use the absolute values, to take into account both cases.
As mentioned in Section~\ref{sec:shared-param}, in the time-invariant case these two consensus points appear to be the only attractors of the dynamics of \eqref{eq:update-multihead-auton}, in spite of the presence of other equilibria, mentioned in Lemma~\ref{lem:self-equil1} in the Supplementary Materials.

\section{Limitations}
\label{sec:limitations}
In the model \eqref{eq:update-multihead1} (or \eqref{eq:update-multihead-auton}), the FFN inside each layer is ignored. 
Furthermore, the input of the model is a vector of tokens, which we choose randomly. 
However, in a real transformer, the model receives as input a prompt, which passes through an embedding layer, upstream of the layers considered here. 
If we repeat the analysis while including both embedding and FFN networks and using the original (affine) $ {\rm LayerNorm}(\cdot)$ normalization, then we obtain the results of Fig.~\ref{fig:albert_real_llm_corr} for ALBERT and Fig.~\ref{fig:psi_bert_gpt} for BERT, GPT2 and GPTNeo. 
%\rednote{[Make a figure for ALBERT similar to Fig. 1(b). CA: No need}.
All $ \gamma_t$, $ \psi_t $ and $ \rho_t $ are degraded, in both layer-invariant and layer-varying transformers.
In particular, while the consensus among the tokens is preserved \cite{abella2024asymptotic}, the direction of alignment is modified. This is obviously expected, and it is due to the FFNs. 
The FFN is a nonlinear static map $ \phi_t \, : \, \mathbb{R}^d \to \mathbb{R}^d $ which ``deforms'' the token space $ \mathbb{R}^d $ at each layer. 
Since $ \phi_t$ is known (its weights are learned quantities for us), the deformation can however be precompensated: the FFN $ \phi_t $ is typically a 2-layer fully connected map with an activation function (GELU or SiLU for our transformers) in the hidden layer, which has dimension larger than $ d$, meaning that we can in general compute a (possibly non-unique) numerical approximation of its inverse. Therefore, if we aim to steer the tokens towards a desired direction $ \bw_{\rm des} $ in the image of $ \phi$ as in Section~\ref{sec:steer-tokens}, we can alter the principal eigenvector of the last layer opportunely through rank-1 maps, similar to what was already considered, but imposing $ \bw_r = \phi_T^{-1}(\bw_{\rm des})$. 

\bibliographystyle{abbrv}
\bibliography{references}

@article{giorlandino2025two,
  title={Two failure modes of deep transformers and how to avoid them: a unified theory of signal propagation at initialisation},
  author={Giorlandino, Alessio and Goldt, Sebastian},
  journal={arXiv preprint arXiv:2505.24333},
  year={2025}
}

@article{saada2024mind,
  title={Mind the gap: a spectral analysis of rank collapse and signal propagation in attention layers},
  author={Saada, Thiziri Nait and Naderi, Alireza and Tanner, Jared},
  journal={arXiv preprint arXiv:2410.07799},
  year={2024}
}

@inproceedings{NEURIPS2023_91edff07,
 author = {Madaan, Aman and Tandon, Niket and Gupta, Prakhar and Hallinan, Skyler and Gao, Luyu and Wiegreffe, Sarah and Alon, Uri and Dziri, Nouha and Prabhumoye, Shrimai and Yang, Yiming and Gupta, Shashank and Majumder, Bodhisattwa Prasad and Hermann, Katherine and Welleck, Sean and Yazdanbakhsh, Amir and Clark, Peter},
 booktitle = {Advances in Neural Information Processing Systems},
 editor = {A. Oh and T. Naumann and A. Globerson and K. Saenko and M. Hardt and S. Levine},
 pages = {46534--46594},
 publisher = {Curran Associates, Inc.},
 title = {Self-Refine: Iterative Refinement with Self-Feedback},
 url = {https://proceedings.neurips.cc/paper_files/paper/2023/file/91edff07232fb1b55a505a9e9f6c0ff3-Paper-Conference.pdf},
 volume = {36},
 year = {2023}
}

@article{ilharco2022editing,
  title={Editing models with task arithmetic},
  author={Ilharco, Gabriel and Ribeiro, Marco Tulio and Wortsman, Mitchell and Gururangan, Suchin and Schmidt, Ludwig and Hajishirzi, Hannaneh and Farhadi, Ali},
  journal={arXiv preprint arXiv:2212.04089},
  year={2022}
}

@article{dathathri2019plug,
  title={Plug and play language models: A simple approach to controlled text generation},
  author={Dathathri, Sumanth and Madotto, Andrea and Lan, Janice and Hung, Jane and Frank, Eric and Molino, Piero and Yosinski, Jason and Liu, Rosanne},
  journal={arXiv preprint arXiv:1912.02164},
  year={2019}
}

@misc{lester2021powerscaleparameterefficientprompt,
      title={The Power of Scale for Parameter-Efficient Prompt Tuning}, 
      author={Brian Lester and Rami Al-Rfou and Noah Constant},
      year={2021},
      eprint={2104.08691},
      archivePrefix={arXiv},
      primaryClass={cs.CL},
      url={https://arxiv.org/abs/2104.08691}, 
}

@misc{hu2021loralowrankadaptationlarge,
      title={LoRA: Low-Rank Adaptation of Large Language Models}, 
      author={Edward J. Hu and Yelong Shen and Phillip Wallis and Zeyuan Allen-Zhu and Yuanzhi Li and Shean Wang and Lu Wang and Weizhu Chen},
      year={2021},
      eprint={2106.09685},
      archivePrefix={arXiv},
      primaryClass={cs.CL},
      url={https://arxiv.org/abs/2106.09685}, 
}

@misc{turner2023steering,
      title={Steering Language Models With Activation Engineering}, 
      author={Alexander Matt Turner and Lisa Thiergart and Gavin Leech and David Udell and Juan J. Vazquez and Ulisse Mini and Monte MacDiarmid},
      year={2024},
      eprint={2308.10248},
      archivePrefix={arXiv},
      primaryClass={cs.CL},
      url={https://arxiv.org/abs/2308.10248}, 
}

@article{schulhoff2024prompt,
  title={The prompt report: A systematic survey of prompt engineering techniques},
  author={Schulhoff, Sander and Ilie, Michael and Balepur, Nishant and Kahadze, Konstantine and Liu, Amanda and Si, Chenglei and Li, Yinheng and Gupta, Aayush and Han, HyoJung and Schulhoff, Sevien and others},
  journal={arXiv preprint arXiv:2406.06608},
  year={2024}
}

@inproceedings{shin2020autoprompt,
  title={Autoprompt: Eliciting knowledge from language models with automatically generated prompts},
  author={Shin, Taylor and Razeghi, Yasaman and Logan IV, Robert L and Wallace, Eric and Singh, Sameer},
  booktitle={Proceedings of the 2020 conference on empirical methods in natural language processing (EMNLP)},
  pages={4222--4235},
  year={2020}
}

@misc{zelikman2022starbootstrappingreasoningreasoning,
      title={STaR: Bootstrapping Reasoning With Reasoning}, 
      author={Eric Zelikman and Yuhuai Wu and Jesse Mu and Noah D. Goodman},
      year={2022},
      eprint={2203.14465},
      archivePrefix={arXiv},
      primaryClass={cs.LG},
      url={https://arxiv.org/abs/2203.14465}, 
}

@article{altafini2025multistability,
      title={Multistability of Self-Attention Dynamics in Transformers}, 
      author={Claudio Altafini},
      year={2025},
      journal={arXiv:2511.11553},
      url={https://arxiv.org/abs/2511.11553}, 
}

@article{wu2024role,
  title={On the role of attention masks and layernorm in transformers},
  author={Wu, Xinyi and Ajorlou, Amir and Wang, Yifei and Jegelka, Stefanie and Jadbabaie, Ali},
  journal={Advances in Neural Information Processing Systems},
  volume={37},
  pages={14774--14809},
  year={2024}
}

@article{dutta2021redesigning,
  title={Redesigning the transformer architecture with insights from multi-particle dynamical systems},
  author={Dutta, Subhabrata and Gautam, Tanya and Chakrabarti, Soumen and Chakraborty, Tanmoy},
  journal={Advances in Neural Information Processing Systems},
  volume={34},
  pages={5531--5544},
  year={2021}
}

@article{lu2019understanding,
  title={Understanding and improving transformer from a multi-particle dynamic system point of view},
  author={Lu, Yiping and Li, Zhuohan and He, Di and Sun, Zhiqing and Dong, Bin and Qin, Tao and Wang, Liwei and Liu, Tie-Yan},
  journal={arXiv preprint arXiv:1906.02762},
  year={2019}
}

@inproceedings{sander2022sinkformers,
  title={Sinkformers: Transformers with doubly stochastic attention},
  author={Sander, Michael E and Ablin, Pierre and Blondel, Mathieu and Peyr{\'e}, Gabriel},
  booktitle={International Conference on Artificial Intelligence and Statistics},
  pages={3515--3530},
  year={2022},
  organization={PMLR}
}

@article{abella2024asymptotic,
  title={The asymptotic behavior of attention in transformers},
  author={Abella, {\'A}lvaro Rodr{\'\i}guez and Silvestre, Jo{\~a}o Pedro and Tabuada, Paulo},
  journal={arXiv preprint arXiv:2412.02682},
  year={2024}
}

@inproceedings{abellaconsensus,
  title={Consensus Is All You Get: The Role of Attention in Transformers},
  author={Abella, {\'A}lvaro Rodr{\'\i}guez and Silvestre, Jo{\~a}o Pedro and Tabuada, Paulo},
  booktitle={Forty-second International Conference on Machine Learning},
  year={2025}
}

@article{vaswani2017attention,
  title={Attention is all you need},
  author={Vaswani, Ashish and Shazeer, Noam and Parmar, Niki and Uszkoreit, Jakob and Jones, Llion and Gomez, Aidan N and Kaiser, {\L}ukasz and Polosukhin, Illia},
  journal={Advances in neural information processing systems},
  volume={30},
  year={2017}
}

@article{scholkemper2024residual,
  title={Residual connections and normalization can provably prevent oversmoothing in gnns},
  author={Scholkemper, Michael and Wu, Xinyi and Jadbabaie, Ali and Schaub, Michael T},
  journal={arXiv preprint arXiv:2406.02997},
  year={2024}
}

@article{dovonon2024setting,
  title={Setting the record straight on transformer oversmoothing},
  author={Dovonon, Gb{\`e}tondji JS and Bronstein, Michael M and Kusner, Matt J},
  journal={arXiv preprint arXiv:2401.04301},
  year={2024}
}

@article{nguyen2023mitigating,
  title={Mitigating over-smoothing in transformers via regularized nonlocal functionals},
  author={Nguyen, Tam and Nguyen, Tan and Baraniuk, Richard},
  journal={Advances in Neural Information Processing Systems},
  volume={36},
  pages={80233--80256},
  year={2023}
}

@inproceedings{zhai2023stabilizing,
  title={Stabilizing transformer training by preventing attention entropy collapse},
  author={Zhai, Shuangfei and Likhomanenko, Tatiana and Littwin, Etai and Busbridge, Dan and Ramapuram, Jason and Zhang, Yizhe and Gu, Jiatao and Susskind, Joshua M},
  booktitle={International Conference on Machine Learning},
  pages={40770--40803},
  year={2023},
  organization={PMLR}
}

@article{shi2022revisiting,
  title={Revisiting over-smoothing in bert from the perspective of graph},
  author={Shi, Han and Gao, Jiahui and Xu, Hang and Liang, Xiaodan and Li, Zhenguo and Kong, Lingpeng and Lee, Stephen and Kwok, James T},
  journal={arXiv preprint arXiv:2202.08625},
  year={2022}
}

@article{noci2022signal,
  title={Signal propagation in transformers: Theoretical perspectives and the role of rank collapse},
  author={Noci, Lorenzo and Anagnostidis, Sotiris and Biggio, Luca and Orvieto, Antonio and Singh, Sidak Pal and Lucchi, Aurelien},
  journal={Advances in Neural Information Processing Systems},
  volume={35},
  pages={27198--27211},
  year={2022}
}

@inproceedings{dong2021attention,
  title={Attention is not all you need: Pure attention loses rank doubly exponentially with depth},
  author={Dong, Yihe and Cordonnier, Jean-Baptiste and Loukas, Andreas},
  booktitle={International conference on machine learning},
  pages={2793--2803},
  year={2021},
  organization={PMLR}
}

@article{oja1985stochastic,
  title={On stochastic approximation of the eigenvectors and eigenvalues of the expectation of a random matrix},
  author={Oja, Erkki and Karhunen, Juha},
  journal={Journal of mathematical analysis and applications},
  volume={106},
  number={1},
  pages={69--84},
  year={1985},
  publisher={Elsevier}
}

@article{yoshizawa2001convergence,
  title={Convergence analysis for principal component flows},
  author={Yoshizawa, Shintaro and Helmke, Uwe and Starkov, Konstantin},
 journal={ International Journal of Applied Mathematics and Computer Science},
  year={2001},
  volume={11},
  number={1},
  pages={223--236},
  publisher={Zielona G{\'o}ra: Uniwersytet Zielonog{\'o}rski}
}

@article{geshkovski2024dynamic,
  title={Dynamic metastability in the self-attention model},
  author={Geshkovski, Borjan and Koubbi, Hugo and Polyanskiy, Yury and Rigollet, Philippe},
  journal={arXiv preprint arXiv:2410.06833},
  year={2024}
}

@article{geshkovski2023emergence,
  title={The emergence of clusters in self-attention dynamics},
  author={Geshkovski, Borjan and Letrouit, Cyril and Polyanskiy, Yury and Rigollet, Philippe},
  journal={Advances in Neural Information Processing Systems},
  volume={36},
  pages={57026--57037},
  year={2023}
}

@article{geshkovski2025mathematical,
  title={A mathematical perspective on transformers},
  author={Geshkovski, Borjan and Letrouit, Cyril and Polyanskiy, Yury and Rigollet, Philippe},
  journal={Bulletin of the American Mathematical Society},
  volume={62},
  number={3},
  pages={427--479},
  year={2025}
}

@misc{lan2019albert,
      title={ALBERT: A Lite BERT for Self-supervised Learning of Language Representations}, 
      author={Zhenzhong Lan and Mingda Chen and Sebastian Goodman and Kevin Gimpel and Piyush Sharma and Radu Soricut},
      year={2020},
      eprint={1909.11942},
      archivePrefix={arXiv},
      primaryClass={cs.CL},
      url={https://arxiv.org/abs/1909.11942}, 
}

@misc{bert,
      title={BERT: Pre-training of Deep Bidirectional Transformers for Language Understanding}, 
      author={Jacob Devlin and Ming-Wei Chang and Kenton Lee and Kristina Toutanova},
      year={2019},
      eprint={1810.04805},
      archivePrefix={arXiv},
      primaryClass={cs.CL},
      url={https://arxiv.org/abs/1810.04805}, 
}

@inproceedings{gpt2,
  title={Language Models are Unsupervised Multitask Learners},
  author={Alec Radford and Jeff Wu and Rewon Child and David Luan and Dario Amodei and Ilya Sutskever},
  year={2019},
  url={https://api.semanticscholar.org/CorpusID:160025533}
}

@article{gptneo,
  title={The Pile: An 800GB Dataset of Diverse Text for Language Modeling},
  author={Gao, Leo and Biderman, Stella and Black, Sid and Golding, Laurence and Hoppe, Travis and Foster, Charles and Phang, Jason and He, Horace and Thite, Anish and Nabeshima, Noa and others},
  journal={arXiv preprint arXiv:2101.00027},
  year={2020}
}

@article{Likhosherstov2021OnTE,
  title={On the Expressive Power of Self-Attention Matrices},
  author={Valerii Likhosherstov and Krzysztof Choromanski and Adrian Weller},
  journal={ArXiv},
  year={2021},
  volume={abs/2106.03764},
  url={https://api.semanticscholar.org/CorpusID:235359134}
}

@misc{ustaomeroglu2025theoreticalstudyhyperselfattention,
      title={A Theoretical Study of (Hyper) Self-Attention through the Lens of Interactions: Representation, Training, Generalization}, 
      author={Muhammed Ustaomeroglu and Guannan Qu},
      year={2025},
      eprint={2506.06179},
      archivePrefix={arXiv},
      primaryClass={cs.LG},
      url={https://arxiv.org/abs/2506.06179}, 
}

@misc{chowdhury2022learningtransformerkernel,
      title={On Learning the Transformer Kernel}, 
      author={Sankalan Pal Chowdhury and Adamos Solomou and Avinava Dubey and Mrinmaya Sachan},
      year={2022},
      eprint={2110.08323},
      archivePrefix={arXiv},
      primaryClass={cs.LG},
      url={https://arxiv.org/abs/2110.08323}, 
}

\clearpage

\appendix 
\onecolumn

\begin{center}
    {\bf \LARGE Supplementary Materials }
\end{center}

\section{Relation to the Standard Transformer Formulation}
\label{sec:standardW}
The main paper adopts a column-token convention and works with a composite matrix (``layer matrix'') $F_t = F_{O,t} F_{V,t}$ that differs in notation from the standard per-head description of transformers. In this section, we make the correspondence between the two conventions explicit.

We adopt the column-token convention: a sequence of $n$ tokens at layer $t$ is represented as $\bm{X}_{t} = [\bx_1(t), \ldots,\bx_n(t)] \in \mathbb{R}^{d\times n} $, where each $\bx_i(t) \in \mathbb{R}^{d}$ is a column vector. Under this convention all matrices are the transposes of the learned parameter matrices commonly used in the ML literature \cite{vaswani2017attention, Likhosherstov2021OnTE, chowdhury2022learningtransformerkernel, ustaomeroglu2025theoreticalstudyhyperselfattention}. We denote them by the letter $F$ to distinguish them from the standard parameter matrices $W$. Explicitly,
\begin{equation*}
  F_{Q,t} = (W_{Q,t})^{\top}, \quad
  F_{K,t} = (W_{K,t})^{\top}, \quad
  F_{V,t} = (W_{V,t})^{\top}, \quad
  F_{O,t} = (W_{O,t})^{\top}.
\end{equation*}

A layer with $H$ heads and head dimension $d_k = d/H$ is parameterized by the per-head query, key, and value projections $F_{Q,t}^{(h)}, F_{K,t}^{(h)},F_{V,t}^{(h)} \in \mathbb{R}^{d_k \times d}$ and the output projection $F_{O,t} \in \mathbb{R}^{d\times d}$, which correspond to the ML convention parameter matrices via $F_{Q,t}^{(h)} = (W_{Q,t}^{(h)})^{\top}, F_{K,t}^{(h)} = (W_{K,t}^{(h)})^{\top},F_{V,t}^{(h)} = (W_{V,t}^{(h)})^{\top}$ and $F_{O,t} = (W_{O,t})^{\top} $. With the attention matrix
\begin{equation*}
    A_t^{(h)} = \text{softmax} \left( \frac{(F_{Q,t}^{(h)} \bm{X}_t) ^{\top} (F_{K,t}^{(h)}\bm{X}_t)}{\sqrt{d_k} }\right) \in \mathbb{R}^{n \times n},
\end{equation*}
the resulting multi-head attention output is
\begin{equation*}
    \text{MHA}_t (\bm{X_t}) = F_{O,t} \sum_{h=1}^H \bar{F}_{V,t}^{(h)} \bm{X}_t (A_t^{(h)})^{\top}. 
\end{equation*}

The main paper analyzes the spectral properties of a single composite matrix $F_t \in \mathbb{R}^{d \times d}$, obtained by stacking the per-head value projections into a single operator: let $\bar{F}_{V,t} ^{(h)} \in \mathbb{R}^{d\times d}$ be the zero-padded embedding of $F_{V,t}^{(h)}$ into the rows corresponding to head $h$. Then 
\begin{equation*}
    F_t = F_{O,t}F_{V,t} \quad \text{with} \quad  F_{V,t} = \sum_{h=1}^H \bar{F}_{V,t}^{(h)}.
\end{equation*}

This is precisely the matrix $F_t$ defined in the main paper, whose spectral properties (eigenvalues, eigenvectors, etc.) are the object of the analysis in Section~\ref{sec:modelformulation} of the main paper.

\clearpage
\section{Experimental Details for the Transformer Dynamics}
\label{sec:supply_experimental_details}
In the experiments considered in this paper, we do not aim to reproduce exactly in full detail the computational blocks of ALBERT, BERT, GPT2, and GPTNeo. Rather, our goal is to construct a controlled experimental setting containing most of the features of real transformer and yet that is coherent with the discrete-time dynamical system \eqref{eq:update-multihead-auton}. In this model, the update law for each token in a multi-head transformer layer includes the attention block, the skip connection, and a normalization step, while disregarding the FFN. More precisely, at layer $t$, each token is first updated by adding to its current state the multi-head attention output, obtained through the value projections and the output projection, lumped into the ``layer matrix'' $ F_t $ mentioned in the paper. The resulting vectors are then normalized. In our experiments, the original LayerNorm operation is replaced by $L_2$ normalization, so that the layer update directly reflects the projection + normalization mechanism underlying the analogy with the power method. The step size is always unitary: $ \eta =1 $.

More specifically, for ALBERT, BERT, GPT2, and GPTNeo, we disregard the embedding layer and use more than $10^4$ randomly initialized token vectors in the corresponding model dimension $d$ as initial conditions. The initial tokens are drawn from a normal distribution. This avoids the effects of the embedding layer, and allows us to focus on the evolution of the hidden states through the layers. In all four model families, we also disregard the FFN sublayer and its residual branch. This yields the desired experiment setting, in which the dynamics retain the (projection + normalization) components most directly related to the layer matrix $F_t = F_{O,t}F_{V,t}$ and to the power method analogy.

For ALBERT and BERT, the attention sublayer is followed by a skip connection and LayerNorm. In our experiments, we keep the attention update and its skip connection, and replace this LayerNorm with $L_2$ normalization. The FFN sublayer, together with its residual branch, is omitted. 
%This gives a layer update consistent with the projection + normalization mechanism underlying the analogy with the power method.

For GPT2 and GPTNeo, the original transformer blocks adopt a pre-norm architecture. In the experimental setting considered here, we retain the attention sublayer and its skip connection, and disregard the FFN sublayer together with its residual branch. The normalization located before the FFN sublayer, namely the normalization applied to the state after the attention residual update, is used as the normalization step of the dynamics. This normalization is then replaced by an $L_2$ normalization. After this simplification, each simulated GPT2/GPTNeo layer first updates the token vectors through the attention sublayer and the skip connection, and then normalizes the updated states. 
%Hence, the resulting layer update is consistent with model behavior underlying the power method analogy studied in this paper.

In summary, although ALBERT, BERT, GPT2, and GPTNeo differ in their original architectural details, our experiments reduce all of them to the same basic layer update: an attention based value and output update with a skip connection, followed by $L_2$ normalization. This makes the experiments comparable across model families and isolates the geometric effect associated with the ``layer matrix'' $F_t$.

For completeness, in Fig.~\ref{fig:albert_real_llm_corr} and Fig.~\ref{fig:psi_bert_gpt} we show the quantities considered in the paper, $ \gamma_t$, $ \psi_t$ and $ \rho_t $, also for the complete transformers in their original architecture, including embedding, FFN sublayers, and true (affine) LayerNorm.

\clearpage

\clearpage
\section{Supplementary figures and tables}

\begin{table}[!htbp]
\centering
\small
\caption{Summary of the architecture of the transformers analyzed in this paper.
$d$ is the model dimension (i.e., $ \bx_i \in \mathbb{R}^d $), $H$ the number of attention heads, and $T$ the
number of layers.}
\label{tab:model_comparison}
\begin{tabular}{@{}llcccccc@{}}
\toprule
\textbf{Model} & \textbf{Architecture} & \textbf{Attention}
  & \textbf{Shared} & $d$ & $H$ & $T$ \\
\midrule
\multicolumn{7}{@{}l}{\textit{Shared-weight transformers (autonomous dynamics)}} \\
\midrule
ALBERT-base    & Encoder & Full   & \ding{51} & 768  & 12 & 12 \\
ALBERT-large   & Encoder & Full   & \ding{51} & 1024 & 16 & 24 \\
ALBERT-xlarge  & Encoder & Full   & \ding{51} & 2048 & 16 & 24 \\
ALBERT-xxlarge & Encoder & Full   & \ding{51} & 4096 & 64 & 12 \\
\midrule
\multicolumn{7}{@{}l}{\textit{Layer-varying transformers (non-autonomous dynamics)}} \\
\midrule
BERT-base      & Encoder & Full   & \ding{55} & 768  & 12 & 12 \\
BERT-large     & Encoder & Full   & \ding{55} & 1024 & 16 & 24 \\
GPT2 small    & Decoder & Causal & \ding{55} & 768  & 12 & 12 \\
GPT2 medium   & Decoder & Causal & \ding{55} & 1024 & 16 & 24 \\
GPT2 large    & Decoder & Causal & \ding{55} & 1280 & 20 & 36 \\
GPT2 xlarge       & Decoder & Causal & \ding{55} & 1600 & 25 & 48 \\
GPTNeo 125M   & Decoder & Causal\textsuperscript{$\dagger$} & \ding{55} & 768  & 12 & 12 \\
GPTNeo 1.3B   & Decoder & Causal\textsuperscript{$\dagger$} & \ding{55} & 2048 & 16 & 24 \\
GPTNeo 2.7B   & Decoder & Causal\textsuperscript{$\dagger$} & \ding{55} & 2560 & 20 & 32 \\
\bottomrule
\end{tabular}
%\vspace{2pt}
\begin{minipage}{\textwidth}
\footnotesize
All models use (variants of) LayerNorm: ALBERT and BERT use post-norm, GPT-2 and GPTNeo use pre-norm. \textsuperscript{$\dagger$}GPT-Neo alternates global causal attention with local (windowed) causal attention across layers.
\end{minipage}
\end{table}

\begin{figure*}[ht]
     \centering   
      \subfigure[ALBERT-base]{
        \includegraphics[trim=0cm 7cm 0cm 0cm, clip=true,width=0.4\textwidth]{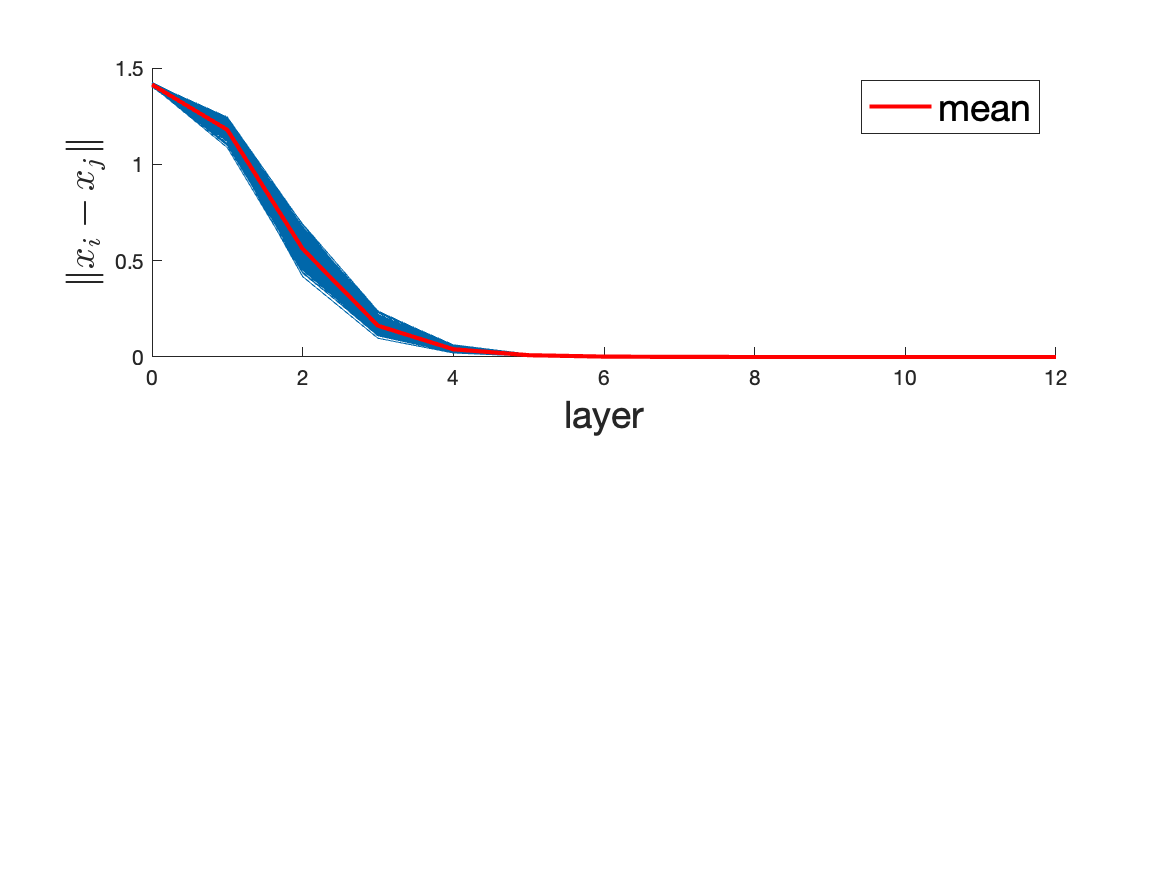}}
      \subfigure[ALBERT-large]{
        \includegraphics[trim=0cm 7cm 0cm 0cm, clip=true,width=0.4\textwidth]{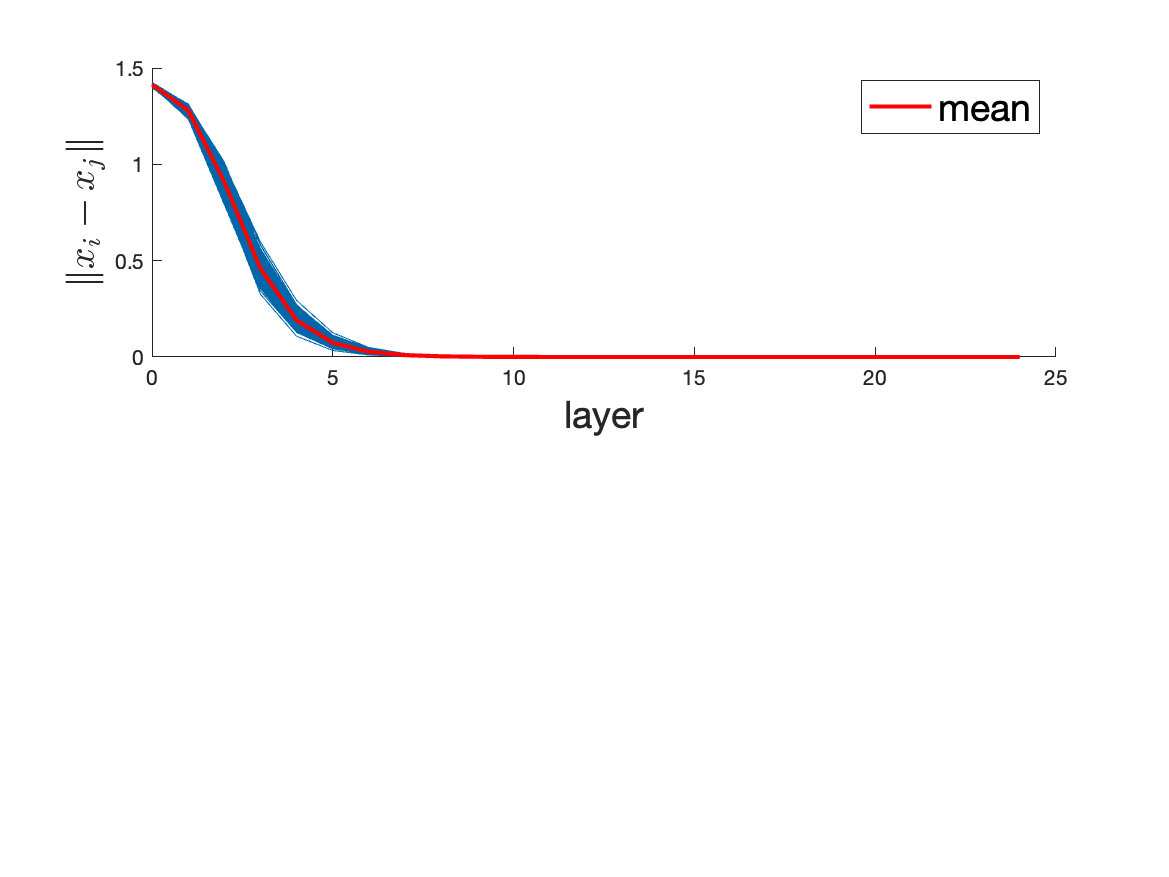}}
      \subfigure[ALBERT-xlarge]{
        \includegraphics[trim=0cm 0cm 0cm 0cm, clip=true,width=0.4\textwidth]{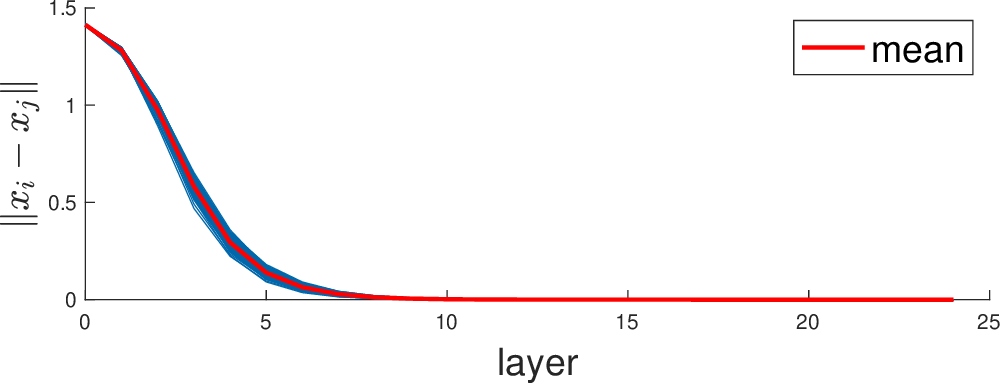}}
      \subfigure[ALBERT-xxlarge]{
        \includegraphics[trim=0cm 0cm 0cm 0cm, clip=true,width=0.4\textwidth]{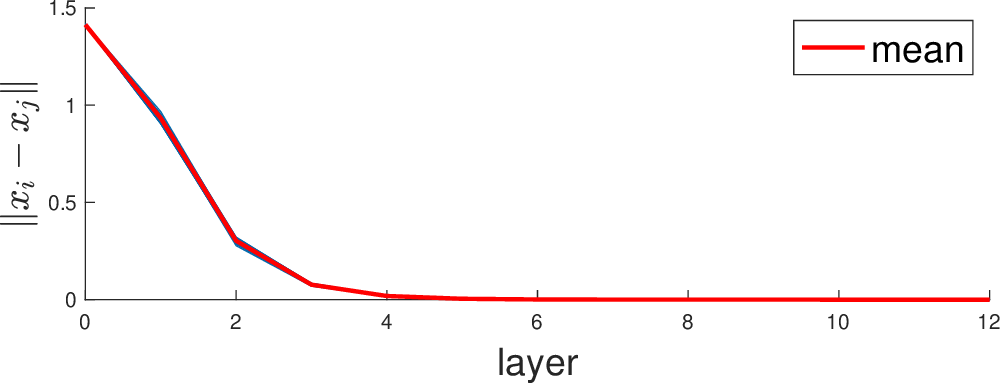}}
        \caption{Intertoken distance in the four ALBERT transformers, computed as $ \| \bx_i (t) - \bx_j(t)\|$. In each of them convergence to 0 (i.e., consensus among the tokens) occurs very fast.}
        \label{fig:Albert2}
\end{figure*}

\begin{figure*}[ht]
     \centering   
      \subfigure[ALBERT-base]{
        \includegraphics[trim=0cm 0cm 0cm 0cm, clip=true,width=0.4\textwidth]{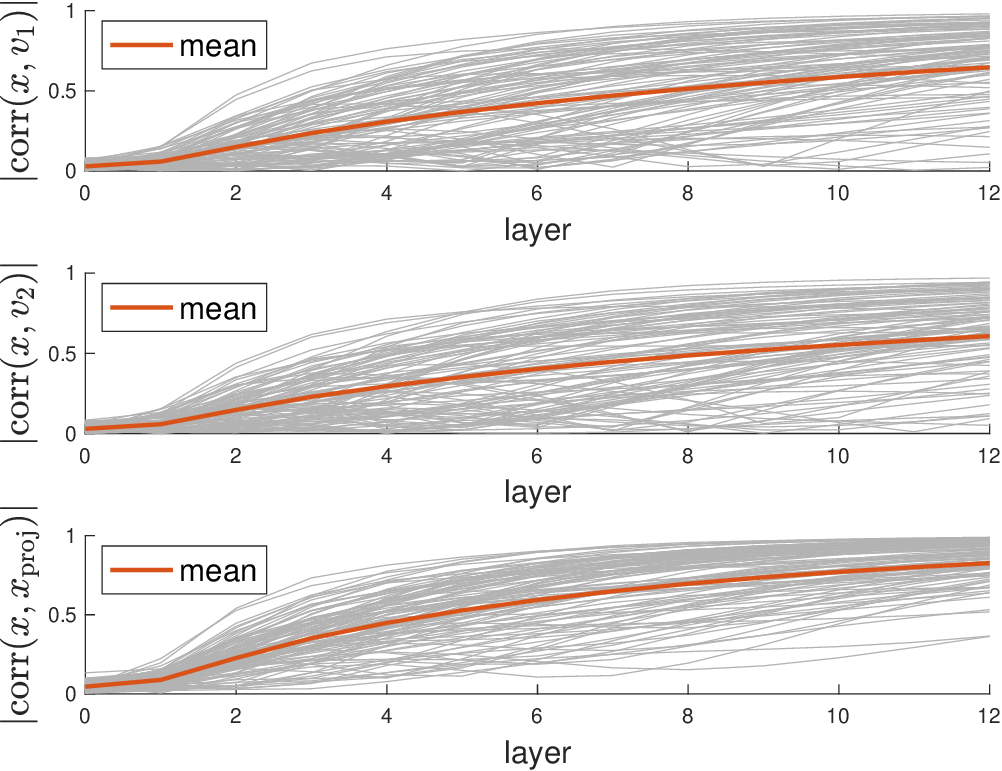}}
      \subfigure[ALBERT-large]{
        \includegraphics[trim=0cm 0cm 0cm 0cm, clip=true,width=0.4\textwidth]{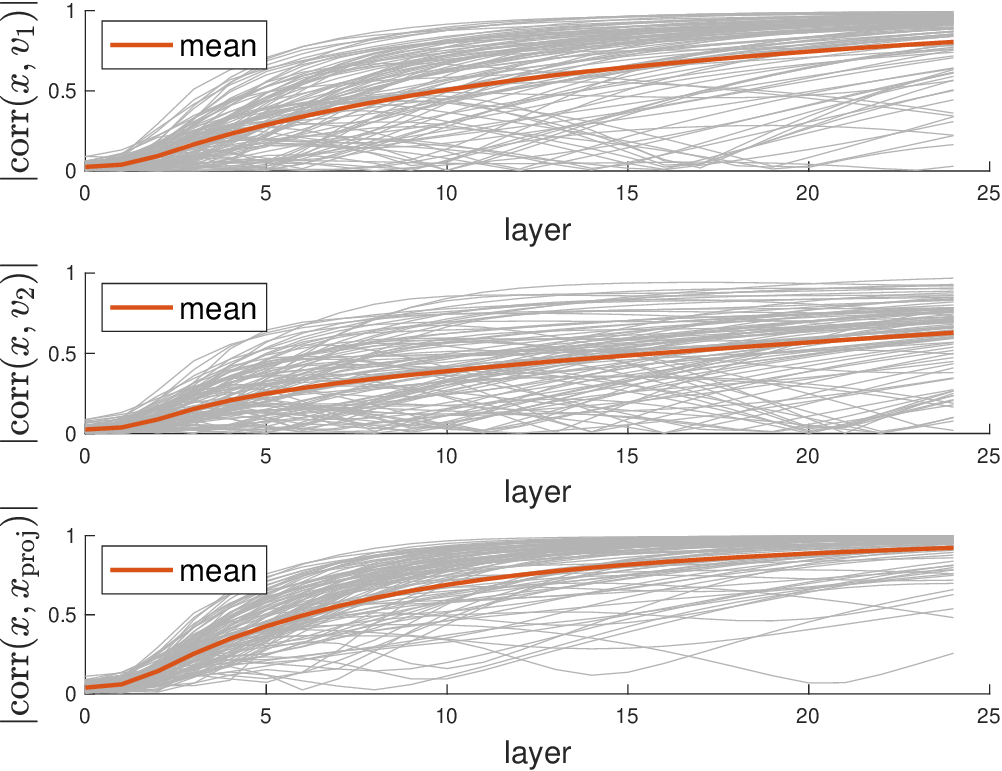}}
      \subfigure[ALBERT-xlarge]{
        \includegraphics[trim=0cm 0cm 0cm 0cm, clip=true,width=0.4\textwidth]{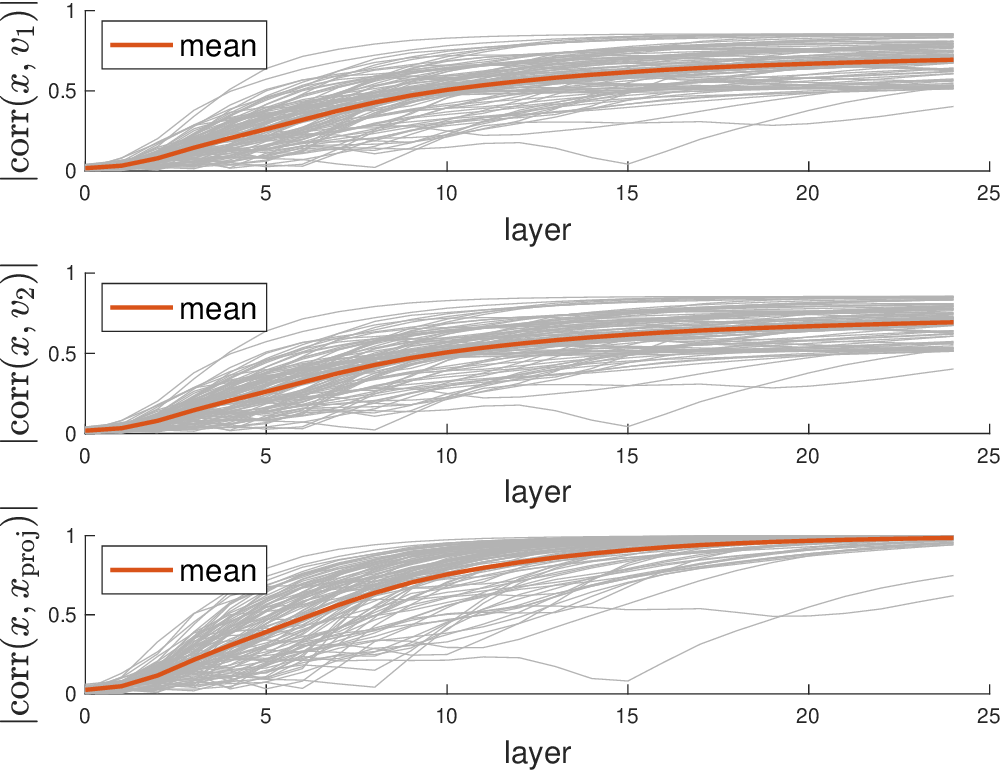}}
      \subfigure[ALBERT-xxlarge]{
        \includegraphics[trim=0cm 0cm 0cm 0cm, clip=true,width=0.4\textwidth]{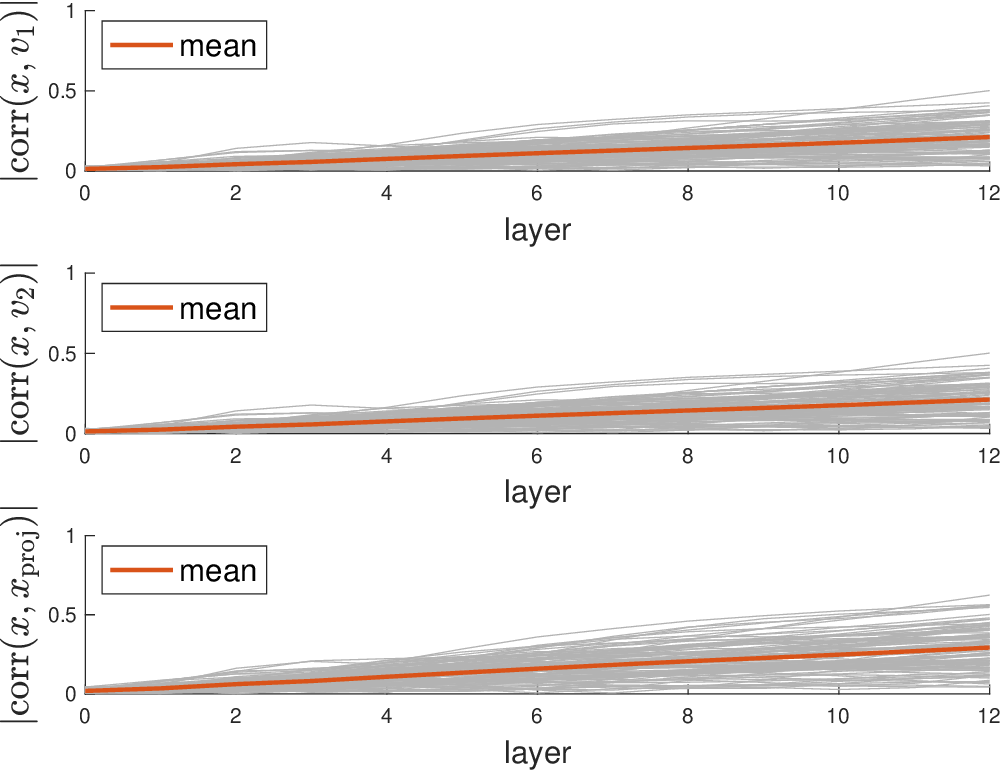}}
        \caption{Correlation $ \gamma $ between tokens and principal eigenvectors in the four ALBERT transformers. In each figure, 100 prompts of 10 token each are shown in gray. The red line is the average. The top panel shows $ |{\rm corr}(\bx_i, \bv_1)|$, the second $ |{\rm corr}(\bx_i, \bv_2)|$ and the third $ |{\rm corr}(\bx_i, \bx_{i,proj})|$.}
        \label{fig:Albert3}
\end{figure*}

% \begin{figure*}[ht]
%      \centering   
%       \subfigure[ALBERT-base]{
%         \includegraphics[trim=0cm 0cm 0cm 0cm, clip=true,width=0.4\textwidth]{Figures/fig_albert_base_v2_corr2_multiple.eps}}
%       \subfigure[ALBERT-large]{
%         \includegraphics[trim=0cm 0cm 0cm 0cm, clip=true,width=0.4\textwidth]{Figures/fig_albert_large_v2_corr2_multiple.eps}}
%       \subfigure[ALBERT-xlarge]{
%         \includegraphics[trim=0cm 0cm 0cm 0cm, clip=true,width=0.4\textwidth]{Figures/fig_albert_xlarge_v2_corr2_multiple.eps}}
%       \subfigure[ALBERT-xxlarge]{
%         \includegraphics[trim=0cm 0cm 0cm 0cm, clip=true,width=0.4\textwidth]{Figures/fig_albert_xxlarge_v2_corr2_multiple.eps}}
%         \caption{\rednote{[Not sure it is needed]} SUPPLEMENTARY Same as Fig.~\ref{fig:Albert3} but without the absolute values in the correlations.}
%         \label{fig:Albert3b}
% \end{figure*}

\begin{figure}[ht]
     \centering    
           \subfigure[ALBERT-base]{
  \includegraphics[trim=0cm 7cm 0cm 0cm, clip=true,width=0.4\textwidth]{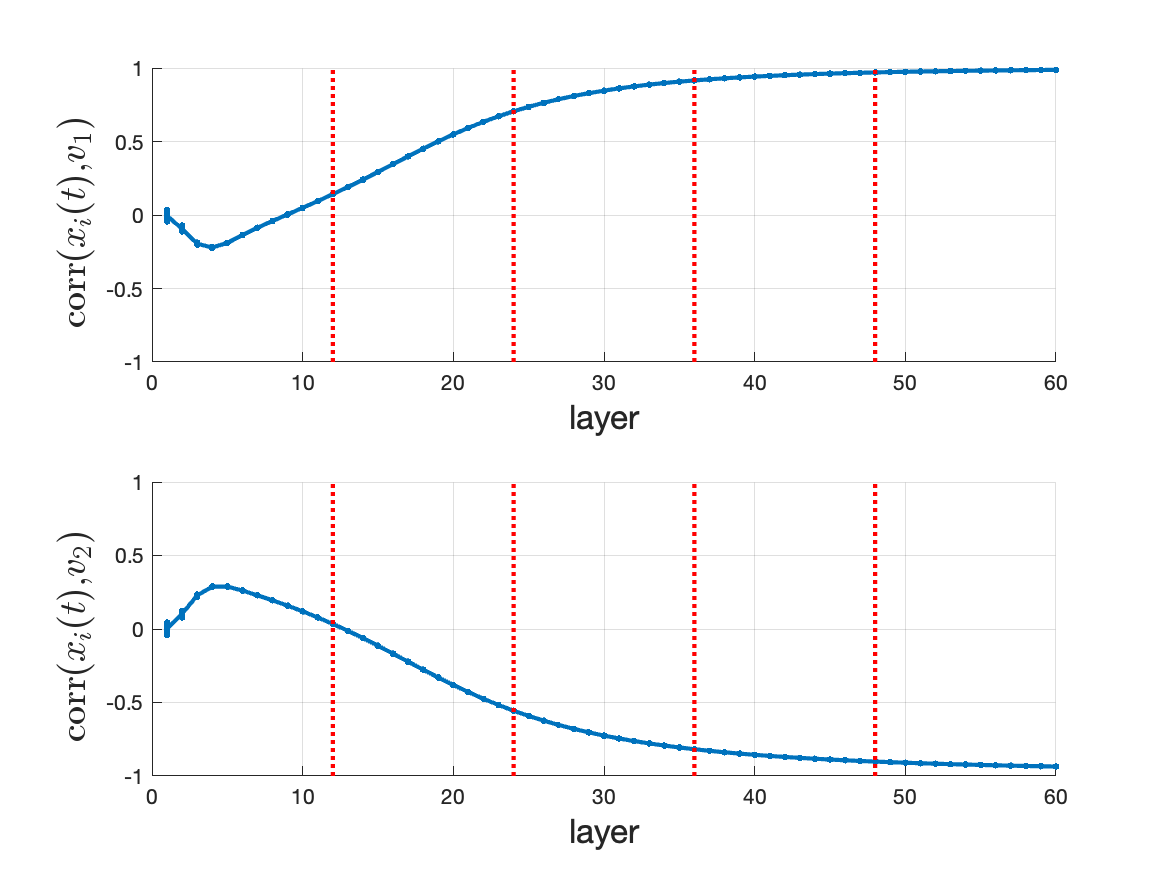}}  
  \subfigure[ALBERT-xxlarge]{
        \includegraphics[trim=0cm 0cm 0cm 0cm, clip=true,width=0.4\textwidth]{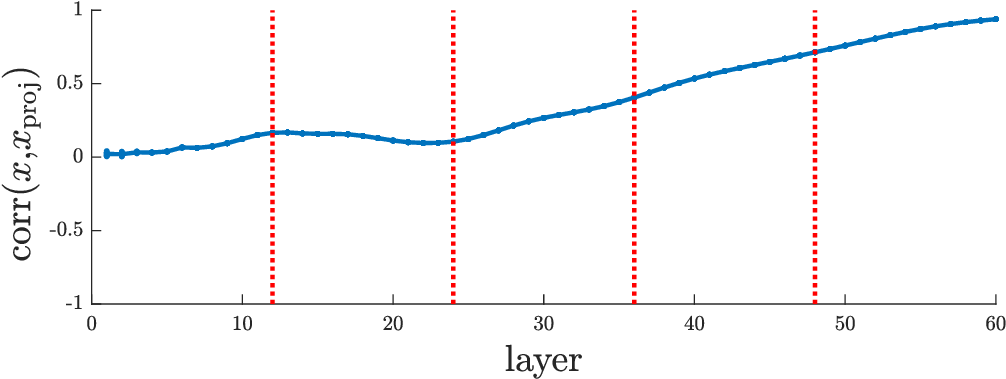}}
        \caption{Artificially prolonging the transformers with extra identical layers, the correlation with $ \bv_1 $ improves. Examples of trajectories from ALBERT-base and ALBERT-xxlarge. }
        \label{fig:Albert4}
\end{figure}

\begin{figure}[!htbp]
    \centering
    \includegraphics[width=\textwidth]{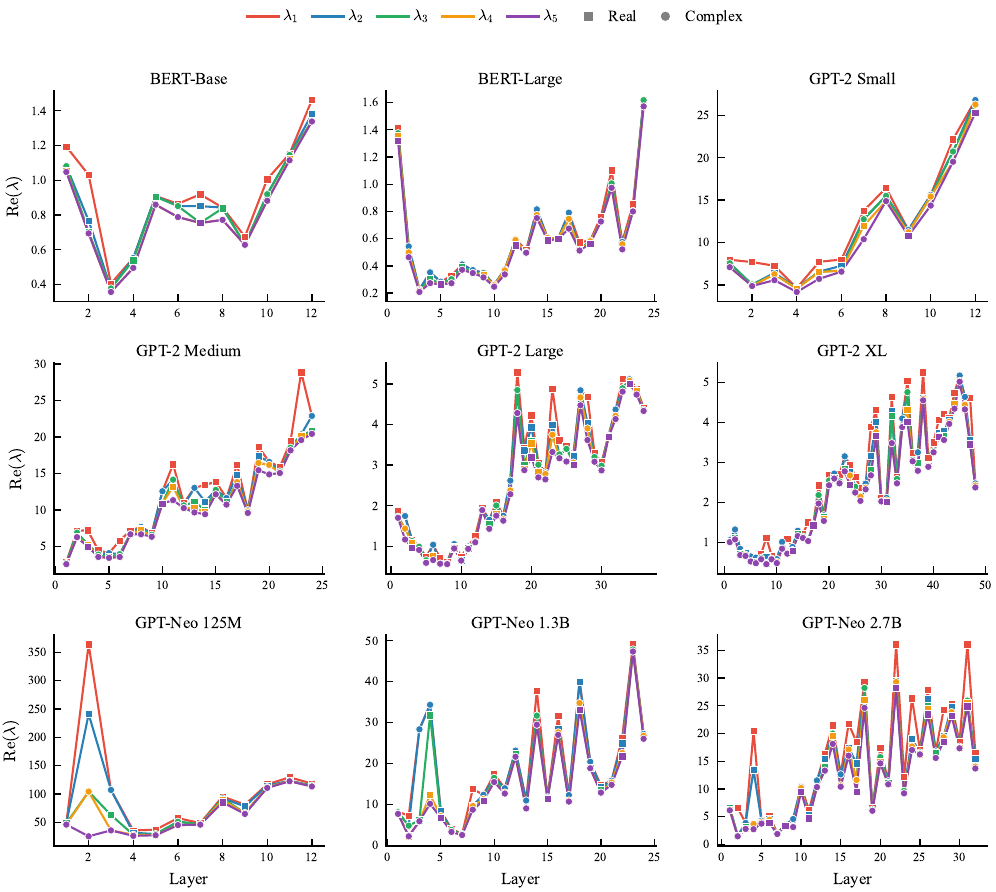}
    \caption{The real part ${\rm Re}(\lambda_{i,t})$, $ i=1, \ldots, 5$, and $ t=1, \ldots, T$, of the first 5 eigenvalues of the layer matrices $ F_t $ for the BERT, GPT2, GPTNeo models investigated in the paper. In the vast majority of cases, the spectral gap $ \lambda_{1,t} - {\rm Re}(\lambda_{2,t})$ (or $ {\rm Re}(\lambda_{1, t}) - {\rm Re}(\lambda_{3,t})$ if $ \lambda_{1,t}$ is complex) is small, meaning that the differences in amplification when the tokens pass through a layer are limited. }
    \label{fig:all_models_re_lambda}
\end{figure}

\newpage
\begin{figure}[htbp]
    \centering
    % Row 1: BERT
        % \includegraphics[width=0.32\textwidth]{psi_figures/BERT_BASE_psi_single.eps}
        \includegraphics[width=0.32\textwidth]{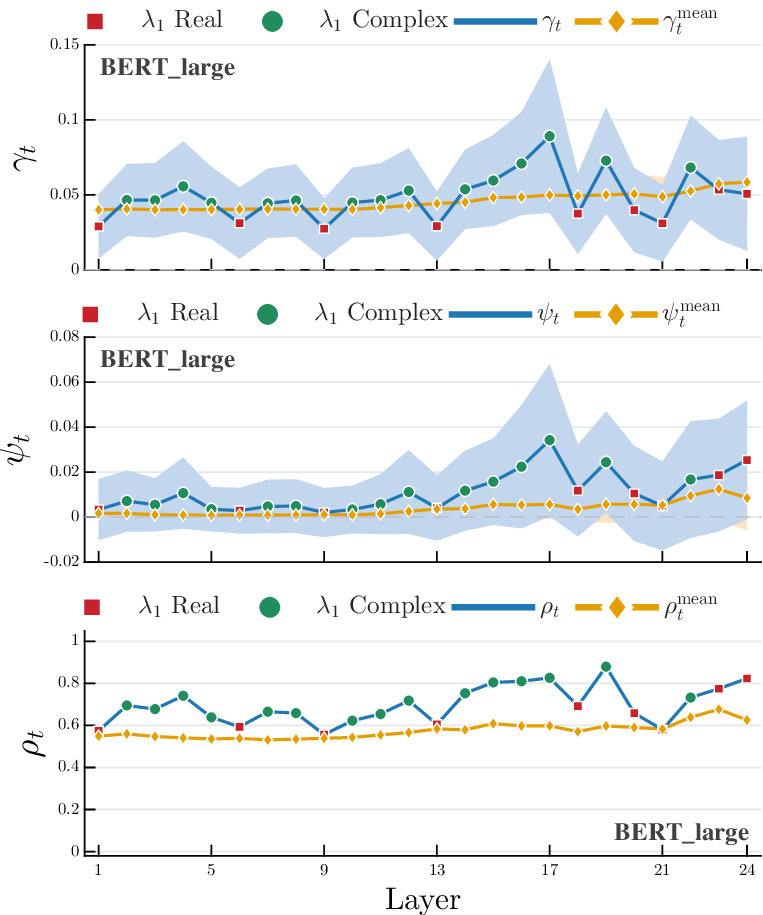}
        \includegraphics[width=0.32\textwidth]{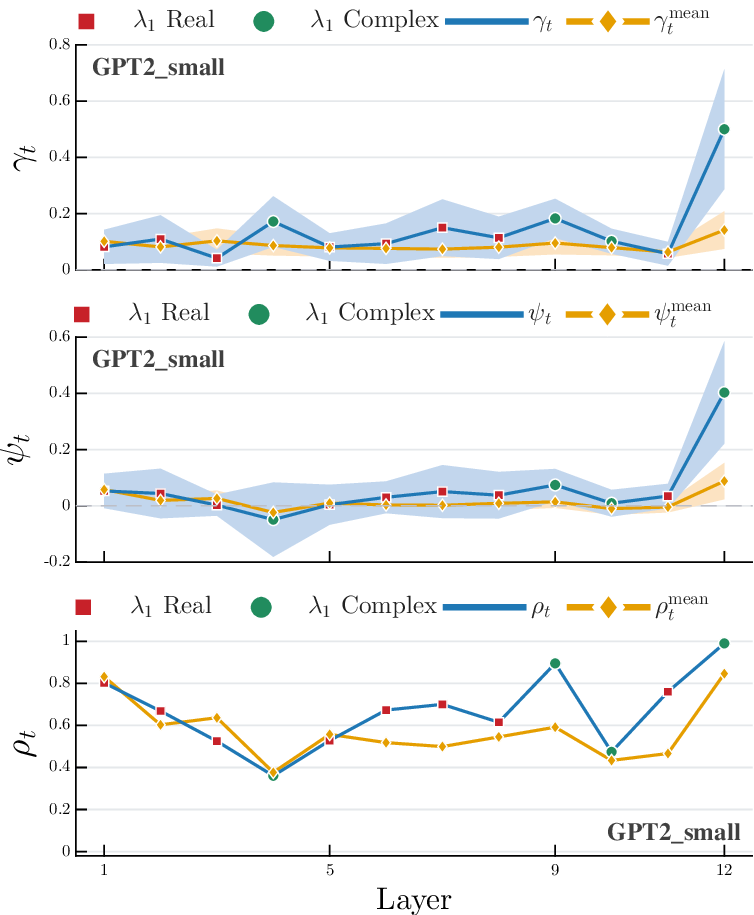}
 %   \vspace{0.5em}
    % Row 2: GPT-2
        % \includegraphics[width=0.32\textwidth]{psi_figures/GPT2_MEDIUM_psi_single.eps}
        \includegraphics[width=0.32\textwidth]{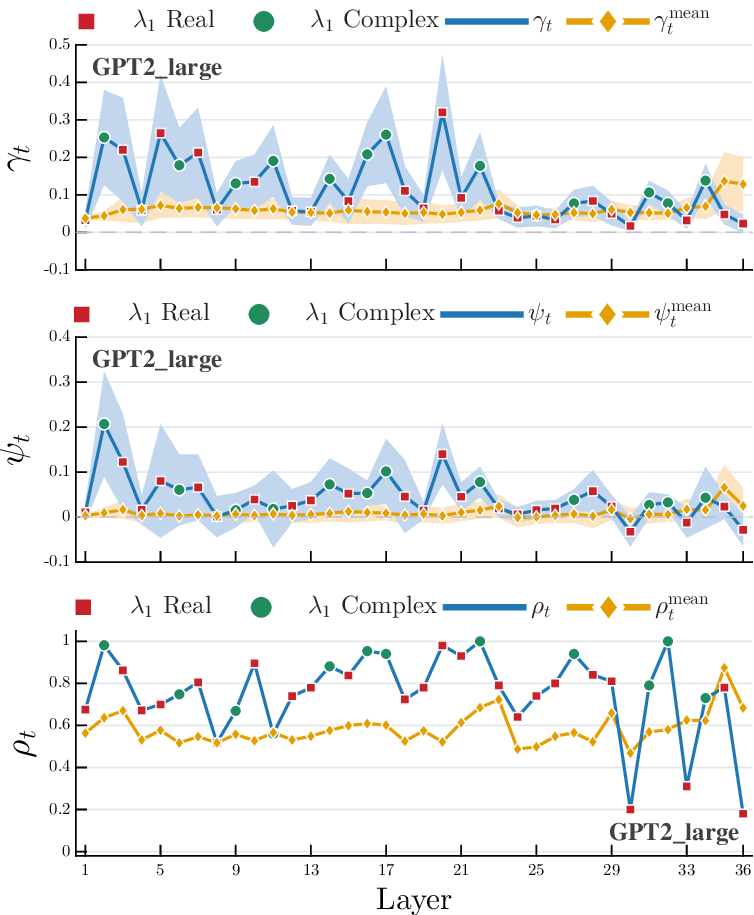}
        \vspace{0.5cm}
        
        \includegraphics[width=0.32\textwidth]{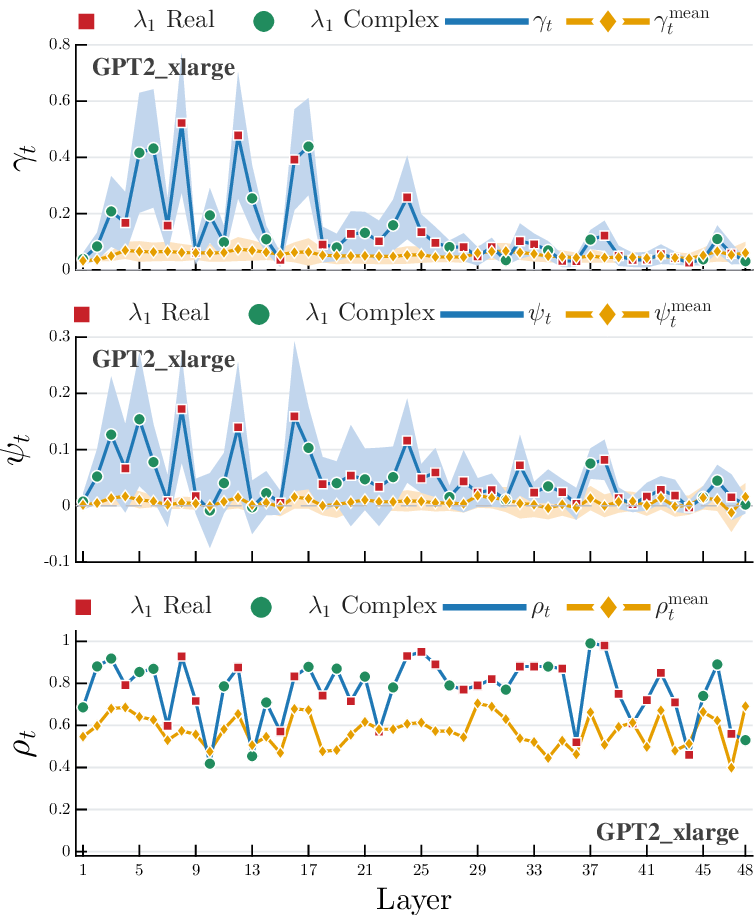}
  %  \vspace{0.5em}
    % Row 3: GPT-Neo
        \includegraphics[width=0.32\textwidth]{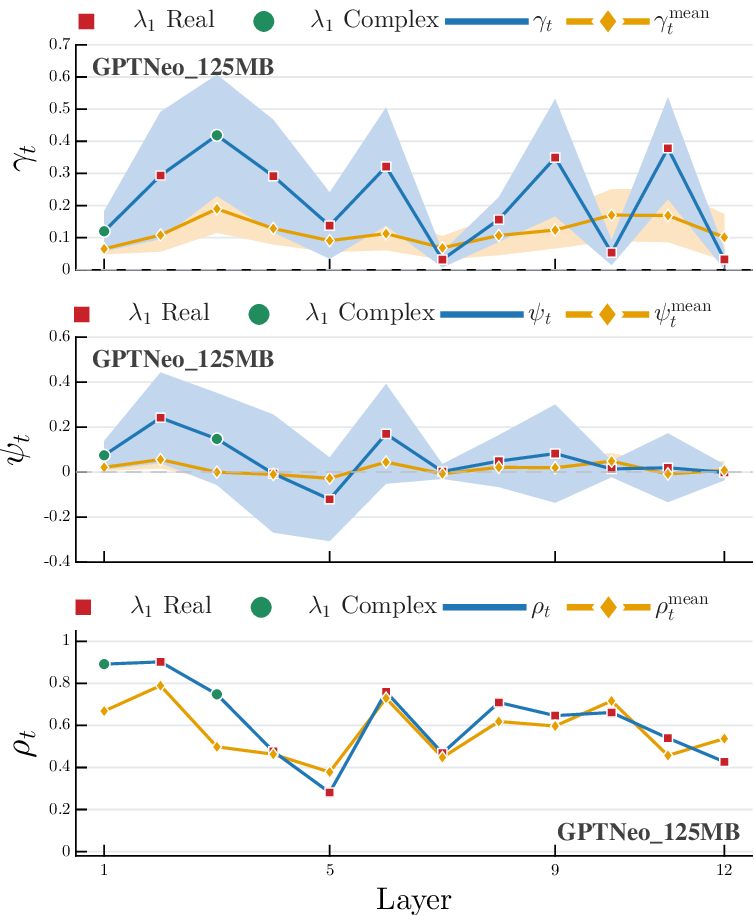}
        \includegraphics[width=0.32\textwidth]{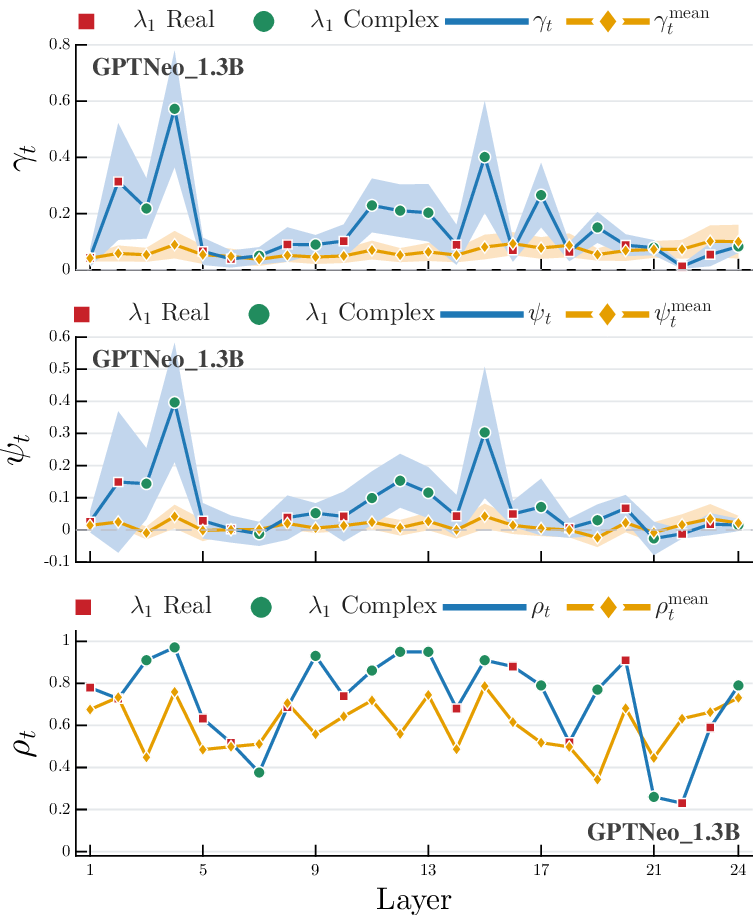}
    \caption{Values of $\gamma_t$ (top row), $\psi_t$ (middle row) and $\rho_t$ (bottom row) across layers for  BERT-large, GPT2 (Small, Large, XL), and GPTNeo (125M, 1.3B). Top: $\gamma_t$ and $\gamma_t^{\rm mean}$ and associated standard deviations over $10^4$ tokens are shown; Middle: $\psi_t$ and $\psi_t^{\rm mean}$ and associated standard deviations over $10^4$ tokens; Bottom: corresponding probabilities $\rho_t$ and $\rho_t^{\rm mean}$.
    See Fig.~\ref{fig:psi_rho_selected} of the paper for the remaining 3 models.}
    \label{fig:psi_single_all}
\end{figure}

% \newpage
% \begin{figure}[!htbp]
%     \centering
%     \begin{subfigure}[b]{0.32\textwidth}
%         \centering
%         \includegraphics[width=\textwidth]{psi_figures/BERT_BASE_F_lambda1.eps}
%     \end{subfigure}
%     \hfill
%     \begin{subfigure}[b]{0.32\textwidth}
%         \centering
%         \includegraphics[width=\textwidth]{psi_figures/BERT_LARGE_F_lambda1.eps}
%     \end{subfigure}
%     \hfill
%     \begin{subfigure}[b]{0.32\textwidth}
%         \centering
%         \includegraphics[width=\textwidth]{psi_figures/GPT2_BASE_F_lambda1.eps}
%     \end{subfigure}
%     \caption{Correlation between output tokens and the dominant eigenvector of $\Pi_{t=1}^TF_t$ across layers. }
%     \label{fig:FC_lambda1}
% \end{figure}
\newpage

\begin{figure}[htbp]
    \centering
    \includegraphics[width=0.32\textwidth]{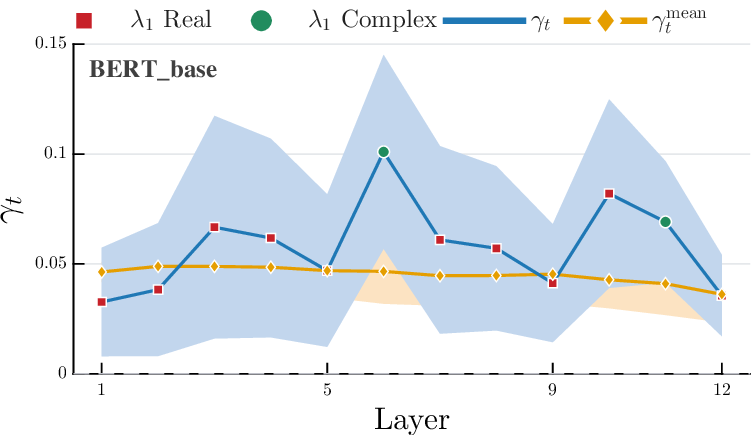}\hfill
    \includegraphics[width=0.32\textwidth]{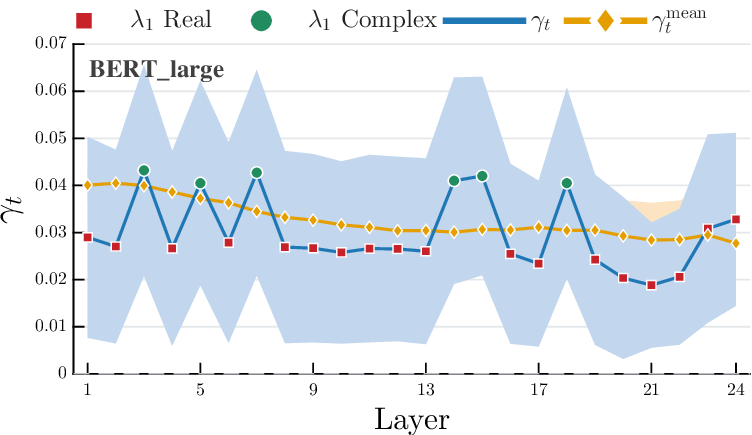}\hfill
    \includegraphics[width=0.32\textwidth]{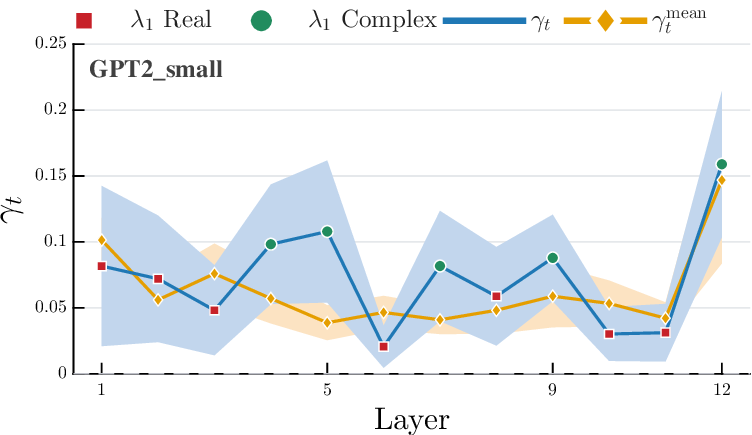}
    \vspace{0.5em}
    \includegraphics[width=0.32\textwidth]{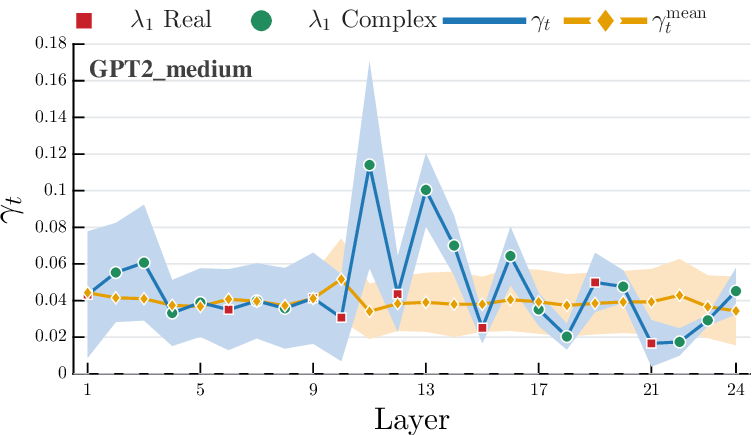}\hfill
    \includegraphics[width=0.32\textwidth]{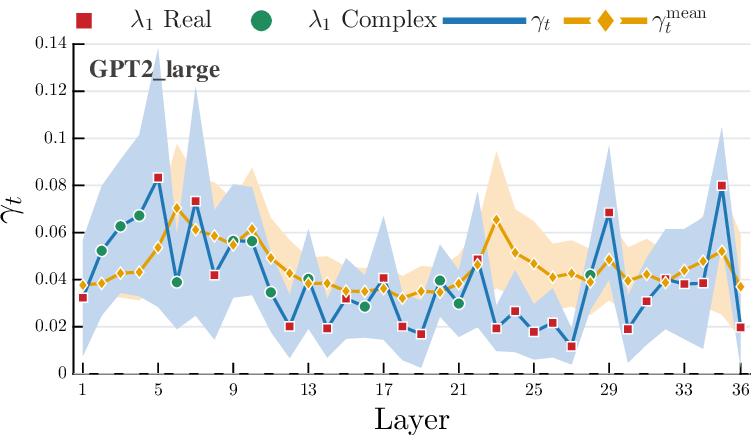}\hfill
    \includegraphics[width=0.32\textwidth]{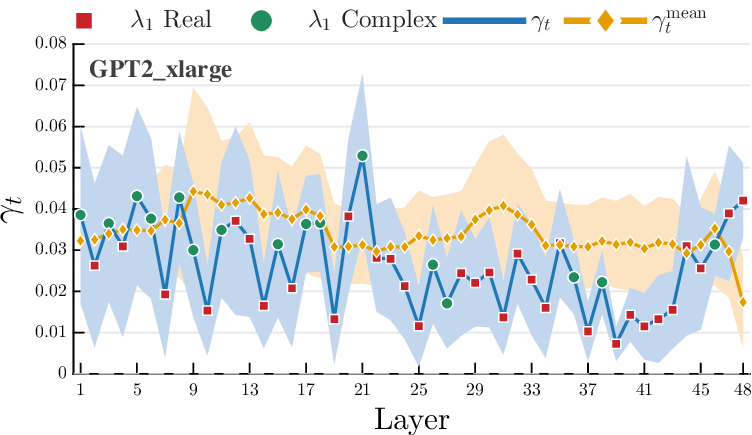}
    \vspace{0.5em}
    \includegraphics[width=0.32\textwidth]{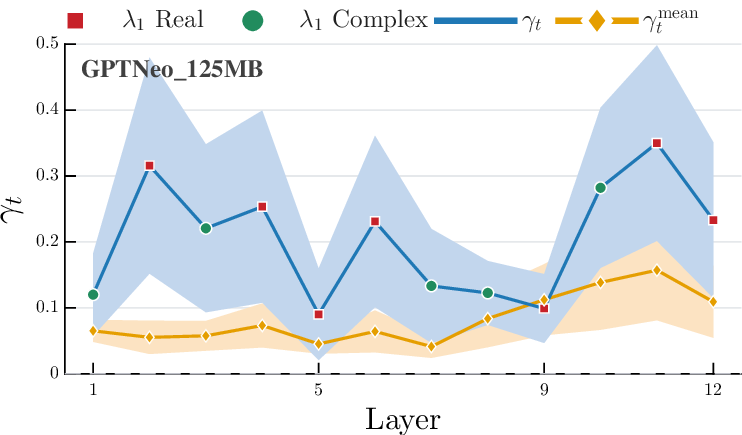}\hfill
    \includegraphics[width=0.32\textwidth]{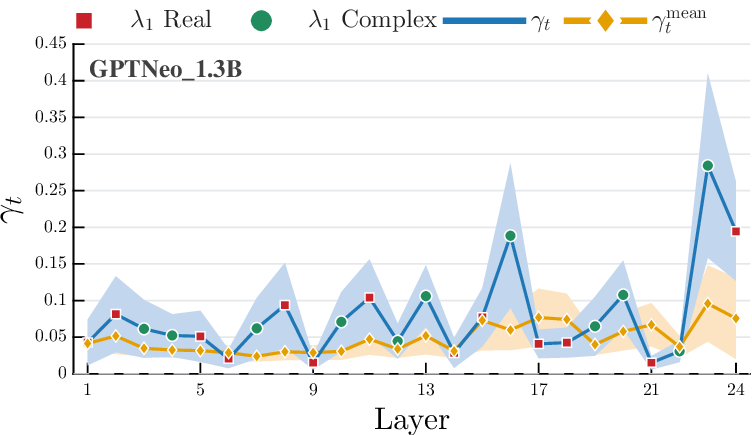}\hfill
    \includegraphics[width=0.32\textwidth]{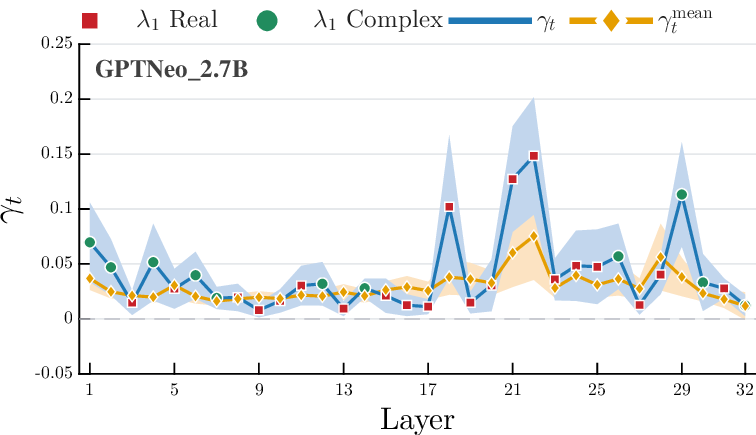}
    \caption{Absolute correlation between the principal eigenvector $ \bv_{1,t}^{\rm cum} $ of the cumulative product $F_t^{\rm cum} = \Pi_{\tau=1}^t F_\tau $ as $ t$ varies from $ 1$ to $T$. The quantity shown is $\gamma_t = \left|\mathrm{corr}(\bx_i{(t+1)},\bv_{1,t}^{\rm cum})\right| $. Except for some models, no clear trend emerges when compared with $\gamma_t^{\rm mean}  = \sum_{j=1}^d  \left|\mathrm{corr}(\bx_i{(t+1)},\bv_{j,t}^{\rm cum})\right| / d $. No growth over $t$ is observed.}
    \label{fig:pi_f_all}
\end{figure}

% \begin{figure}[!htbp]
% \centering
% \includegraphics[width=0.45\textwidth]{Real_LLM_Figures/ALBERT_BASE_psi_real_LLM.eps}
% \includegraphics[width=0.45\textwidth]{Real_LLM_Figures/ALBERT_LARGE_psi_real_LLM.eps}

% \includegraphics[width=0.45\textwidth]{Real_LLM_Figures/ALBERT_XLARGE_psi_real_LLM.eps}
% \includegraphics[width=0.45\textwidth]{Real_LLM_Figures/ALBERT_XXLARGE_psi_real_LLM.eps}
% \caption{ \rednote{[Probably remove]}(everything included)The $\psi$ metric across layers for ALBERT models() (Base, Large, XLarge, XXLarge).}
% \label{fig:psi_albert}
% \end{figure}

\begin{figure}[!htbp]
  \centering
  \includegraphics[width=0.32\textwidth]{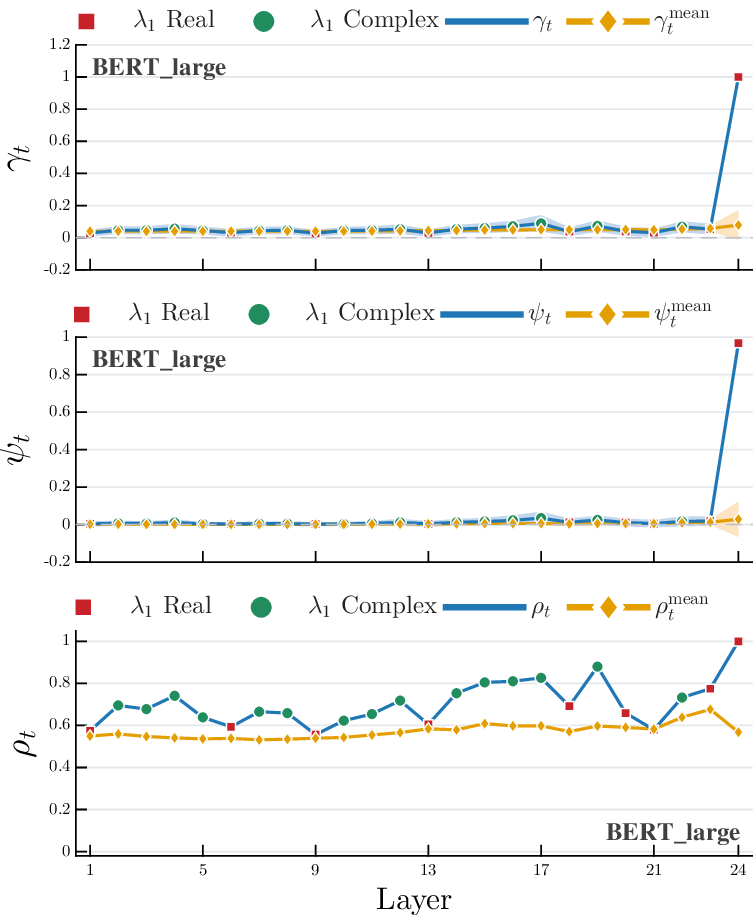}\hfill
  \includegraphics[width=0.32\textwidth]{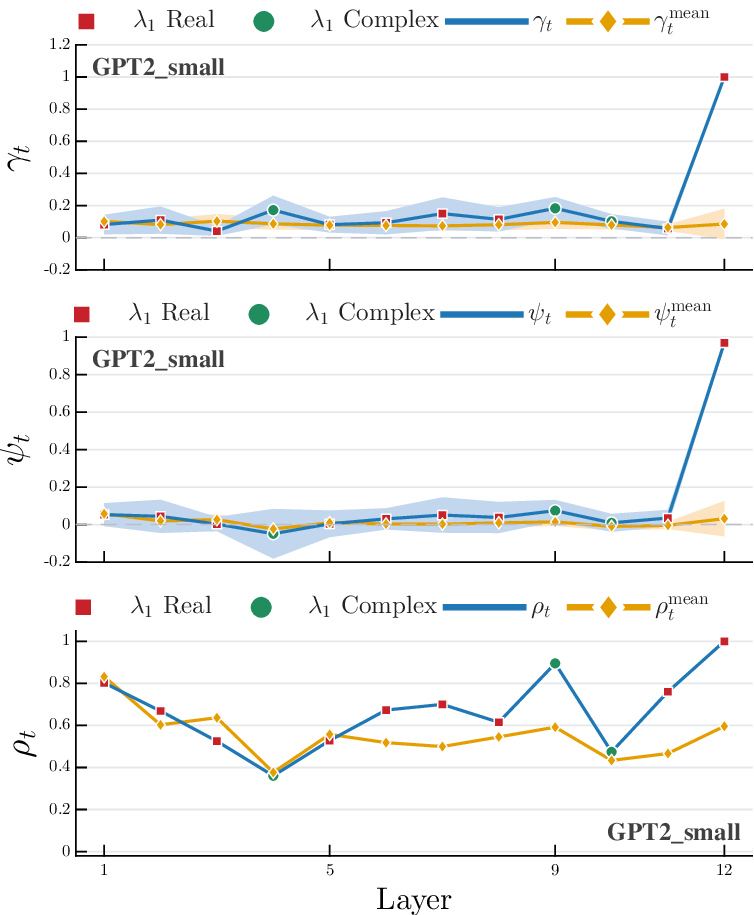}\hfill 
  \includegraphics[width=0.32\textwidth]{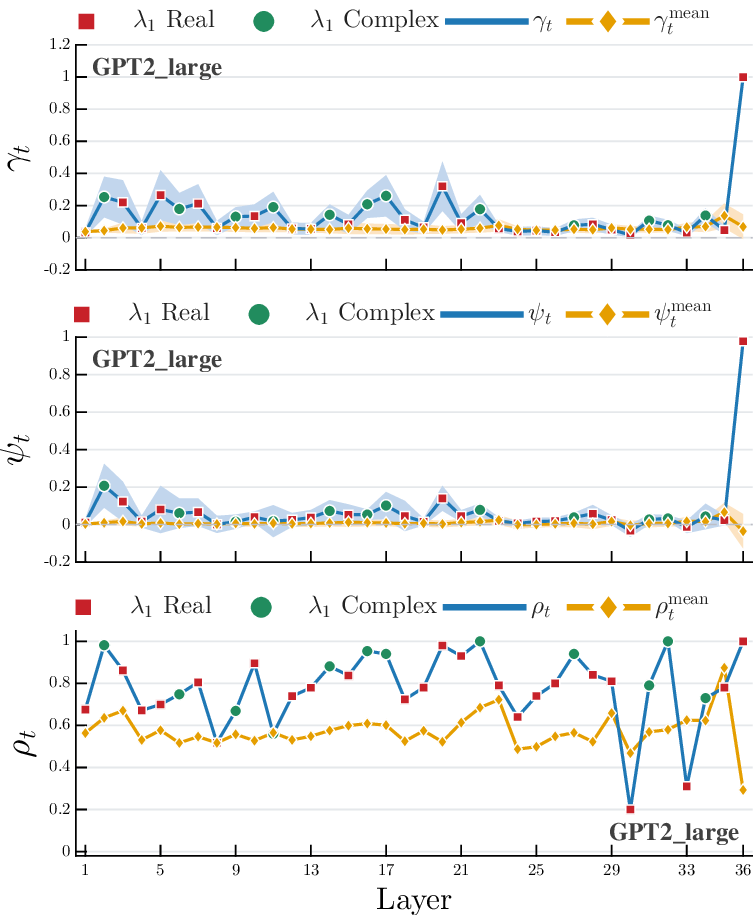}\hfill
  \includegraphics[width=0.32\textwidth]{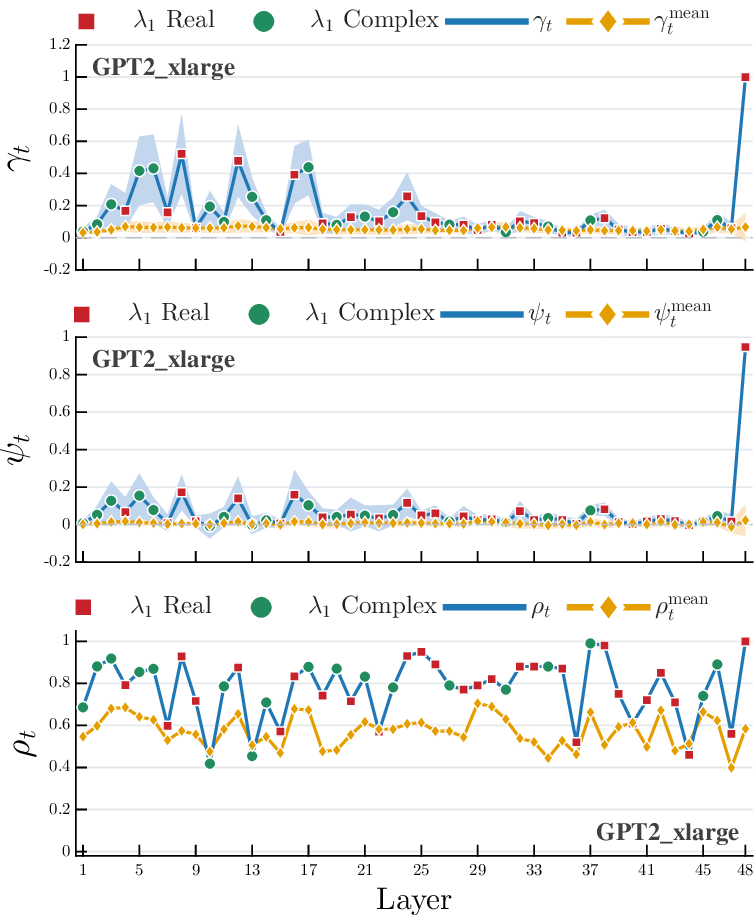} \hfill 
  \includegraphics[width=0.32\textwidth]{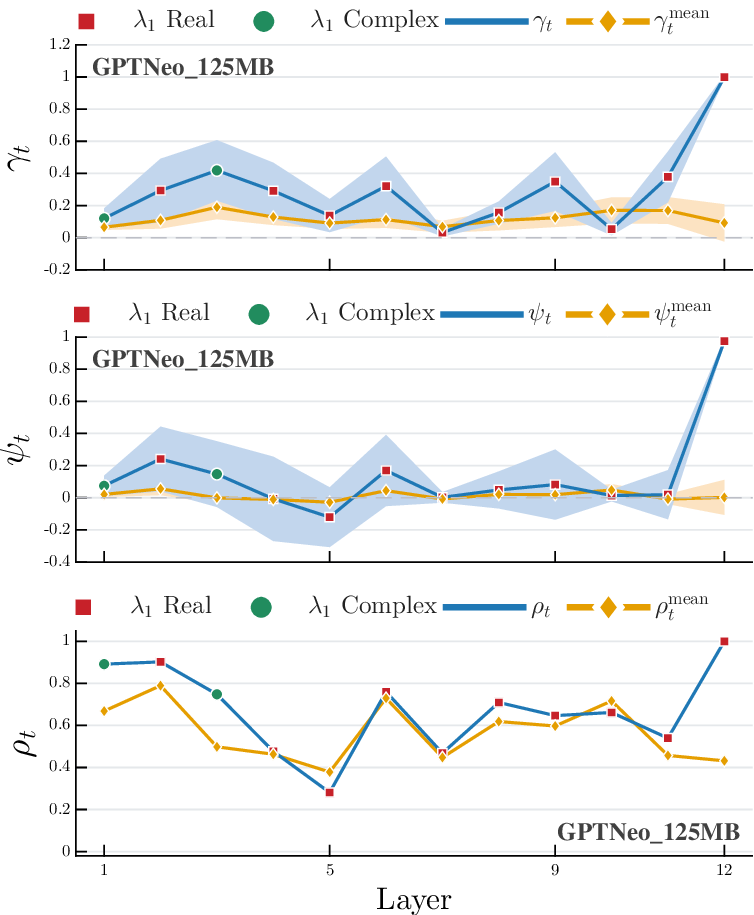}\hfill
  \includegraphics[width=0.32\textwidth]{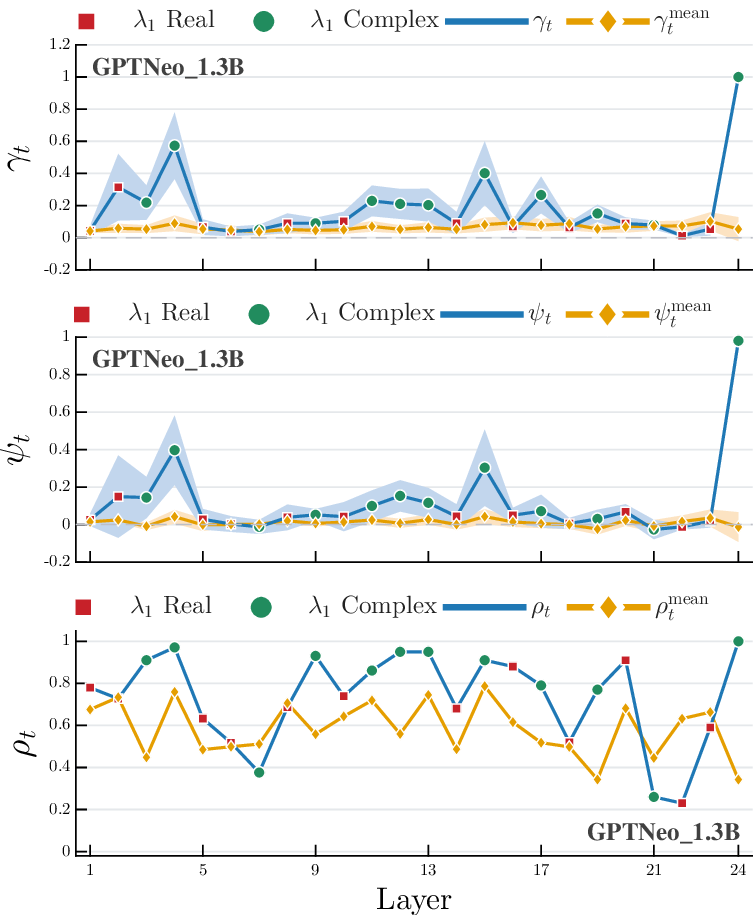}\hfill

  \caption{Values of $ \gamma_t $ (top row), $\psi_t$ (middle row) and $ \rho_t $ (bottom row) across layers for BERT-large, GPT2 (Small, Large, XL), and GPTNeo (125M, 1.3B), with the last layer output matrix $ F_{O,T,{\rm mod}}$ modified as described in Section~\ref{sec:steer-tokens}. See Fig.~\ref{fig:modified_fo_some} in the paper for the other 3 models. Perfect alignment with $ \bv_{1,T,{\rm mod}}$ is  achieved in all transformers.}
  \label{fig:modified_fo_all}
\end{figure}

% (everything included)
\begin{figure}[!htbp]
    \centering
    %ALBERT
    \includegraphics[width=0.48\textwidth]{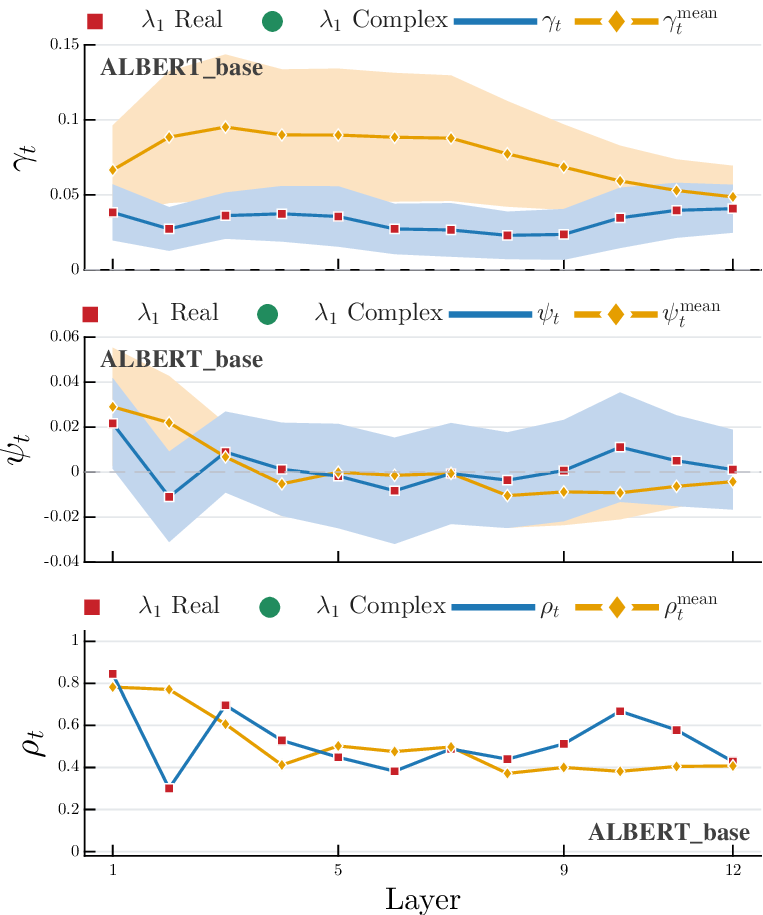}\hfill
    \includegraphics[width=0.48\textwidth]{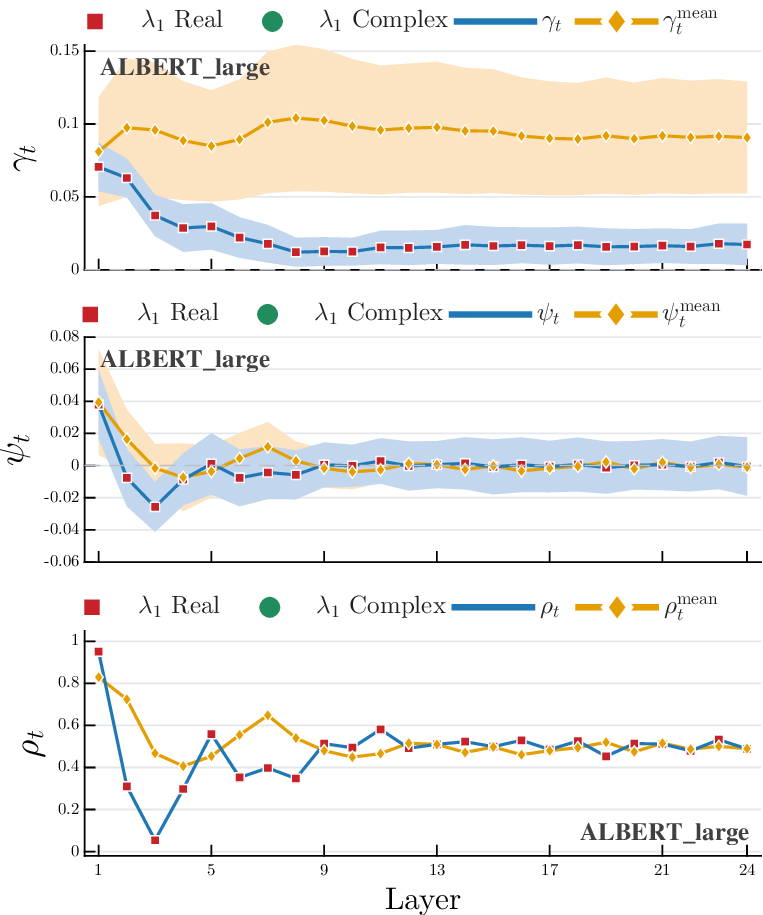}
    %\vspace{0.5em}
    \includegraphics[width=0.48\textwidth]{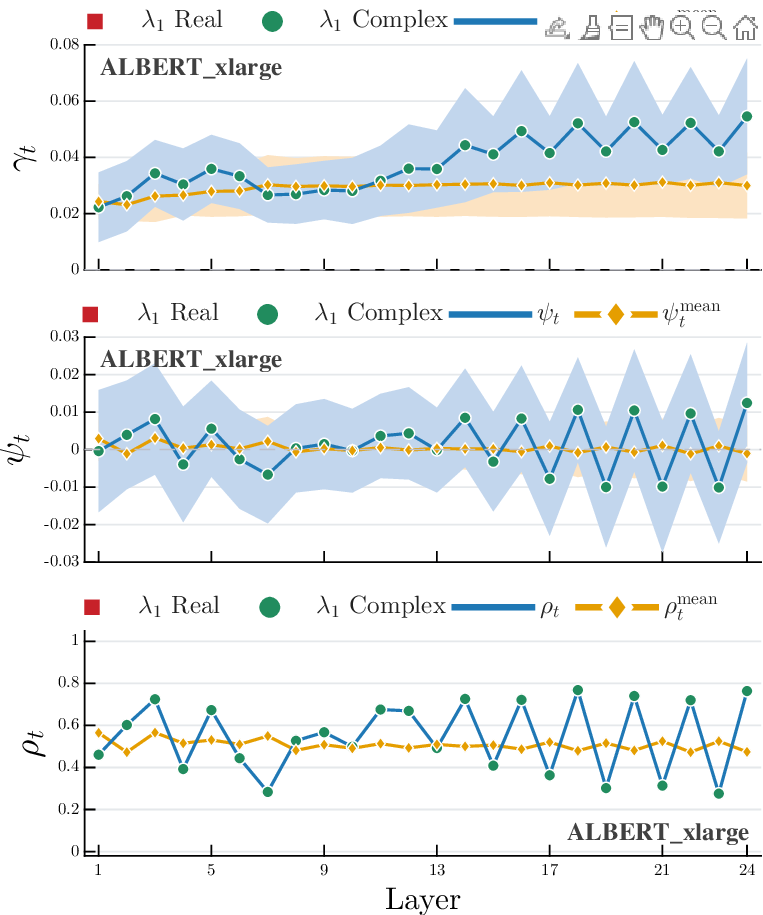}\hfill
    \includegraphics[width=0.48\textwidth]{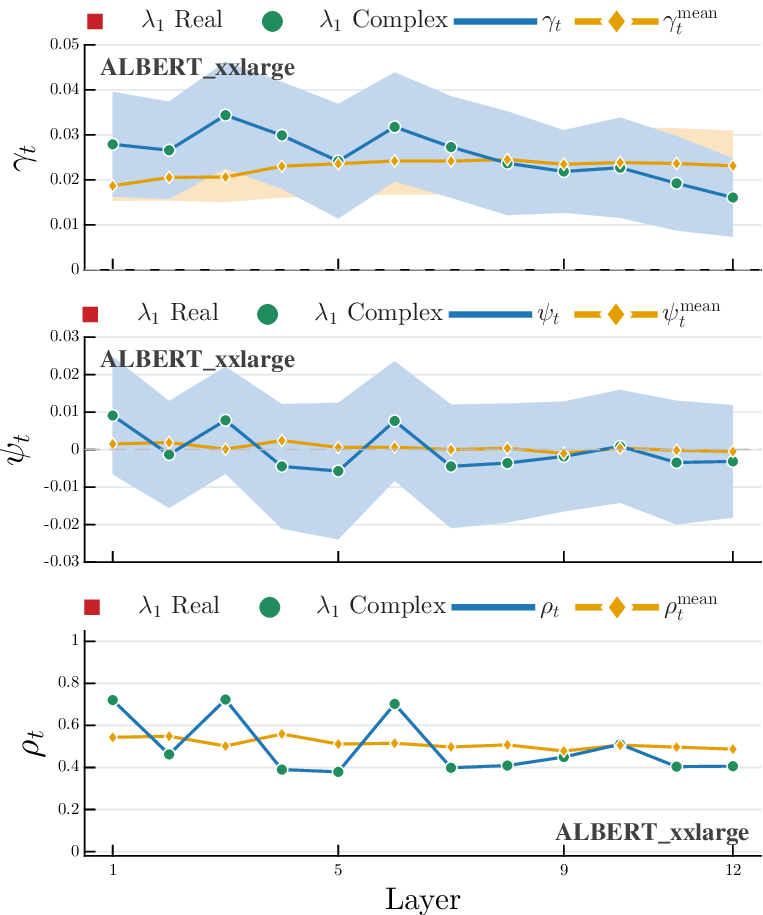}
    
    \caption{Values of $\gamma_t$ (top row), $\psi_t$ (middle row) and $ \rho_t $ (bottom row) across layers for ALBERT (Base, Large, XLarge, XXLarge), when the embedding layer and the FFN at each layer are included, and the original (affine) $ {\rm LayerNorm}(\cdot)$ normalization is used.}
    \label{fig:albert_real_llm_corr}
\end{figure}

% (everything included)
\begin{figure}[!htbp]
\centering
    \centering
    % BERT
    \includegraphics[width=0.32\textwidth]{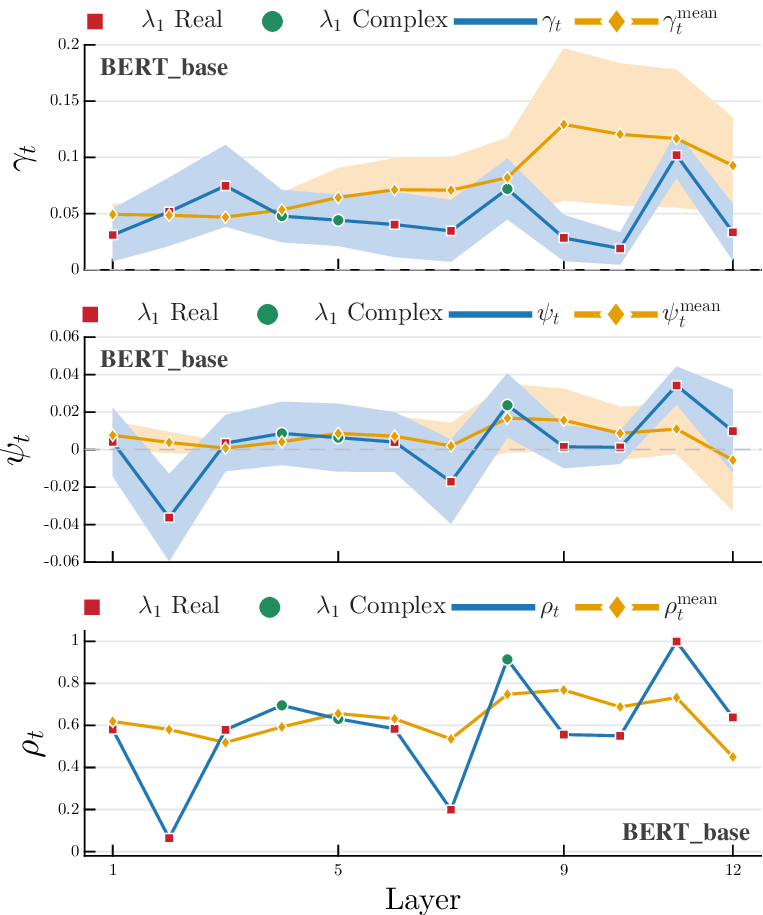}
    \hfill
    \includegraphics[width=0.32\textwidth]{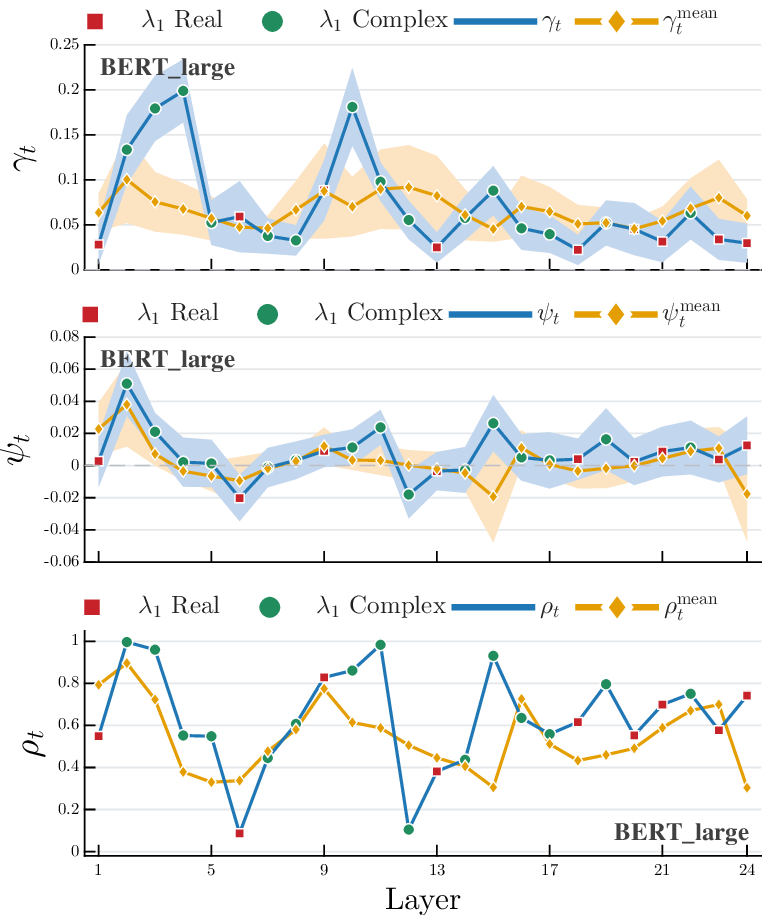}
    \hfill
    \includegraphics[width=0.32\textwidth]{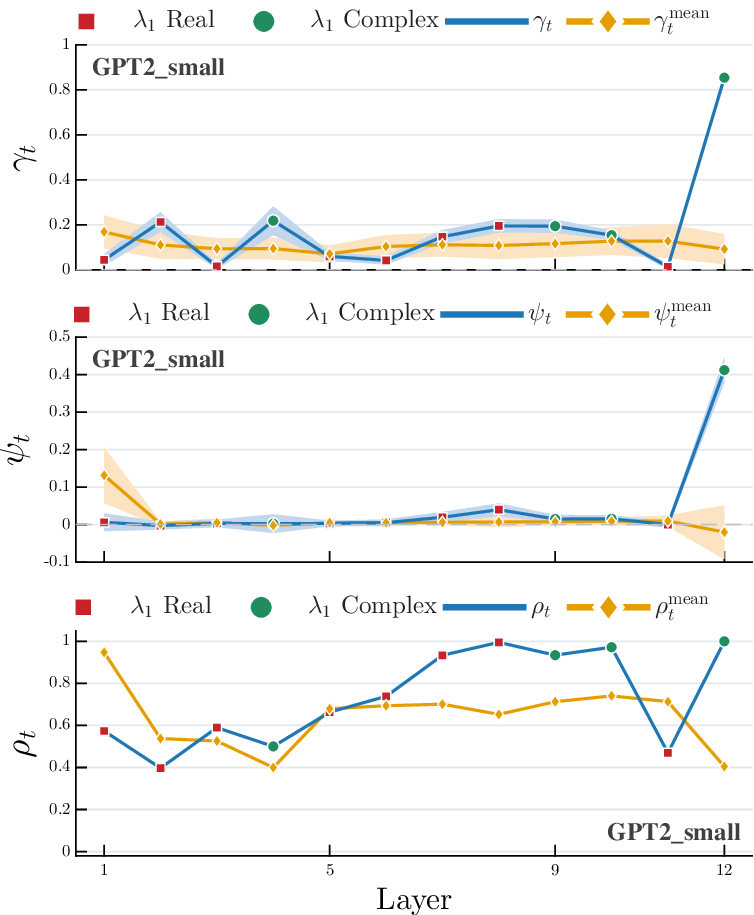}
    
    %\vspace{0.5em}
    %GPT2
    \includegraphics[width=0.32\textwidth]{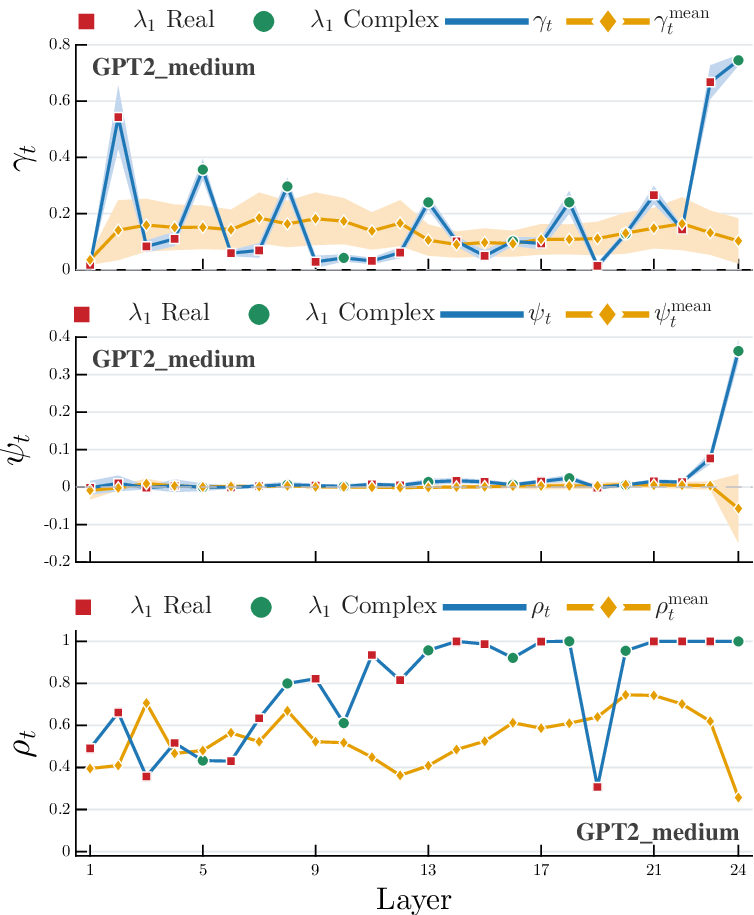}
    \hfill
    \includegraphics[width=0.32\textwidth]{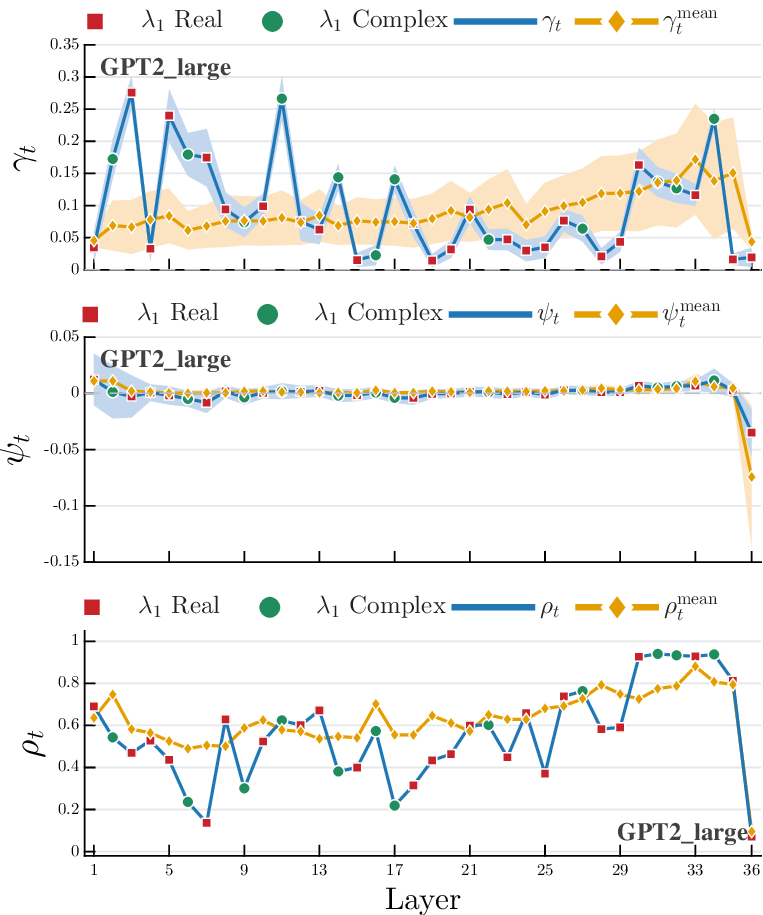}
    \hfill
    \includegraphics[width=0.32\textwidth]{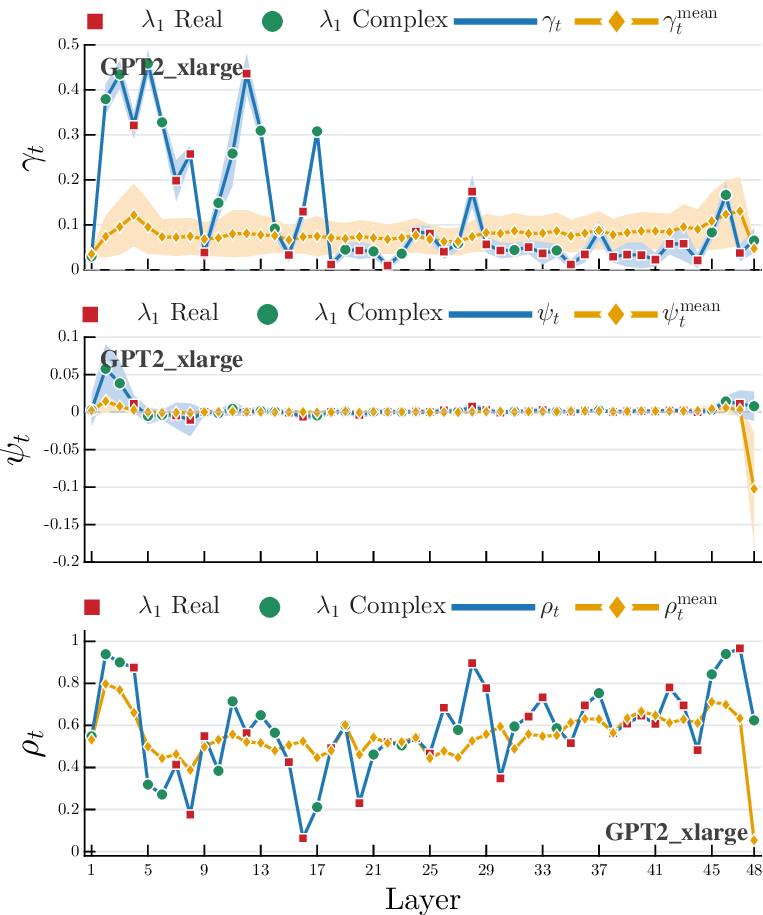}
    
   % \vspace{0.5em}
    %GPTNeo
    \includegraphics[width=0.32\textwidth]{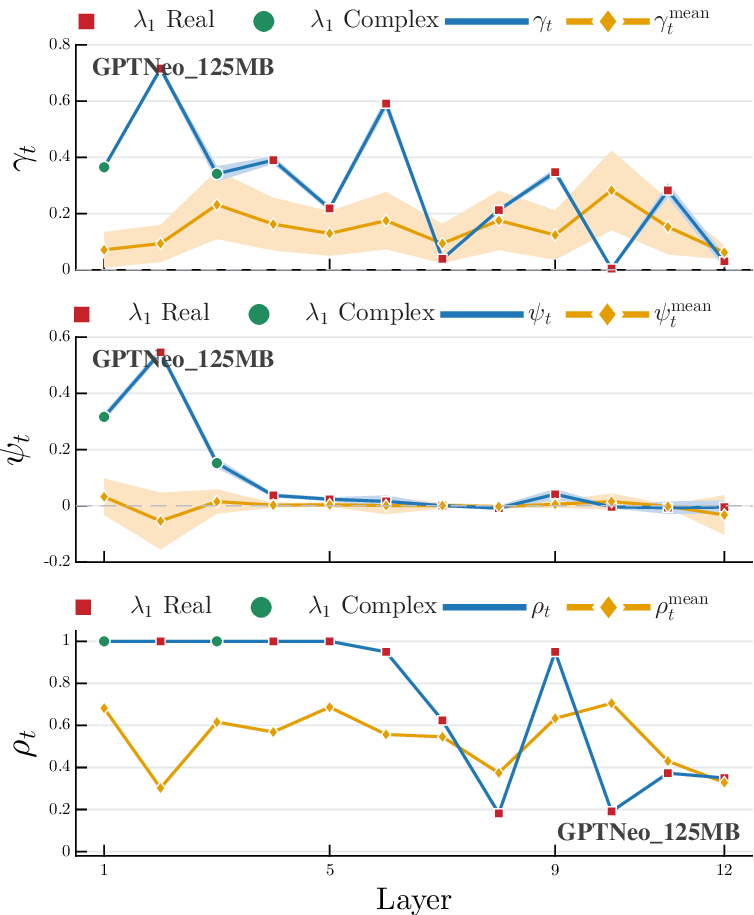}
    \hfill
    \includegraphics[width=0.32\textwidth]{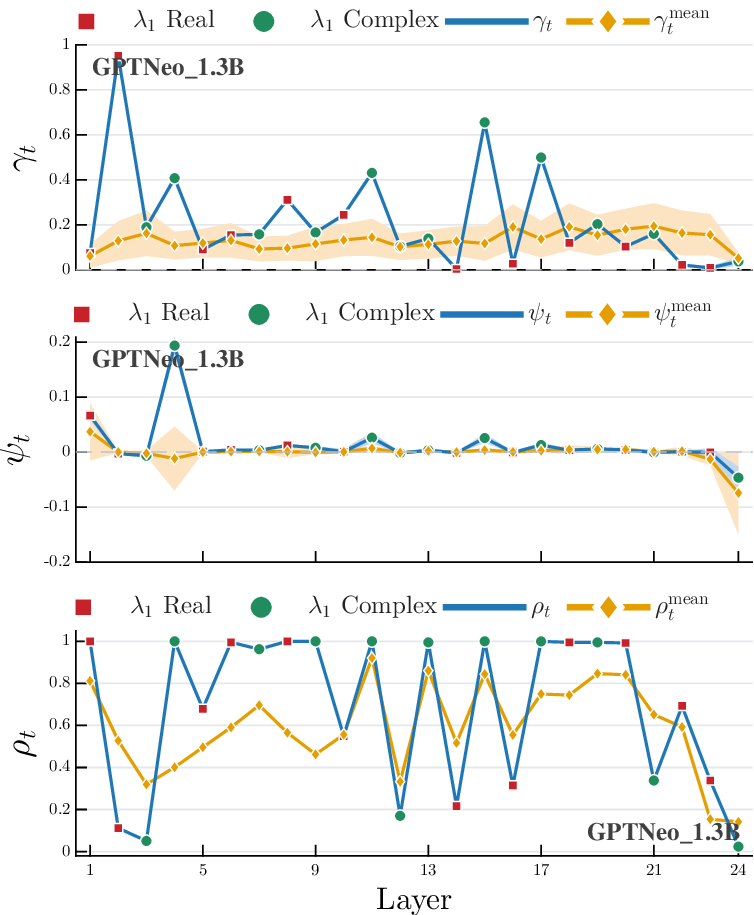}
    \hfill
    \includegraphics[width=0.32\textwidth]{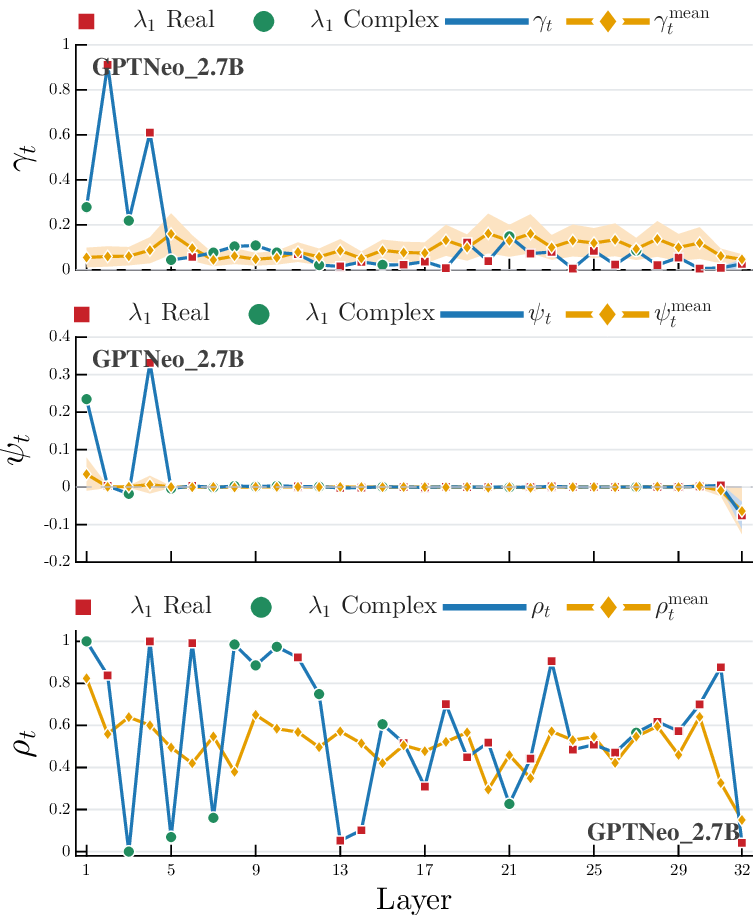}
    
    \caption{Values of $\gamma_t$ (top row), $\psi_t$ (middle row) and $ \rho_t$ (bottom row) across layers for BERT (Base, Large), GPT2 (Small, Medium, Large, XL), and GPTNeo (125M, 1.3B, 2.7B), when the embedding layer and the FFN at each layer are included, and the original (affine) $ {\rm LayerNorm}(\cdot)$ normalization is used.}
\label{fig:psi_bert_gpt}
\end{figure}

\clearpage

\section{A rigorous analysis of the single-head self-attention dynamics}
\label{sec:single-head-analysis}

In this section we consider the equivalent of \eqref{eq:update-multihead-auton} but for a single-head. Under certain specific conditions (symmetric value matrix, characterized by a strictly dominant eigenvalue) we can obtain a fairly complete analytical description of its asymptotic behavior, namely classify its equilibria, and infer their local stability. In particular, we show that the landscape of equilibria resembles that of the corresponding continuous-time system investigated in e.g. \cite{altafini2025multistability}. In the discrete-time case discussed here, the analysis is significantly more complicated even though it is based on the same tools (Jacobian linearization). Such analysis is the main contribution of this section. 
In order to carry it out, it is useful to analyze first a simpler model which we call the multiagent discrete-time Oja system \cite{yoshizawa2001convergence}, and which corresponds to the single-head self-attention dynamics but without the attention matrix. 

\subsection{Model formulation}

Consider $n$ tokens represented as unit length column vectors $ \bx_i \in \mathbb{S}^{d-1} \subset \mathbb{R}^d$, $ i=1, \ldots, n $. 
Denote $ \bQ,\bK\in \mathbb{R}^{m\times d}$ the query and key matrices (for simplicity and without loss of generality we hereafter assume $ m=d$)  and $ \bV \in \mathbb{R}^{d \times d }$ the value matrix. 
Because of our representation of tokens as column vectors, these matrices correspond to the transposes of the weight matrices normally considered in the ML literature \cite{vaswani2017attention, Likhosherstov2021OnTE, chowdhury2022learningtransformerkernel, ustaomeroglu2025theoreticalstudyhyperselfattention}, see Section~\ref{sec:shared-param} of the paper or Section~\ref{sec:standardW} of these Supplementary Materials for the details. 
Since we are dealing with a single head, an explicit output matrix ($ F_O $ in \eqref{eq:update-multihead1} in the main paper) is not required, and is therefore omitted.

The following discrete-time difference equation encodes a single-head self-attention dynamics, including a skip connection and an $ L_2 $ layer normalization,
\begin{equation}
\label{eq:update}
\bx_i^+ = \frac{\bx_i +  \eta \bV \sum_{j=1}^n A_{ij}(\bx)  \bx_j }{\| \bx_i + \eta \bV\sum_{j=1}^n A_{ij}(\bx)  \bx_j \|},
\end{equation}
where $\eta>0$ is a step-size, $ A_{ij}(\bx) = \frac{e^{\langle \bQ \bx_i, \bK \bx_j \rangle } }{\sum_{\ell=1}^n e^{\langle \bQ \bx_i, \bK \bx_\ell \rangle } } $ is the attention that the token $ \bx_i $ gives to the token $ \bx_j$, computed through a softmax function ($ \bx $ is the stack of $ \bx_1, \ldots, \bx_n $ vectors) and where the attention matrix $ A(\bx) $ is assumed full, i.e., each token attends to every other token.
In \eqref{eq:update} we are implicitly interpreting the tokens as layer-dependent state variables: $ \bx_i = \bx_i(t) $ and $ \bx_i^+ = \bx_i(t+1)$, for some layer index $ t$.
We let $\|\cdot \|$ denote the Euclidean vector norm and the operator norm for matrices.
For the $i$-th token, define the tangent projector on $ \mathbb{S}^{d-1}$ at $\bx_i$ by $ \Pi_{\bx_i} := I -\bx_i  \bx_i ^\top$, which satisfies $\Pi_{\bx_i} \bx_i =0$ and $\Pi_{\bx_i} ^\top=\Pi_{\bx_i} $.

We are interested in investigating the asymptotic behavior of \eqref{eq:update}, i.e., in computing the equilibria of \eqref{eq:update} and studying their stability properties. 
For that (as already mentioned), it is convenient to investigate first a simpler system, which, for analogy with the continuous-time case, we call the multiagent DT Oja system \cite{altafini2025multistability}.
This system has the same structure as \eqref{eq:update} but lacks the attention matrix.

\subsection{Multiagent discrete-time Oja system}
\label{sec:multi-oja}
When we consider uniform attention \cite{noci2022signal}, i.e., we take $ A(\bf{x})=\frac{1}{n} \mathds{1} \mathds{1}^T$, then instead of \eqref{eq:update} we obtain the simpler system
\begin{equation}
\label{eq:update-moja}
\bx_i^+ = \frac{\bx_i +  \frac{\eta}{n} \bV \sum_{j=1}^n   \bx_j }{\| \bx_i + \frac{\eta}{n} \bV\sum_{j=1}^n \bx_j \|},
\end{equation}
which is easier to study. 
Let $ \by= \bV \bm{m} ( \bx)  =  \frac{1}{n}\bV \sum_{j=1}^n  \bx_j \in \mathbb{R}^d$ be the total influence of all agents on agent $i$ (which is the same for all agents).
Splitting \eqref{eq:update-moja} into parts before/after the normalization we can also write:
\beq
\tilde{\bx}_i = \bx_i + \eta \by 
% \frac{\eta}{n} \bV \sum_{j=1}^n   \bx_j 
\quad \text{and} \quad \bx_i^+ = \frac{\tilde{\bx}_i}{\| \tilde{\bx}_i \|} .
\label{eq:update-moja2}
\eeq

We make the following assumption which holds throughout the rest of this section.

\begin{assumption} 
\label{ass:V-sym}
The value matrix $ \bV\in \mathbb{R}^{d\times d } $ is nonsingular and symmetric, of eigenvalues $ \lambda_1>  \lambda_2 \geq \ldots \geq \lambda_d $ with $ \lambda_1 >0 $ simple and positive.
\end{assumption}
Let $ \bv_1, \ldots, \bv_d $ be the associated eigenvectors, normalized s.t. $ \| \bv_k \|=1$.

\subsubsection{Equilibria}

Let us first give our definition of consensus and bipartite consensus equilibria, which we request to be collinear with the eigenvectors of $\bV$ (the reason will become clear later on when we study the full model \eqref{eq:update}). 
The points $ \bx_1, \ldots, \bx_n \in \mathbb{S}^{d-1} $ are said to be in a {\em consensus equilibrium} if $ \bx_i = \bx_j = \bv_k $ $ \forall \, i, j = 1, \ldots, n$ (or $ \bx_i = \bx_j = - \bv_k $ $ \forall \, i, j = 1, \ldots, n$) and for some $ k=1, \ldots, d$. 
They are said to be in a {\em bipartite consensus equilibrium} if $ \bx_i =  \bv_k  $ for $ i \in  \mathcal{V}_1 \subset \mathcal{V} $  and $ \bx_i =  - \bv_k  $ for $ i \in  \mathcal{V} \setminus \mathcal{V}_1  $, for some $ k=1, \ldots, d$.

\begin{lemma}
\label{lem:moja-equil1}
Let $ \bx = [ (\bx_1)^T \, \ldots  (\bx_1)^T ]^T$ be an equilibrium point of the system \eqref{eq:update-moja}. 
Then $ \bx $  belongs to one of the following classes:
\benu
\item consensus: $ \bx_i = \bv_k $ $ \forall \, i$, and $ k=1, \ldots, d$;
\item bipartite consensus: $ \bx_i = \pm \bv_k $ $ \forall \, i$ and $ k=1, \ldots, d$;
\item polygonal equilibria: $ \{ {\bx_i \in \mathbb{S}^{d-1}} \; \text{s.t.} \;  \bV \sum_{j=1}^n  \bx_j =0\}. $
\eenu
\end{lemma}

\proof
An equilibrium point $ \bx_i $ satisfies $ \bx_i^+ = \bx_i$ for all $i$, hence, from \eqref{eq:update-moja2}, $ \bx_i = \frac{\tilde{\bx}_i}{\| \tilde{\bx}_i \|}  $, which implies that $ \tilde{\bx}_i$ and $ \bx_i $ must be collinear, i.e., $ \tilde{\bx}_i = \varphi  \bx_i $ for some $ \varphi \in \mathbb{R}$.
If $ \by =   \frac{1}{n}\bV \sum_{j=1}^n  \bx_j$, then from \eqref{eq:update-moja2}, $ \tilde{\bx}_i = \bx_i + \eta \by  $ holds iff also $\by  $ and $ \bx_i $ are collinear or $ \by =0$.
This condition can be written equivalently using projections as 
 \[
\Pi_{\bx_i} \by =  (I - \bx_i (\bx_i)^T )\by   =0 .
 \]
 There are 3 cases in which this  can happen:
\benu
\item consensus: $ \by = \varphi_i \bx_i $, for some scalar $ \varphi_i >0$;
\item bipartite consensus: $ \by = -\varphi_i \bx_i $, for some scalar $\varphi_i >0$;
\item polygonal equilibria: $ \by =0$.
\eenu
In fact, in the first two cases, $ \Pi_{\bx_i} \bx_i = \pm  \varphi_i (\bx_i - \bx_i  (\bx_i)^T \bx_i ) = 0$, since $ (\bx_i)^T \bx_i =1$. The third case follows trivially.
To show that consensus must be an eigenvector of $\bV$, observe that at this equilibrium point we have  $ \by= \bV \bx_i $, since $ \bx_i= \bx_j $.
Combining with $ \by= \varphi_i \bx_i $, we have $  \bV \bx_i =\varphi_i \bx_i $, i.e., $ \bx_i $ is an eigenvector of $\bV$ and $ \varphi_i$ one of its eigenvalues. The argument for bipartite consensus is identical. 

\qed

\begin{remark}
For a given $ \bv_k$ there are $ 2^n $ possible consensus or bipartite consensus equilibria.
These are always paired by a global symmetry w.r.t. the origin, i.e., if $ \bx_i = \bv_k $ $ \forall\, i$  is a consensus equilibrium point, its antipodal point $ \bx_i = - \bv_k $ $ \forall\, i$ is also a consensus equilibrium point, and similarly for the bipartite consensus equilibria.
This is implicitly assumed from now on, but  not explicitly written, to avoid overloading the formulation of the results.
\end{remark}

\subsubsection{Stability}
The stability properties of the multiagent Oja system \eqref{eq:update-moja} are summarized in the following theorem. 

\begin{theorem}
\label{thm:moja-DT}
For the system \eqref{eq:update-moja}, under Assumption~\ref{ass:V-sym}, if 
$ \eta< \min \left( \frac{2}{ |\lambda_1 + \lambda_d |}, \, \frac{1}{2 \max \{  \lambda_1 , | \lambda_d | \} }  \right) $, 
the principal consensus equilibrium $ \bx_i=\bv_1 $ $\forall \, i$ is asymptotically stable, while all other consensus equilibria, all bipartite consensus equilibria (including those aligned with $ \bv_1$), and all polygonal equilibria are unstable.
The trajectories of all tokens $ \bx_i(t)$ converge to $ \bv_1 $ (or to $ - \bv_1 $) for almost all initial conditions $ \bx_i(0)$. 
\end{theorem}

The theorem is proven through a series of lemmas. 
Let $ \delta_{ih} $ be a delta function ($ \delta_{ih}=1 $ if $ i=h$ and $ \delta_{ih}=0$ otherwise). 
Let  the right-hand side of \eqref{eq:update-moja}  be $ f_i(\bx)  =  \frac{\bx_i +  \frac{\eta}{n} \bV \sum_{j=1}^n   \bx_j }{\| \bx_i + \frac{\eta}{n} \bV\sum_{j=1}^n \bx_j \|}$, and denote $ f(\bx) =\begin{bmatrix} f_1(\bx)^T & \ldots & f_n(\bx)^T\end{bmatrix}^T $ the $ (nd) $-dimensional vector field associated to the stacked state vector $ \bx= \begin{bmatrix} \bx_1^T & \ldots & \bx_n ^T \end{bmatrix}^T$. 
The Jacobian $ J(\bx) = \pde{f(\bx)}{\bx} = \begin{bmatrix}\pde{f_i(\bx)}{\bx_h} \end{bmatrix} $ is composed of a diagonal part and a part that repeats itself on all $ (i, h)$ blocks.

\begin{lemma}
\label{lem:Jac-moja}
The Jacobian $ J(\bx) $ of the system \eqref{eq:update-moja} has the following blocks:
\beq
\begin{split}
\pde{f_i(\bx)}{\bx_h}  =& \frac{\frac{\eta}{n} \bV}{\| \bx_i +  \frac{\eta}{n} \bV \sum_j \bx_j \| } 
- \frac{ ( \bx_i + \frac{\eta}{n} \bV \sum_j \bx_j ) \frac{\eta}{n} \left( \bx_i^T \bV + \frac{\eta}{n} \left( \sum_j \bx_j^T \bV^T (I - \bx_i \bx_i^T )\bV \right) \right) }
{\| \bx_i +  \frac{\eta}{n} \bV \sum_j \bx_j \|^2 } \\
& + \left( \frac{I}{\| \bx_i +  \frac{\eta}{n} \bV \sum_j \bx_j \| }  
- \frac{ ( \bx_i + \frac{\eta}{n} \bV \sum_j \bx_j ) \frac{\eta}{n} \left( \sum_j \bx_j^T \bV^T (I -\frac{\eta}{n} \bx_i \sum_\ell \bx_\ell^T \bV ) \right)}
{\| \bx_i +  \frac{\eta}{n} \bV \sum_j \bx_j \|^2 }\right) \delta_{ih}
\end{split}
\label{eq:Jacob-moja}
\eeq
\end{lemma}

\proof
Let us first express the series expansion of the normed term:
\[
\begin{split}
\| \bx_i +  \frac{\eta}{n} \bV \sum_j \bx_j \| = & \left(  (  \bx_i +  \eta \by )^T (  \bx_i + \eta \by) \right)^{\frac{1}{2}}  
= \left( 1 + 2  \eta \bx_i ^T \by + \eta^2 \| \by \|^2   \right)^{\frac{1}{2}} \\
\approx & 1 + \frac{1}{2} \left( 2  \eta \bx_i ^T \by + \eta^2 \| \by \|^2  \right) - \frac{1}{2}  \eta^2 ( \bx_i^T \by ) ^2 + O(\eta^3)  \\
= & 1 + \frac{\eta}{n} \bx_i^T \bV \sum_j \bx_j + \frac{\eta^2}{2 n^2} \left( \sum_j \bx_j^T \bV^T \bV \sum_j \bx_j - ( \bx_i^T \bV \sum_j \bx_j)^2 \right)  + O(\eta^3) 
\end{split}
\]
where we have used the binomial series expansion ($ (1 + \xi)^{\frac{1}{2}}\approx 1 + \frac{\xi}{2} - \frac{ \xi^2}{8}  + O(\xi^3)  $), disregarding terms of order higher than 2 in $ \eta$.
Differentiating this term
\[
\begin{split}
\pde{\| \bx_i +  \frac{\eta}{n} \bV \sum_j \bx_j \|}{\bx_h } 
 = & \frac{\eta}{n} \bx_i^T \bV  + \frac{\eta^2}{n^2}  \sum_j \bx_j^T \bV^T ( I - \bx_i \bx_i^T ) \bV   \\
& + \left( \frac{\eta}{n} \sum_j \bx_j^T \bV^T - \frac{\eta^2}{n^2}  \sum_j \bx_j^T \bV^T \bx_i \sum_\ell \bx_\ell^T \bV \right) \delta_{ih}
\end{split}
\]
Inserting this expression into
\[
\begin{split}
\pde{f_i(\bx)}{\bx_h}  =& \frac{(I \delta_{ih} + \frac{\eta}{n} \bV)  \| \bx_i +  \frac{\eta}{n} \bV \sum_j \bx_j \| - (  \bx_i +  \frac{\eta}{n} \bV \sum_j \bx_j ) \pde{\| \bx_i +  \frac{\eta}{n} \bV \sum_j \bx_j \|}{\bx_h }  }
{\| \bx_i +  \frac{\eta}{n} \bV \sum_j \bx_j \|^2}
\end{split}
\]
the result follows.
\qed

\begin{lemma}
\label{lem:moja-F-eigenvalues}
For the system \eqref{eq:update-moja}, under Assumption~\ref{ass:V-sym}, at a consensus equilibrium $ \bx_i=\bv_k $ $\forall \, i$ the Jacobian $ J(\bx) $ has the following 3 classes of eigenvalues:
\benu
\item if $ 1+ \eta \lambda_k >0$, $ \frac{1-\eta \lambda_k + \eta^2 \lambda_k^2 }{1+ \eta \lambda_k } $ of multiplicity $ n$, while if $ 1+ \eta \lambda_k <0 $, $ \frac{-1-3 \eta \lambda_k + \eta^2 \lambda_k^2 }{1+ \eta \lambda_k } $ of multiplicity $ 1$ and $ \frac{-1-\eta \lambda_k + \eta^2 \lambda_k^2 }{1+ \eta \lambda_k } $ of multiplicity $ n-1$.
\item $ \frac{1+ \eta \lambda_h }{| 1+ \eta \lambda_k |} $ of multiplicity 1, $ h=1, \ldots, k-1, k+1, \ldots, d$. 
\item $  \frac{1 }{| 1+ \eta \lambda_k |} $ of multiplicity $ nd-d-n+1$.
\eenu
\end{lemma}

\proof
When computing the Jacobian \eqref{eq:Jacob-moja} at the consensus $ \bx_i=\bv_k $ $\forall \, i$, we get
\beq
 \begin{split}
& J(\bv_k)  = \left. \pde{f(\bx)}{\bx} \right|_{\bx_i = \bv_k} \\
 & = \mathds{1} \mathds{1}^T \otimes \left(  \frac{\sign(1+ \eta \lambda_k) \frac{\eta}{n} \bV - \frac{\eta}{n} \lambda_k \bv_k \bv_k^T }{1 + \eta \lambda_k }  \right) 
 + I \otimes \left(   \frac{\sign(1+ \eta \lambda_k)  I - \eta \lambda_k (1-\eta \lambda_k ) \bv_k \bv_k^T }{1 + \eta \lambda_k } \right).
 \end{split}
 \label{eq:JacobF-moja}
 \eeq
 where we have used $ \bv_k^T ( I - \bv_k \bv_k^T) =0$.
 The first term represents a factor present in all entries of $J(\bv_k) $, while the second one is present only on the diagonal.
 For $ J(\bv_k) $ there are 3 classes of eigenvectors:
 \benu
 \item The eigenvectors of the first class depend on the sign of $ 1+ \eta \lambda_k $. If $ 1+ \eta \lambda_k >0$, then choose $ \bp = [ 0 \ldots 0 \; \bv_k \; 0 \ldots 0 ] $ and compute the two components of $ J(\bv_k) $:
 \[
 \begin{cases}
&  \frac{\sign(1+ \eta \lambda_k) \frac{\eta}{n} \bV - \frac{\eta}{n} \lambda_k \bv_k \bv_k^T }{1 + \eta \lambda_k } \bv_k 
= \frac{\frac{\eta}{n} (1 -1)}{1+ \eta \lambda_k } \bv_k  =0 \\ 
& 
 \frac{\sign(1+ \eta \lambda_k)  I - \eta \lambda_k (1-\eta \lambda_k ) \bv_k \bv_k^T }{1 + \eta \lambda_k } \bv_k =
  \frac{ 1 - \eta \lambda_k ( 1- \eta \lambda_k ) }{1+ \eta \lambda_k } \bv_k
  \end{cases}
\]
meaning that 
\[
J(\bv_k ) \bp = \frac{ 1 - \eta \lambda_k  + \eta^2 \lambda_k^2  }{1+ \eta \lambda_k } \bp
\]
i.e., $ \frac{ 1 - \eta \lambda_k  + \eta^2 \lambda_k^2  }{1+ \eta \lambda_k }  $ is an eigenvalue of $ J(\bv_k)$ of multiplicity $ n$ (as there are $n$ possible slots for $ \bv_k $ in $ \bp$).
If instead $ 1+ \eta \lambda_k <0 $, then choose the following $ n$ eigenvectors
\[
\bp = \begin{cases} \bp_1 =  [\bv_k \; \bv_k \ldots \bv_k ]  \\
 \bp_\ell = [\phi_1^\ell \bv_k \;  \phi_2^\ell \bv_k \ldots \phi_n^\ell \bv_k ] \; \text{s.t.} \; \sum_i \phi_i^\ell =0 , \; \ell=2, \ldots, n
 \end{cases}
 \]
 In correspondence of $ \bp_1 $, one gets 
 \[
 \begin{cases}
&  \frac{\sign(1+ \eta \lambda_k) \frac{\eta}{n} \bV - \frac{\eta}{n} \lambda_k \bv_k \bv_k^T }{1 + \eta \lambda_k } \bv_k 
=\frac{\eta}{n}  \frac{ - 2 \lambda_k }{1+ \eta \lambda_k } \bv_k   \\ 
& 
 \frac{\sign(1+ \eta \lambda_k)  I - \eta \lambda_k (1-\eta \lambda_k ) \bv_k \bv_k^T }{1 + \eta \lambda_k } \bv_k =
  \frac{ -1 - \eta \lambda_k + \eta^2 \lambda_k^2 }{1+ \eta \lambda_k } \bv_k
  \end{cases}
\]
hence 
\[
J(\bv_k) \bp_1 = \eta  \frac{ - 2 \lambda_k }{1+ \eta \lambda_k } \bp_1 +  \frac{ -1 - \eta \lambda_k + \eta^2 \lambda_k^2  }{1+ \eta \lambda_k } \bp_1 = \frac{-1-3 \eta \lambda_k + \eta^2 \lambda_k^2 }{1+ \eta \lambda_k }  \bp_1 
\]
while in correspondence of $ \bp_\ell $ one gets
\[
J(\bv_k) \bp_\ell = \begin{bmatrix} \sum_i \phi_i^\ell \left( \frac{\eta}{n}   \frac{ - 2 \lambda_k }{1+ \eta \lambda_k } \right) \bv_k +  \frac{ -1 - \eta \lambda_k + \eta^2 \lambda_k^2  }{1+ \eta \lambda_k } \phi_1 \bv_k \\
\vdots \\
\sum_i \phi_i^\ell \left( \frac{\eta}{n}   \frac{ - 2 \lambda_k }{1+ \eta \lambda_k } \right) \bv_k +  \frac{ -1 - \eta \lambda_k + \eta^2 \lambda_k^2  }{1+ \eta \lambda_k } \phi_n \bv_k
\end{bmatrix} =  \frac{ -1 - \eta \lambda_k + \eta^2 \lambda_k^2  }{1+ \eta \lambda_k } \bp_\ell 
\]
i.e., $  \frac{ -1 - \eta \lambda_k + \eta^2 \lambda_k^2  }{1+ \eta \lambda_k } $ is an eigenvalue of multiplicity $ n-1$. 

\item For the second class of eigenvectors choosing $ \bq_h = [ \bv_h\, \ldots \, \bv_h] $ for $ h\neq k$, leads to 
\[
 \begin{cases}
&  \frac{\sign(1+ \eta \lambda_k) \frac{\eta}{n} \bV - \frac{\eta}{n} \lambda_k \bv_k \bv_k^T }{1 + \eta \lambda_k } \bv_h 
=\frac{\eta}{n}  \frac{  \lambda_h }{|1+ \eta \lambda_k |} \bv_h  \\ 
& 
 \frac{\sign(1+ \eta \lambda_k)  I - \eta \lambda_k (1-\eta \lambda_k ) \bv_k \bv_k^T }{1 + \eta \lambda_k } \bv_h =
  \frac{ 1  }{|1+ \eta \lambda_k |} \bv_h
  \end{cases}
\]
and hence to 
\[
J(\bv_k) \bq_h = \frac{1 + \eta \lambda_h }{|1+ \eta \lambda_k |} \bq_h 
\]
i.e., $ \frac{1 + \eta \lambda_h }{|1+ \eta \lambda_k |} $ is an eigenvalue of $ J(\bv_k)$. There are $ d-1 $ such eigenvalues. 
\item The remaining $ nd-d-n+1 $ eigenvectors are assembled by considering vectors $ \br =\begin{bmatrix} (\bz^1)^T & (\bz^2)^T & \ldots  & (\bz^n)^T \end{bmatrix}^T$ s.t. $ \bv_k^T \bz^i =0 $ for all $ i =1, \ldots, n $ and all $ h=1, \ldots, k-1, k+1, \ldots, d $ and $ \bq_h^T \br =0 $. Since the number of such constraints is $ d-1+n$, there exist $ nd- n-d+1 $ such vectors $ \bm{r}$, and they can always be chosen so that  $ \sum_{i=1}^n  \bz^i =0 $.
Computing,
\[
 \begin{cases}
&  \frac{\sign(1+ \eta \lambda_k) \frac{\eta}{n} \bV - \frac{\eta}{n} \lambda_k \bv_k \bv_k^T }{1 + \eta \lambda_k } \bz^i 
= \frac{\eta}{n} \frac{ \bV }{|1+ \eta \lambda_k |} \bz^i  \\ 
& 
 \frac{\sign(1+ \eta \lambda_k)  I - \eta \lambda_k (1-\eta \lambda_k ) \bv_k \bv_k^T }{1 + \eta \lambda_k } \bz^i =
  \frac{ 1  }{|1+ \eta \lambda_k |} \bz^i
  \end{cases}
\]
hence
\[
J(\bv_k) \br =   \frac{ 1  }{|1+ \eta \lambda_k |} \begin{bmatrix}  \frac{\eta}{n} \bV \sum_i \bz_i + \bz_1 \\
\vdots \\
 \frac{\eta}{n} \bV \sum_i \bz_i + \bz_n
\end{bmatrix} =   \frac{ 1  }{|1+ \eta \lambda_k |} \br
\]
\end{enumerate}
\qed

\begin{remark}
\label{rem:two-classes}
The discrete-time system \eqref{eq:update} can be seen as a map from the manifold $ (\mathbb{S}^{d-1})^n $ to itself. 
Hence its tangent space is $ T_{\bx}  (\mathbb{S}^{d-1})^n $, of dimension $ n (d-1)$, while the Jacobian $ J(\bx)$ computed above is in dimension $ nd $. 
Such Jacobian is what is sometimes called the ``extrinsic Jacobian'' living in the ambient manifold $ \mathbb{R}^{ nd \times nd}$ rather than in $ T_{\bx}  (\mathbb{S}^{d-1})^n $. 
What decides stability is however the ``intrinsic Jacobian'' living on  $ T_{\bx}  (\mathbb{S}^{d-1})^n $. For Lemma~\ref{lem:moja-F-eigenvalues}, if we look at the eigenvectors, we can easily deduce that the first class of eigenvectors (of cardinality $n$) does not belong to $ T_{\bx}  (\mathbb{S}^{d-1})^n $, while the other two classes are always in $ T_{\bx}  (\mathbb{S}^{d-1})^n $, hence it is only these last two classes that decide the stability character of the equilibria.
\end{remark}

\begin{lemma}
\label{lem:moja-consensus-stab}
For the system \eqref{eq:update-moja}, under Assumption~\ref{ass:V-sym}, and, if $ - \lambda_d > \lambda_1 $,  for $ \eta< \frac{2}{ |\lambda_1 + \lambda_d |} $, then the principal consensus equilibrium point, $ \bx_i=\bv_1 $ $\forall \, i$, is asymptotically stable.
For $ k=2, \ldots, d $, the consensus equilibria $ \bx_i=\bv_k $ $\forall \, i$ are all unstable. %  $ \forall \, \eta>0$.
\end{lemma}

\proof
Consider the  principal consensus equilibrium point $ \bv_1$ and the eigenvalues of $ J(\bv_1)$ computed in Lemma~\ref{lem:moja-F-eigenvalues}. 
Notice from Remark~\ref{rem:two-classes} that only the second and third class of eigenvalues matter for the stability.
Under Assumption~\ref{ass:V-sym}, $ 1 + \eta \lambda_1 >1$, hence the 3rd class of eigenvalues is inside the unit disk. When $ | \lambda_d | < \lambda_1 $, then also the second class is inside the disk for all $ h=2, \ldots, d$, regardless of the sign of $ \lambda_h$, while when $ - \lambda_d >\lambda_1 $ we need to have $  \frac{1+ \eta \lambda_d }{ 1+ \eta \lambda_1 } >-1 $, i.e., $ \eta <  \frac{-2}{\lambda_1 + \lambda_d } = \frac{2}{| \lambda_1 + \lambda_d |} $. 
% Concerning the first class of eigenvectors it must be $ -1 < \frac{1 -\eta \lambda_1 + \eta^2 \lambda_1^2 }{1+ \eta \lambda_1 } < 1 $, i.e.,  $ 1- \eta \lambda_k + \eta^2 \lambda_k^2 > -1 - \eta \lambda_k $ whch is always true and $ 1- \eta \lambda_k + \eta^2 \lambda_k^2 < 1+ \eta \lambda_1$, which is $\eta \lambda_1 ( \eta \lambda_1 -2 ) <0$,  i.e. $ \eta <  \frac{2}{\lambda_1 }$. 
% Summarizing, sufficient conditions for stability of the principal consensus equilibrium point are 
% \[
% \begin{cases}
% \eta< \min\left( \frac{2}{\lambda_1 } , \, \frac{-2}{\lambda_1 + \lambda_d } \right) &  \text{when }\;  \lambda_d +\lambda_1 <0 \\
% \eta< \frac{2}{\lambda_1 } & \text{when }\;  \lambda_d +\lambda_1 >0 
% \end{cases}
% \]
% For later reference, notice that when $ \lambda_d <0 $
% \beqa
% \frac{1}{\lambda_1 } & > \frac{1}{| \lambda_1 + \lambda_d|} \quad \Longleftrightarrow \quad \lambda_d < - 2 \lambda_1 \label{eq:ineq_lambda_1} \\ 
% \frac{1}{\lambda_1 } & < \frac{1}{| \lambda_1 + \lambda_d|} \quad \Longleftrightarrow \quad \lambda_d > - 2 \lambda_1 \label{eq:ineq_lambda_2}
% \eeqa

When $ k>1$ (consensus at non-principal eigenvectors), then one of the eigenvalues of $ J(\bv_k) $ is $ \frac{1 + \eta \lambda_1 }{|1+ \eta \lambda_k |}  $ which is always $ >1 $. This is obvious when $ \lambda_k \geq 0 $, as $ \lambda_k< \lambda_1$. What needs to be checked is the case of $ \lambda_k $ negative with $ - \lambda_k \gg \lambda_1 $. 

For this case we consider the two situations
\benu
\item If $ 1 + \eta \lambda_k >0 $, then we have $ \frac{1 + \eta \lambda_1 }{ 1+ \eta \lambda_k }>1$ because this is equivalent to $ 1 + \eta \lambda_1 > 1+ \eta \lambda_k $. 

\item If $ 1 + \eta \lambda_k <0 $, then we must show that it is $ - \frac{1 + \eta \lambda_1 }{ 1+ \eta \lambda_k }>1$, i.e., that $ 1 + \eta \lambda_1 > - 1 - \eta \lambda_k$, or $ \eta (\lambda_1 + \lambda_k) >-2$. As $ \lambda_1 + \lambda_k<0$, the expression holds if it holds for the largest possible value of $ \eta $.
% If \eqref{eq:ineq_lambda_1} holds, then we can take $ \eta = \frac{2}{| \lambda_1 + \lambda_d|}$ and we obtain the following inequality $ - \frac{2(\lambda_1 + \lambda_k)}{\lambda_1 + \lambda_d}>-2 $, or, equivalently, $ \lambda_1 + \lambda_k > \lambda_1 + \lambda_d$ which is always true. 
% When instead \eqref{eq:ineq_lambda_2} holds, then it must be $\frac{2}{\lambda_1} ( \lambda_1 + \lambda_k) > -2 $ or $ \lambda_k > -2 \lambda_1 $ which is always true since, from \eqref{eq:ineq_lambda_2}, $ -2 \lambda_1 < \lambda_d \leq \lambda_k $. 
Taking $ \eta = \frac{2}{| \lambda_1 + \lambda_d|}$, we obtain the following inequality $ - \frac{2(\lambda_1 + \lambda_k)}{\lambda_1 + \lambda_d}>-2 $, or, equivalently, $ \lambda_1 + \lambda_k > \lambda_1 + \lambda_d$ which is always true. 
\eenu
In both situations $ J(\bv_k) $ has an eigenvalue outside the unit disk. 
\qed

At a bipartite consensus point $ \bx_i = \pm \bv_k $, split the $n$ tokens into two sets $ \mathcal{V}_1 $ and $ \mathcal{V}_2 $, $ \mathcal{V}_1\cup \mathcal{V}_2 = \{ 1, \ldots, n \}$,  according to whether $ \bx_i =\bv_k $ or $ \bx_i =- \bv_k$ at equilibrium. Assume that $ n_1 = | \mathcal{V}_1|  $ tokens are equal to $ \bv_k $ and $ n_2 = | \mathcal{V}_2 | $ equal to $ -\bv_k$, with $ n_1 + n_2 =n$. 
Let $ \nu=\frac{n_1-n_2}{n}$, where we consider w.l.o.g. $ n_1 \geq n_2 $ (hence $ \nu\geq 0$). 

%---------- NONSENSE ????? ----------------
%
%\bite
%\item case $ i \in \mathcal{V}_1 $:
%\[
% \begin{split}
%J(\bv_k)   = & \mathds{1} \mathds{1}^T \otimes \left(  \frac{\sign(1+ \eta \nu \lambda_k) \frac{\eta}{n} \bV - \frac{\eta}{n} \lambda_k \bv_k \bv_k^T }{1 + \eta \nu \lambda_k }  \right) \\
%& - I \otimes \left(   \frac{\sign(1+ \eta \nu \lambda_k)  I - \eta \nu \lambda_k (1-\eta \nu \lambda_k ) \bv_k \bv_k^T }{1 + \eta \nu \lambda_k } \right).
% \end{split}
% \]
%\item case $ i \in \mathcal{V}_2 $:
%\[
% \begin{split}
%J(\bv_k)   = & \mathds{1} \mathds{1}^T \otimes \left(  \frac{\sign(-1+ \eta \nu \lambda_k) \frac{\eta}{n} \bV + \frac{\eta}{n} \lambda_k \bv_k \bv_k^T }{-1 + \eta \nu \lambda_k }  \right) \\
%& - I \otimes \left(   \frac{\sign(-1+ \eta \nu \lambda_k)  I - \eta \nu \lambda_k (1+\eta \nu \lambda_k ) \bv_k \bv_k^T }{-1 + \eta \nu \lambda_k } \right).
% \end{split}
% \]
% \eite
 
\begin{lemma}
\label{lem:moja-bip-consensus-stab}
For the system \eqref{eq:update-moja}, under Assumption~\ref{ass:V-sym}, the bipartite consensus equilibria, $ \bx_i=\pm \bv_k $ $\forall \, i$ are all unstable $ \forall \, k =1, \ldots, d$ and $ \forall \, \eta>0$.
\end{lemma}

\proof
At a bipartite consensus equilibrium there exists at least onne index $ i \in \mathcal{V}_1 $. Consider the perturbation $ \bx_i = \bv_k + \bw_i  $ with $ \bv_k^T \bw_i =0$ $ \forall \, i$. For this perturbation, the unnormalized next iterate is 
\[
\tilde{\bx}_i = \bv_k+\bw_i + \frac{\eta}{n} \bV\sum_j (\bv_k + \bw_j) = ( 1+ \eta \nu \lambda_k ) \bv_k + \bw_i + \frac{\eta}{n} \bV \sum_j \bw_j 
\]
and its norm
\[
\| \tilde{\bx}_i \| = \left( \alpha^2\left( 1+ \frac{\| \bu_i \|^2 }{\alpha^2} \right) \right)^{\frac{1}{2}}
\]
where $ \alpha := 1 + \eta \nu \lambda_k $, $ \bu_i := \bw_i + \frac{\eta}{n} \bV \sum_j \bw_j $, and we have used $ \bv_k^T \bv_k =1 $ and $ \bv_k^T \bw_i =0 $.
Hence, using the binomial series expansion ($ (1 + \xi)^{-\frac{1}{2}}\approx 1 - \frac{\xi}{2}  + h.o.t. $), 
\[
\begin{split}
\bx_i^+ &= \frac{\tilde{\bx}_i}{\| \tilde{\bx}_i \|}
= \frac{ \alpha \bv_k + \bu_i }{  \alpha \left( 1+ \frac{\| \bu_i \|^2 }{\alpha^2} \right) ^{\frac{1}{2}} } 
\approx \bv_k + \frac{\bu_i}{\alpha} + h.o.t. \\
&= \bv_k + \frac{ \bw_i + \frac{\eta}{n} \bV \sum_j \bw_j }{ 1+ \eta \nu \lambda_k }  + h.o.t.
\approx  \bv_k + ( 1- \eta \nu \lambda_k ) ( \bw_i + \frac{\eta}{n} \bV \sum_j \bw_j ) 
\end{split}
\]
where we have used the expansion $ \frac{1}{1+ \eta \nu \lambda_k } \approx 1-  \eta \nu \lambda_k$. 
The linearized equation is then 
\[
\bw_i^+ =  ( 1- \eta \nu \lambda_k )  \bw_i + \frac{\eta}{n} \bV \sum_j \bw_j .
\]
Expanding all $ \bw_i $ in the orthonormal basis $ \bv_1, \ldots, \bv_d $, and denoting $ \bw_i = \sum_\ell \xi^i_\ell \bv_\ell $, 
\[
\bw_i^+ = \sum_\ell (\xi^i_\ell)^+ \bv_\ell =  ( 1- \eta \nu \lambda_k ) \sum_\ell \xi^i_\ell \bv_\ell + \frac{\eta}{n} \sum_{j, \ell} \lambda_\ell \xi^j_\ell  \bv_\ell .
\]
Projecting along $ \bv_1 $, 
\[
(\xi^i_1)^+ =  ( 1- \eta \nu \lambda_k ) \xi^i_1 +\frac{\eta}{n} \lambda_1 \sum_j \xi^j_1 .
\]
If $ \bm{\xi}_1 := [ \xi_1^1 \ldots \xi_1^n ]^T$, then these $ n$ difference equations can be written in vector form as
\[
\bm{\xi}_1^+ =  ( 1- \eta \nu \lambda_k ) \bm{\xi}_1 + \frac{\eta}{n} \lambda_1 \mathds{1} \mathds{1}^T \bm{\xi}_1 .
\]
We need to compute the eigenvalues of the $ n \times n$ matrix $ M(k,1) = ( 1- \eta \nu \lambda_k )I + \frac{\eta}{n} \lambda_1 \mathds{1} \mathds{1}^T $. 
One such eigenvalues is associated to the consensus vector $ \mathds{1} \in \mathbb{R}^d$:
\[
 M(k,1) \mathds{1} = ( 1- \eta \nu \lambda_k ) \mathds{1} +  \eta  \lambda_1 \mathds{1} = ( 1- \eta \nu \lambda_k +\eta \lambda_1 ) \mathds{1} .
 \]
 If $ k\neq 1 $ then either $ \lambda_k >0 $ and $ \lambda_k < \lambda_1$, or $ \lambda_k<0$. Since $ \nu<1$, in both cases $ 1- \eta \nu \lambda_k +\eta \lambda_1 >1 $ and the bipartite equilibrium point is unstable. 
 If $ k=1$, then $ 1+\eta(1-\nu) \lambda_1 >1$, i.e., also in this case we have found an eigenvalue outside the unit disk, and hence the equilibrium point is unstable. 
\qed

\begin{lemma}
\label{lem:moja-polygonal-stab}
For the system \eqref{eq:update-moja}, under Assumption~\ref{ass:V-sym}, all polygonal equilibria are unstable.
\end{lemma}

\proof
At a polygonal equilibrium point $ \bx=\bs =[\bs_1 \, \dots \, \bs_n] $, consider the perturbation $ \bx_i = \bs_i + \bw_i  $ with $ \bs_i^T \bw_i =0$ $ \forall \, i$. For this perturbation, the unnormalized next iterate is 
\[
\tilde{\bx}_i = \bs_i+\bw_i + \frac{\eta}{n} \bV\sum_j (\bs_j + \bw_j) =\bs_i + \frac{\eta}{n} \underbrace{\bV\sum_j \bs_j}_{=0} + \bw_i + \frac{\eta}{n}  \bV \sum_j \bw_j 
\]
and its norm
\[
\| \tilde{\bx}_i \| = \left( 1+ \| \bu_i \|^2 +2  \frac{\eta}{n}  \bs_i^T \bV \sum_j \bw_j  \right)^{\frac{1}{2}}
\]
where $ \bu_i := \bw_i + \frac{\eta}{n} \bV \sum_j \bw_j $, and we have used $ \bs_i^T \bw_i =0 $.
Hence, using the binomial series expansion ($ (1 + \xi)^{-\frac{1}{2}}\approx 1 - \frac{\xi}{2}  + h.o.t. $), 
\[
\begin{split}
\bx_i^+ &= \frac{\tilde{\bx}_i}{\| \tilde{\bx}_i \|}
= \frac{  \bs_i + \bu_i }{ \left( 1+ \| \bu_i \|^2 + 2  \frac{\eta}{n}  \bs_i^T \bV \sum_j \bw_j  \right) ^{\frac{1}{2}} } 
\approx (\bs_i + \bu_i) \left(1- \frac{\| \bu_i \|^2}{2} -   \frac{\eta}{n}  \bs_i^T \bV \sum_j \bw_j \right)  + h.o.t. \\
&= \bs_i + \bu_i -   \frac{\eta}{n}  \bs_i^T \bV \sum_j \bw_j \bs_i  + h.o.t.
= \bs_i + \bw_i + \frac{\eta}{n} (I - \bs_i \bs_i^T) \bV \sum_j \bw_j + h.o.t.
\end{split}
\]
The linearized equation is then 
\[
\bw_i^+ = \bw_i + \frac{\eta}{n} (I - \bs_i \bs_i^T) \bV \sum_j \bw_j .
\]
Expanding all $ \bw_i $ and $ \bs_i $ in the orthonormal basis $ \bv_1, \ldots, \bv_d $, and denoting $ \bw_i = \sum_\ell \xi^i_\ell \bv_\ell $ and $ \bs_i = \sum_\ell \zeta^i_\ell \bv_\ell $, 
\[
\begin{split}
\bw_i^+ = &\sum_\ell (\xi^i_\ell)^+ \bv_\ell 
=  \sum_\ell \xi^i_\ell \bv_\ell + \frac{\eta}{n} \left( I - \sum_h \zeta^i_h \bv_h \sum_\ell \zeta^i_\ell \bv_\ell^T \right)  \sum_{j, \ell} \lambda_\ell \xi^j_\ell  \bv_\ell \\
= & \sum_\ell \xi^i_\ell \bv_\ell + \frac{\eta}{n} \left( I - \sum_h \zeta^i_h \bv_h \sum_\ell \zeta^i_\ell  \lambda_\ell \sum_j  \xi^j_\ell  \right).
\end{split}
\]
Recall that $ \bV$ has always at least one positive eigenvalue $\lambda_1 >0$. We can always choose $ \bm{w}$ s.t. each $ \bm{w}_j $ has a nonzero component in the direction of the associated eigenvector $ \bv_1$: $ \xi^j_1\neq 0$.
Projecting along $ \bv_1 $
\[
(\xi^i_1)^+ = \xi^i_1 + \frac{\eta}{n} \left( 1-{\zeta_1^i}  \sum_\ell \zeta^i_\ell  \lambda_\ell \sum_j  \xi^j_\ell \right).
\]
If the perturbation is such that $ \bw_i $ has only a component along $ \bv_1 $, $ \xi_1^i\neq 0$, $ \xi_j^i=0$ $ \forall \, j \neq i$, then
\[
(\xi^i_1)^+ = \xi^i_1 + \frac{\eta}{n} ( 1-{\zeta_1^i}^2) \lambda_1 \sum_j \xi^j_1 .
\]
If $ \bm{\xi}_1 :=  \begin{bmatrix} \xi_1^1 &\ldots &\xi_1^n\end{bmatrix}^T$ and $  \bm{\zeta}_1 = \begin{bmatrix} \zeta^1_1 & \ldots & \zeta^n_1 \end{bmatrix}^T$, then in vector form
\[
\bm{\xi}_1^+ =  \bm{\xi}_1 + \frac{\eta}{n} \lambda_1 \left( I- Z_1^2 \right)\mathds{1} \mathds{1}^T \bm{\xi}_1 = \left( I+ \frac{\eta}{n} \lambda_1 \left( I- Z_1^2 \right)\mathds{1} \mathds{1}^T \right) \bm{\xi}_1 .
\]
where $ Z_1 = \diag(\bm{\zeta}_1)$. 
Since $ \bm{s} $ is not an eigenvector of $\bV$, it is $ |\zeta^i_1 |< 1$. From $ \lambda_1 >0$, it is then $\rho(I +\eta \lambda_1(I-Z_1^2) \mathds{1} \mathds{1}^T)>1 $, meaning that when $ \bm{\xi}_1 = \epsilon \mathds{1}$ for some scalar coefficient $ \epsilon$ the system diverges, hence the polygonal equilibrium point $ \bs$ is unstable. 
\qed 

\begin{remark}
    In both Lemma~\ref{lem:moja-bip-consensus-stab} and Lemma~\ref{lem:moja-polygonal-stab} the perturbation $ \bw_i$ used in the proof is ``intrinsic'', i.e., it belongs to the tangent space $ T_{\bv_k}\mathbb{S}^{d-1} $ resp. $ T_{\bs_i}\mathbb{S}^{d-1} $ at the equilibrium point.
\end{remark}
For the multiagent Oja system \eqref{eq:update-moja} it is also possible to show that the trajectories are generically converging to an equilibrium point. 

\begin{lemma}
\label{lem:layp-moja}
The function $ W(\bx) = -\frac{1}{2 n^2} \left( \sum_{j=1}^n \bx_j^T \bV \sum_{j=1}^n \bx_j  \right)$ is a Lyapunov function for  \eqref{eq:update-moja}. When $  \eta\leq \frac{1}{2 \|\bV\|}  $ it guarantees that  \eqref{eq:update-moja} converges to one of the equilibria determined in Lemma~\ref{lem:moja-equil1}.
\end{lemma}
\proof
% see my notes page t101
Since the state space is compact,  all functions arising in the expansion of the Lyapunov function $ W(\bx)  $ (and hence also of its increment) are uniformly bounded. For the second order truncation of the increment $ W(\bx^+) - W(\bx) $, it is possible to compute an explicit bound on $ \eta $ for which the increment is guaranteed to be nonincreasing.
From \eqref{eq:update-moja2},  using the binomial series expansion $ (1 + \xi)^{-\frac{1}{2}}\approx 1 - \frac{\xi}{2}  + \frac{3}{8} \xi^2 + O(\xi^3) $  
\[
\begin{split}
 \bx_i^+ & = \frac{\tilde{\bx}_i}{\| \tilde{\bx}_i \| } = (\bx_i + \eta \by )  \left(1- \eta \bx_i^T \by  + \frac{\eta^2}{2} \left( 3 (\bx_i^T \by )^2 - \| \by\|^2   \right) + O(\eta^3)  \right) \\
 & = \bx_i +\eta g_i + \eta^2 \left( \frac{3}{2} r_i^2 \bx_i  - \frac{1}{2} \| \by\|^2 r_i - r_i \by \right) + O(\eta^3)  
 \end{split}
 \]
 where we have indicated $ r_i := \bx_i^T \by $ and $ g_i= g(\bx_i ) := (I - \bx_i \bx_i^T) \by $.
 The increment of $ W$ up to second order in $ \eta $ is then 
 \[
 \begin{split}
 W(\bx^+)& - W(\bx) =  -\frac{1}{2 n^2} \left( \sum_{j=1}^n (\bx_j^+)^T \bV \sum_{j=1}^n \bx_j^+  - \sum_{j=1}^n \bx_j^T \bV \sum_{j=1}^n \bx_j  \right) \\
 = & - \frac{\eta}{n} \sum_j g_j^T \by - \frac{\eta^2}{2 n^2} \left( 2 n \by^T \sum_j \left( \frac{3}{2} r_i^2 \bx_i - \frac{1}{2} \| \by\|^2 \bx_i - r_i \by \right) + \sum_j g_j^T \bV \sum_j g_j \right) + O(\eta^3) \\
 = & - \frac{\eta}{n} \sum_j \| g_j \|^2 - \frac{\eta^2}{ n} \left(  \sum_j  \left( \frac{3}{2} r_j^2 \by^T \bx_j - \frac{1}{2} \| \by\|^2 \by^T \bx_j - r_j \by^T \by \right) + \frac{1}{2n}  \sum_j g_j^T \bV \sum_j g_j \right) + O(\eta^3)  \\
 & = - \frac{\eta}{n} \sum_j \| g_j \|^2 - \frac{\eta^2}{ n} \left( \frac{3}{2} \sum_j r_i  (r_i^2 - \| \by \|^2 )+ \frac{1}{2n}  \sum_j g_j^T \bV \sum_j g_j  \right) + O(\eta^3)  \\
 & = - \frac{\eta}{n} \sum_j \| g_j \|^2 + \frac{\eta^2}{ n} \left( \frac{3}{2} \sum_j r_i  \| g_j \|^2 - \frac{1}{2n}  \sum_j g_j^T \bV \sum_j g_j  \right) + O(\eta^3)  \\
 \end{split}
 \]
 where we have used idempotency of the projection $ (I - \bx_i \bx_i^T)$ to get
 \[
 \begin{split}
 g_i^T \by & = \by^T (I - \bx_i \bx_i^T) \by = \by^T (I - \bx_i \bx_i^T)(I - \bx_i \bx_i^T) \by = \| g_i\|^ 2\\
 \|\by\|^2 - r_i^2 & = \by^T \by - \by^T \bx_i \bx_i^T \by = \by^T (I - \bx_i \bx_i^T) \by = \by^T (I - \bx_i \bx_i^T) (I - \bx_i \bx_i^T) \by = g_i^T g_i = \| g_i \|^2 .
 \end{split}
 \]
 Since $ \| \sum_j g_j \|^2 \leq n \sum_j \| g_j \|^2 $, it is
 \[
 - n \| \bV \| \sum_j \| g_j \|^2 \leq \sum_j g_j^T \bV \sum_j g_j \leq n \| \bV \| \sum_j \| g_j \|^2 .
 \]
  Furthermore
 \[
 - \| \bV \| \leq - \| \by \| \leq r_i \leq  \| \by \| \leq \| \bV \| .
 \]
 These lead to the inequality
 \[
 W(\bx^+) - W(\bx) \leq  - \frac{\eta}{n} \sum_j \| g_j \|^2 +\frac{\eta^2}{n} \left( \frac{3}{2} \| \bV \| \sum_j \| g_j \|^2 + \frac{1}{2} \| \bV \| \sum_j \| g_j \|^2 \right) 
 \]
 or
 \[
 W(\bx^+) - W(\bx) \leq \frac{\eta}{n}  \left(  - 1 + 2 \eta \| \bV \| \right) \sum_j \| g_j \|^2.
 \]
 Hence when $ \eta\leq \frac{1}{2 \|\bV\|} $ we have $  W(\bx^+) - W(\bx) \leq 0$. 
 Since the sequence $ W(\bx (k) ) $ (where $k$ is the time index) is nonincreasing and bounded below, it must converge. Moreover, summability of the increments implies $ \sum_{k=0}^\infty \sum_{i=1}^n \| g(\bx_i(k)) \|^2 < \infty $, hence $  \| g(\bx_i(k)) \| \to 0 $ for each $i$. 
 At an equilibrium point $ \bx $ it must then be $ g(\bx_i) = (I - \bx_i \bx_i) \by =0 $ for all $ i$, i.e., $ \by $ must be collinear with $ \bx_i $. Recalling Lemma~\ref{lem:moja-equil1}, this implies that $ \bx $ must be a consensus or a bipartite consensus equilibrium point. 
 \qed

\noindent {\bf Proof of Theorem~\ref{thm:moja-DT}}
The proof of the theorem is simply the composition of the results of Lemmas~\ref{lem:moja-F-eigenvalues},~\ref{lem:moja-consensus-stab},~\ref{lem:moja-bip-consensus-stab},~\ref{lem:moja-polygonal-stab} and~\ref{lem:layp-moja}.
Lemmas~\ref{lem:moja-consensus-stab},~\ref{lem:moja-bip-consensus-stab},~\ref{lem:moja-polygonal-stab} say that only the consensus point $ \bv_1 $ corresponding to the principal eigenvalue of $\bV$ is locally asymptotically stable, while all other equilibria are unstable. 
Since all trajectories of \eqref{eq:update-moja} have a limit point, generically such limit point must be $ \bv_1$. 
\qed

\begin{remark}
% The condition $ \eta< \min \left( \frac{2}{\lambda_1 } , \, \frac{2}{ |\lambda_1 + \lambda_d |}, \, \frac{2n}{(3n+1) \max \{  \lambda_1 , | \lambda_d | \} }  \right) $ is sufficient but not necessary. 
% %However, the spirit of the model is that the step size $ \eta $ is a ``small'' parameter. 
% For larger values of $ \eta$ the truncated binomial series expansion used in the construction of the Jacobian (and hence of the eigenvalues that are used to determine stability) may no longer be a valid approximation.
The condition $ \eta< \min \left(\frac{2}{ |\lambda_1 + \lambda_d |}, \, \frac{1}{2 \max \{  \lambda_1 , | \lambda_d | \} }  \right) $ is sufficient but not necessary. 
Recall from Lemma~\ref{lem:moja-consensus-stab} that the condition $\eta< \frac{2}{ |\lambda_1 + \lambda_d |} $ is needed only when $ | \lambda_d|> \lambda_1$. 
\end{remark}

\subsection{Single-head self-attention dynamics}
\label{sec:self-att}

Let us now consider the self-attention system \eqref{eq:update}, which includes the attention matrix $ A(\bx)$. For this model, an ``individual'' weighted average with state dependent weights given by the attention coefficients, $ \bm{m}_i =\bm{m}_i(\bx)  =   \sum_{j=1}^n  A_{ij}(\bx) \bx_j $ replaces the simple average used in the multiagent Oja model. 
Consequently, the total action of all agents on agent $i$ varies from agent to agent:  $\by_i = \by_i (\bx) =  \bV\sum_{j=1}^n A_{ij}(\bx)   \bm{\bx}_j $. 

\subsubsection{Equilibria}

In addition to the three classes of equilibria already obtained for the multiagent Oja flow, in the self-attention dynamics we have an extra class, due to the fact that  $\by_i $ now differs from agent to agent.

\begin{lemma}
\label{lem:self-equil1}
The system \eqref{eq:update} has the following classes of equilibria
\benu
\item consensus: $ \bx_i = \bv_k $ $ \forall \, i$, and $ k=1, \ldots, d$;
\item bipartite consensus: $ \bx_i = \pm \bv_k $ $ \forall \, i$ and $ k=1, \ldots, d$;
\item polygonal equilibria: $ \{ \bx_i \in \mathbb{S}^{d-1} \; \text{s.t.} \;  \bV \sum_{j=1}^n A_{ij}(\bx) \bx_j =0\} $, $ i=1, \ldots, n$;
\item clustering equilibria: $  \{ \bx_i \in \mathbb{S}^{d-1} \; \text{s.t.} \;  \varphi_i \bx_i  = \bV \sum_{j=1}^n A_{ij}(\bx) \bx_j   \}$ for some scalars $ \varphi_i$, $ i=1, \ldots, n$, $ k=1, \ldots, d$.
\eenu

\end{lemma}
\proof
Once we replace $ \by $ with $\by_i  =  \sum_{j=1}^n A_{ij}(\bx)  \bV \bm{\bx}_j $, the proof of the first three cases is identical to that of Lemma~\ref{lem:moja-equil1}. 
The clustering equilibria correspond to $ \by_i \in \ker(I - \bx_i \bx_i^T) $, or, equivalently, $ \by_i $ aligned with $ \bx_i $, $ i=1, \ldots, n$, but not necessarily aligned with an eigenvector $ \bv_k$, i.e., $ \by_i = \varphi_i \bx_i $ for some scalar $ \varphi_i$, but possibly $ \by_i \nparallel  \bv_k $ $ \forall \, k$.  
\qed

\begin{remark}
Clustering equilibria own their name to the fact that typically multiple agents are found at the same value. In particular we say that $ \bx_1, \ldots, \bx_n$ are at an $ m$-clustering equilibrium point if $ \bx_i \in \{ \bm{w}_1, \ldots, \bm{w}_m \}$ with $ 1 \leq m \leq n$, i.e., if $ \bx_1, \ldots, \bx_n$ cluster at the $m$ vectors $ \bm{w}_1, \ldots, \bm{w}_m$. Notice that some $ \bm{w}_i $ can be eigenvectors of $\bV$.
\end{remark}

\begin{remark}
A clustering equilibrium can be a consensus, i.e, $ \bx_i  = \bx_j $ for all $ i,j$, but with $ \bx_i $ which is not aligned with any eigenvector of $\bV$, i.e., $ \bx_i \nparallel  \bv_k $ $ \forall \, k$. We refer to these as 1-clustering, while the ``consensus'' characterization is reserved for the case $ \bx_i \parallel \bv_k $ for some $k$. A 2-clustering instead typically is s.t. $ \bx_i \in \{ \bm{w}_1, \bm{w}_ 2\}$ $ \forall \, i$, with $ \bm{w}_1 \neq - \bm{w}_2 $.
\end{remark}

\begin{proposition}
\label{prop:clustering-eq}
The system  \eqref{eq:update} has a clustering equilibrium $ \bx \in (\mathbf{S}^{d-1})^n $ iff $ \exists $ scalars $ \theta_1, \ldots, \theta_n $ s.t. the matrix $I_n \otimes I_d  -(\Theta \otimes I_d) (A(\bx) \otimes I_d)(I_n\otimes \bV)$ is singular, where $ \Theta=\diag(\theta_1, \ldots, \theta_n )$. 
\end{proposition}
\proof 
In vector form, the clustering equilibrium condition $ \bx_i =  \theta_i  \bV \sum_{j=1}^n A_{ij}(\bx)  \bm{\bx}_j $ becomes $ \bx = (\Theta \otimes I_d) (A(\bx) \otimes I_d)(I_n\otimes \bV)  \bx$, where $ \bx^T = [ \bx_1^T \, \ldots\, \bx_n^T ]$. 
Rewriting it as $  (I_n \otimes I_d -(\Theta \otimes I_d) (A(\bx) \otimes I_d)(I_n\otimes \bV)) \bx =0 $, the algebraic equation has nontrivial solutions if and only if the matrix $I_n \otimes I_d  -(\Theta \otimes I_d) (A(\bx) \otimes I_d)(I_n\otimes \bV)$ is singular.
\qed

\begin{remark}
From the proof of Proposition~\ref{prop:clustering-eq}, it is $ \bx = (\Theta \otimes I_d) (A(\bx) \otimes I_d)(I_n\otimes \bV)  \bx$, where $ \bx^T = [ \bx_1^T \, \ldots\, \bx_n^T ]$. It follows that if $ \bx_1, \ldots, \bx_n$ form an $m$-clustering equilibrium point, then also the antipodal point  $ -\bx_1, \ldots, -\bx_n $ is an $ m$-clustering equilibrium point because $ A(\bx) = A(-\bx)$. If in a clustering equilibrium only some (but not all) $ \bx_i $ flip sign, then the new state is typically no longer an equilibrium point.
\end{remark}

While the attention matrix $ A(\bx) $ is a function of the state even at the equilibrium point (with the exception of a consensus state), its rank is however fixed for various classes of equilibria. 

\begin{lemma}
\label{lem:cons-Aij}
For a consensus state $ \bx_i = \bv_k $ $ \forall \, i$ we have $ A_{ij}(\bx) = \frac{1}{n} $, i.e., the attention matrix $ A(\bx) $ is the rank-1 uniform matrix $ A(\bx)  =\frac{1}{n} \mathds{1} \mathds{1}^T $. For a bipartite consensus state $ \bx_i  = \pm \bv_k $ $ \forall \, i$, the attention matrix $A(\bx)$ has rank 2.
For an $m$-clustering equilibrium $ \bx$ the rank of $ A(\bx)$ is $m$. \end{lemma}
\proof
At a consensus equilibrium $ \bx_i = \bv_k $ for all $ i $, and $ A_{ij}(\bx) $ is composed of all equal terms
\[ 
A_{ij}(\bx) = \frac{e^{  \bv_k^T \bQ^T\bK\bv_k  } }{\sum_{\ell=1}^n e^{\bv_k^T \bQ^T\bK \bv_k } }= \frac{1}{n}.
\]
At a bipartite consensus point $ \bx_i = \pm \bv_k $, split the $n$ tokens into two sets $ \mathcal{V}_1 $ and $ \mathcal{V}_2 $, $ \mathcal{V}_1\cup \mathcal{V}_2 = \{ 1, \ldots, n \}$,  according to whether $ \bx_i =\bv_k $ or $ \bx_i =- \bv_k$ at equilibrium. Assume that $ n_1 = | \mathcal{V}_1|  $ tokens are equal to $ \bv_k $ and $ n_2 = | \mathcal{V}_2 | $ equal to $ -\bv_k$, with $ n_1 + n_2 =n$.

Four different entries appear in $ A_{ij}(\bx)$, two due to the denominator (depending on whether $ \bx_i =\bv_k $ or $ \bx_i =- \bv_k $), and two due to the numerator (depending on whether $ \bx_i$ and $ \bx_j $ have the same sign). 
%\[ 
%A_{ij}(\bx) = \frac{ e^{ \pm \bv_k^T \bQ^T\bK \bv_k } }{n_1 e^{\pm \bv_k^T \bQ^T\bK \bv_k} + n_2 e^{\mp \bv_k^T \bQ^T\bK \bv_k } } 
%\]
More specifically, denoting  $ \alpha_1^k = e^{  \bv_k^T \bQ^T\bK \bv_k }  $ and $ \alpha_2^k = e^{ - \bv_k^T \bQ^T\bK \bv_k } $, then the entries of the attention matrix are 
\[ 
A_{ij}(\bx) 
= \begin{cases} 
\frac{\alpha_1^k}{n_1 \alpha_1^k + n_2 \alpha_2^k} & \text{if}\; i \in \mathcal{V}_1, j \in \mathcal{V}_1 \\ 
\frac{\alpha_2^k}{n_1 \alpha_1^k + n_2 \alpha_2^k}  & \text{if}\; i \in \mathcal{V}_1, j \in \mathcal{V}_2 \\ 
\frac{\alpha_2^k}{n_1 \alpha_2^k + n_2 \alpha_1^k}   & \text{if}\; i \in \mathcal{V}_2, j \in \mathcal{V}_1 \\ 
\frac{\alpha_1^k}{n_1 \alpha_2^k + n_2 \alpha_1^k}   & \text{if}\; i \in \mathcal{V}_2, j \in \mathcal{V}_2  .
   \end{cases}
\]
Letting $ \beta_1^k = n_1 \alpha_1^k + n_2 \alpha_2^k$, and $ \beta_2^k = n_1 \alpha_2^k + n_2 \alpha_1^k$ and assuming w.l.o.g. that the first $ n_1 $ agents are in $ \mathcal{V}_1 $ and the last $ n_2 $ in $ \mathcal{V}_2 $, the attention matrix is 
\beq
A(\bx) = \begin{bmatrix} \frac{1}{\beta_1^k} ( \alpha_1^k & \ldots & \ldots & \alpha_1^k & \alpha_2^k & \ldots & \alpha_2^k) \\
\vdots & & & \vdots & \vdots & & \vdots \\
\frac{1}{\beta_1^k} ( \alpha_1^k & \ldots & \ldots & \alpha_1^k & \alpha_2^k & \ldots & \alpha_2^k) \\
\frac{1}{\beta_2^k} ( \alpha_2^k &  \ldots &  \ldots & \alpha_2^k & \alpha_1^k & \ldots &\alpha_1^k) \\
\vdots & & & \vdots  & \vdots & &  \vdots \\
\frac{1}{\beta_2^k} ( \alpha_2^k & \ldots &  \ldots & \alpha_2^k & \alpha_1^k & \ldots & \alpha_1^k) 
\end{bmatrix} ,
\label{eq:A-bipart}
\eeq
from which it is obvious that the rank must be 2. 

As for an $m$-clustering equilibrium point with vectors $ \bm{w}_1, \ldots, \bm{w}_m$, following a similar procedure it is easy to realize that since $ \bm{w}_i \neq \bm{w}_j $ and $ \bQ^T \bK$ is not symmetric,  there are $ m^2 $ different entries in the numerators of $A_{ij}$, and only $m$ in the denominators. Rearranging the entries as in \eqref{eq:A-bipart}, each row of $A$ has exactly $m$ different terms, and a block counting argument leads to $ \rank(A)=m$.
\qed

\subsubsection{Stability analysis}

As in the proof of Lemma~\ref{lem:cons-Aij}, let $ \alpha_1^k = e^{  \bv_k^T \bQ^T\bK \bv_k }  $, $ \alpha_2^k = e^{ - \bv_k^T \bQ^T\bK \bv_k } $ and $ \beta_1^k = n_1 \alpha_1^k + n_2 \alpha_2^k$, $ \beta_2^k = n_1 \alpha_2^k + n_2 \alpha_1^k$. Denote also $ \gamma_1^k = \frac{n_1 \alpha_1^k - n_2 \alpha_2^k }{\beta_1^k} $ and $ \gamma_2^k =  \frac{n_2 \alpha_1^k - n_1 \alpha_2^k }{\beta_2^k} $.

\begin{theorem}
\label{thm:stab-main-self}
For the self-attention dynamics~\eqref{eq:update}, under Assumption~\ref{ass:V-sym}, and, if $ -\lambda_d > \lambda_1 $, for $ \eta<  \frac{2}{ |\lambda_1 + \lambda_d |} $, the consensus equilibrium associated to the principal eigenvector $ \bv_1 $ of $\bV$ is always locally asymptotically stable, while the other consensus equilibria $ \bv_k$, $ k=2, \ldots, d $, are all unstable. 
A bipartite consensus equilibrium $ \bx_i = \pm \bv_k$ $ \forall \, i$ is asymptotically stable if $ \eta< \min\left( \frac{1}{\lambda_1 } , \, \frac{1}{ |\lambda_d |} \right) $ and the following inequalities are all satisfied
\benu
\item  
$ \frac{ |1 - (1-\eta \lambda_k \gamma_i^k)  \eta  \lambda_k\gamma_i^k | } {1+ \eta \lambda_k \gamma_i^k} <1 $, $ i =1,2 $;  
\item  $  a_j^k d_j^k - b_j^k c_j^k < 1 $ and $ a_j^k + d_j^k <1+ a_j^k d_j^k - b_j^k c_j^k $, $\, \forall \, j=1, \ldots, k-1, k+1,\ldots,  d $, where
 \beq
 \begin{split}
 a_j^k & = \frac{ 1+ \eta n_ 1 \frac{\alpha_1^k}{\beta_1^k}  \lambda_j  }{1+ \eta \lambda_k \gamma_1^k} , \\
 b_j^k  &= \frac{ \eta n_2 \frac{\alpha_2^k}{\beta_1^k}  \lambda_j  }{1+ \eta \lambda_k \gamma_1^k}, \\
 c_j^k & = \frac{ \eta n_1 \frac{\alpha_2^k}{\beta_2^k}   \lambda_j  }{1+ \eta \lambda_k \gamma_2^k} , \\
 d_j^k  & =\frac{  1+  \eta n_ 2 \frac{\alpha_1^k}{\beta_2^k}   \lambda_j  }{1+ \eta \lambda_k \gamma_2^k  }
 \end{split}
 \label{eq:a-d} 
 \eeq
 \item $  \frac{1}{1+ \eta \lambda_k \gamma_i^k} <1$, $ i=1,2$.
 \eenu
%and $ c_j^k b_j^k > a_j^k d_j^k $ $ \forall \, j=1, \ldots, k-1, k+1,\ldots,  d $.
All polygonal equilibria are unstable.
\end{theorem}

The proof, based on Lyapunov indirect method, is broken down into various lemmas. 
As in Section~\ref{sec:multi-oja}, let $ f_i(\bx)   $ be the right-hand side of \eqref{eq:update} and $ f(\bx)  $ its vectorization. The (extrinsic, see Remark~\ref{rem:two-classes}) Jacobian of $ f(\bx)  $  is as follows.

\begin{lemma}
\label{lem:Jacob-self}
The Jacobian of  \eqref{eq:update}, $ J(\bx) = \pde{f(\bx)}{\bx}= \begin{bmatrix}\pde{f_i(\bx)}{\bx_h} \end{bmatrix} $, has the following blocks:
 \beq
 \label{eq:jacon-atten1}
 \begin{split}
 \pde{f_i(\bx)}{\bx_h} = &  
 \frac{I \delta_{ih} + \eta \bV \left( \sum_j  \bx_j \pde{A_{ij} (\bx) }{\bx_h}   + A_{ih}(\bx) I\right) }{\|  \bx_i  +  \eta \bV \sum_j A_{ij}(\bx) \bx_j \| } \\
&  -  \frac{  \left(  \bx_i  + \eta \bV \sum_j A_{ij}(\bx) \bx_j \right) \pde{}{\bx_h} \|  \bx_i  +  \eta \bV \sum_j A_{ij}(\bx) \bx_j \| } {\|  \bx_i  +  \eta \bV \sum_j A_{ij}(\bx) \bx_j \|^2 }
\end{split}
 \eeq
 where
 \beq
 \pde{A_{ij} (\bx) }{\bx_h} 
  =  \left(-  \bx_i^T \bQ^T \bK A_{ih}(\bx)  +  \bx_i^T \bQ^T \bK  \delta_{jh} + \left( \bx_j^T \bK^T \bQ  - \sum_\ell \bx_\ell^T \bK^T \bQ A_{i\ell} (\bx) \right) \delta_{ih} \right) A_{ij} (\bx)
 \label{eq:partial-A}
 \eeq
and
\beq
\label{eq:Jac-attent1}
\begin{split}
&\pde{\| \bx_i  +  \eta \bV \sum_j A_{ij}(\bx) \bx_j\|}{\bx_h } = \\
&  = \left( \eta \bx_i^T \bV  + \eta^2 \sum_j \bx_j^T A_{ij}(\bx) \bV^T (I - \bx_i \bx_i^T)  \bV   \right)  \left( \sum_j\bx_j  \pde{A_{ij}(\bx)}{\bx_h} + A_{ih}(\bx) I\right)  \\
& 
+ \left( \left(  \eta - \eta^2    \bx_i^T  \bV \sum_j A_{ij}(\bx) \bx_j \right)  \sum_j A_{ij}(\bx) \bx_j^T \bV^T  \right) \delta_{ih}
\end{split}
\eeq
\end{lemma}

\proof
The partial derivatives of $ A_{ij}(\bx) $ are 
 \[
 \begin{split}
 \pde{A_{ij} (\bx) }{\bx_h} 
  = & \pde{}{\bx_h} \frac{e^{\langle \bQ \bx_i, \bK \bx_j \rangle } }{\sum_{\ell=1}^n e^{\langle \bQ \bx_i, \bK \bx_\ell \rangle } }\\
 =  & -  \bx_i^T \bQ^T \bK A_{ij} (\bx) A_{ih}(\bx) +  \bx_i^T \bQ^T \bK A_{ij} (\bx)\delta_{jh} \\
&  + \bx_j^T \bK^T \bQ A_{ij} (\bx)\delta_{ih}  - \sum_\ell \bx_\ell^T \bK^T \bQ A_{i\ell} (\bx)A_{ij}(\bx) \delta_{ih} 
 \end{split}
 \]
 or \eqref{eq:partial-A} after regrouping.
Similarly to the proof of Lemma~\ref{lem:Jac-moja},  the series expansion of the normed term is
\[
\begin{split}
\| & \bx_i  +  \eta \bV \sum_j A_{ij}(\bx) \bx_j \| =  \left(  (  \bx_i +  \eta \by_i )^T (  \bx_i + \eta \by_i) \right)^{\frac{1}{2}}  
% = \left( 1 + 2  \eta \bx_i ^T \by_i + \eta^2 \| \by_i \|^2   \right)^{\frac{1}{2}} 
\\
& \approx  1 + \frac{1}{2} \left( 2  \eta \bx_i ^T \by_i + \eta^2 \| \by_i \|^2  \right) - \frac{1}{2}  \eta^2 ( \bx_i^T \by_i ) ^2 + O(\eta^3) \\
& =  1 + \eta \bx_i^T \bV \sum_j A_{ij}(\bx) \bx_j + \frac{\eta^2}{2} \left( \sum_j A_{ij}(\bx) \bx_j^T \bV^T \bV \sum_j A_{ij}(\bx) \bx_j - ( \bx_i^T \bV \sum_j A_{ij}(\bx) \bx_j)^2 \right)  \\
&+ O(\eta^3)
\end{split}
\]
Differentiating it
\[
\begin{split}
& \pde{\| \bx_i   +  \eta \bV \sum_j A_{ij}(\bx) \bx_j\|}{\bx_h } = \\
& \qquad =  \,  \eta \,  \bx_i^T \bV \left(  \sum_j \bx_j \pde{A_{ij}(\bx)}{\bx_h} + A_{ih}(\bx) I \right) \\
& \qquad \qquad + \eta^2 \left(  \sum_j \bx_j^T A_{ij}(\bx) \bV^T \bV  \left(  \sum_j \bx_j\pde{A_{ij}(\bx)}{\bx_h}  + A_{ih}(\bx)I \right) 
 \right. \\
& \qquad \qquad \left. -  \left( \bx_i^T \bV \sum_j A_{ij}(\bx) \bx_j \right) \bx_i^T \bV \left(  \sum_j\bx_j \pde{A_{ij}(\bx)}{\bx_h}  + A_{ih}(\bx) I\right) \right) \\
& \qquad \qquad + \left( \eta  \sum_j \bx_j^T A_{ij}(\bx) \bV^T - \eta^2  \left( \bx_i^T \bV \sum_j A_{ij}(\bx) \bx_j \right)\sum_j \bx_j^T A_{ij}(\bx) \bV^T \right) \delta_{ih}
\end{split}
\]
or \eqref{eq:Jac-attent1} after regrouping.
One gets the result inserting these quantities into \eqref{eq:jacon-atten1}.
% \[
% \pde{f_i(\bx)}{\bx_h} = 
% \frac{I \delta_{ih} + \eta \bV \left( \sum_j \pde{A_{ij}(\bx)}{\bx_h} \bx_j + A_{ih}(\bx) \right) }{\|  \bx_i  +  \eta \bV \sum_j A_{ij}(\bx) \bx_j \| } 
% -  \frac{  \left(  \bx_i  + \eta \bV \sum_j A_{ij}(\bx) \bx_j \right) \pde{}{\bx_h} \|  \bx_i  +  \eta \bV \sum_j A_{ij}(\bx) \bx_j \| } {\|  \bx_i  +  \eta \bV \sum_j A_{ij}(\bx) \bx_j \|^2 }
% \]
% one gets the result.
 \qed

The following lemma states that the behavior of the self-attention model \eqref{eq:update} at a consensus point is locally identical to that of the multiagent Oja model \eqref{eq:update-moja}.
\begin{lemma}
\label{lem:lin-stab-self1}
At a consensus point $ \bv_k$, the eigenvalues of the Jacobian matrix $ J(\bv_k) $ are those indicated in Lemma~\ref{lem:moja-F-eigenvalues}.
In particular, , under Assumption~\ref{ass:V-sym}, and, if $ -\lambda_d > \lambda_1 $, for $ \eta<  \frac{2}{ |\lambda_1 + \lambda_d |} $, the consensus point $ \bx_i =\bv_1$ $ \forall\, i$ is locally asymptotically stable for \eqref{eq:update}, while the remaining consensus equilibria $ \bx_i= \bv_k$ $ \forall \, i $, with $ k= 2, \ldots, d $, are all unstable. 
\end{lemma}

\proof
From Lemma~\ref{lem:cons-Aij}, in a consensus equilibrium the attention matrix is the uniform matrix $ A(\bm{\bv_k}) =\frac{1}{n} \mathds{1}\mathds{1}^T$.
It follows that the Jacobian at $ \bx_i = \bv_k $, $ J(\bv_k) $, is still given by \eqref{eq:JacobF-moja}, and hence, from Lemma~\ref{lem:moja-consensus-stab}, that $ \bv_1 $ is locally asymptotically stable while $ \bv_k $, $k=2, \ldots, d $, are all unstable.
\qed

We consider now a bipartite consensus equilibrium and assume w.l.o.g. that the first $ n_1 $ agents are in $ \mathcal{V}_1 $ and the last $ n_2 $ in $ \mathcal{V}_2 $.

\begin{lemma}
\label{lem:Jac-bip-cons}
At a bipartite consensus point $ \pm  \bv_k $, the Jacobian $ J(\pm\bv_k) $ is given by 
\[
\begin{split}
J(\pm\bv_k) & = 
 \begin{bmatrix} 
J_o^{11} & \ldots &  J_o^{11} &  J_o^{12}   & \ldots &  J_o^{12} \\ 
\vdots & & \vdots & \vdots & & \vdots \\
J_o^{11} & \ldots &  J_o^{11} &  J_o^{12}   & \ldots &  J_o^{12} \\ 
J_o^{21} & \ldots &  J_o^{21} &  J_o^{22}   & \ldots &  J_o^{22} \\ 
\vdots & & \vdots & \vdots & & \vdots \\
J_o^{21} & \ldots &  J_o^{21} &  J_o^{22}   & \ldots &  J_o^{22} \\ 
 \end{bmatrix} 
 +
  \begin{bmatrix} J_d^1 & \\ &  \ddots \\ & &  J_d^1 \\
& & &  J_d^2 \\ & & & & \ddots \\ & & & & &  J_d^2   \end{bmatrix} 
,
 \end{split}
\]
where
\beq\label{eq:J_bip}
 \begin{split}
 J_o^{11}& =  \frac{
  \eta  \frac{\alpha_1^k}{\beta_1^k}  \left(\sigma_1  \bV - \lambda_k \bv_k \bv_k^T  - ( 1-\sigma_1) (1-\gamma_1^k) \lambda_k \bv_k \bv_k^T \bQ^T \bK  \right) }
{1+ \eta \lambda_k \gamma_1^k} \\
 J_o^{12} & =  \frac{
  \eta  \frac{\alpha_2^k}{\beta_1^k}  \left( \sigma_1 \bV - \lambda_k \bv_k \bv_k^T + (1- \sigma_1) (1+\gamma_1^k) \lambda_k \bv_k \bv_k^T \bQ^T \bK  \right) }
{1+ \eta \lambda_k \gamma_1^k} 
 \\
  J_o^{21} &  =  \frac{
  \eta  \frac{\alpha_2^k}{\beta_2^k}  \left( \sigma_2 \bV - \lambda_k \bv_k \bv_k^T + (1-\sigma_2) (1+\gamma_2^k) \lambda_k \bv_k \bv_k^T \bQ^T \bK  \right) }
{1+ \eta \lambda_k \gamma_2^k} 
 \\
   J_o^{22} & =    \frac{
  \eta  \frac{\alpha_1^k}{\beta_2^k}  \left( \sigma_2 \bV - \lambda_k \bv_k \bv_k^T - (1-\sigma_2) (1-\gamma_2^k) \lambda_k \bv_k \bv_k^T \bQ^T \bK \right) }
{1+ \eta \lambda_k \gamma_2^k} 
 \\
 J_d^1 &  = \frac{\sigma_1 I - ( 1- \sigma_1) \eta (1-(\gamma_1^k)^2) \lambda_k \bv_k \bv_k^T \bQ^T \bK - (1-\eta \lambda_k \gamma_1^k)  \eta  \lambda_k \gamma_1^k  \bv_k \bv_k^T }
 {1+ \eta \lambda_k \gamma_1^k} \\
  J_d^2 &  = \frac{\sigma_2 I -(1-\sigma_2)  \eta (1-(\gamma_2^k)^2) \lambda_k \bv_k \bv_k^T \bQ^T \bK  - (1-\eta \lambda_k \gamma_2^k) \eta  \lambda_k \gamma_2^k  \bv_k \bv_k^T }
{1+ \eta \lambda_k \gamma_2^k} \\
 \end{split}
 \eeq
 with $ \sigma_1 = \sign({1+ \eta \lambda_k \gamma_1^k}) $ and $ \sigma_2 = \sign({1+ \eta \lambda_k \gamma_2^k}) $.
 
\end{lemma}

\proof
Some tedious but elementary calculations yield:
 \[
 \left.  \bV \sum_j A_{ij}(\bx) \bx_j  \right|_{\bx_{i,j} = \pm \bv_k} = \begin{cases} 
 \gamma_1^k \lambda_k \bv_k &  \text{ if } \; \; i \in \mathcal{V}_1 \\
- \gamma_2^k \lambda_k \bv_k &  \text{ if } \; \; i \in \mathcal{V}_2 
\end{cases}
\]
\[
 \left. \bx_i^T  \bV \sum_j A_{ij}(\bx) \bx_j  \right|_{\bx_{i,j} = \pm \bv_k} =
 \begin{cases} \bv_k^T \bV \gamma_1^k \bv_k = \lambda_k \gamma_1^k & \text{ if } \; \; i \in \mathcal{V}_1 \\
  - \bv_k^T \bV (- \gamma_2^k) \bv_k =  \lambda_k \gamma_2^k & \text{ if } \; \; i \in \mathcal{V}_2 
  \end{cases}
  \]
  Since $ \bv_k^T (I - \bv_k \bv_k^T) =0 $, it is
\[
\left. \sum_j \bx_j^T A_{ij}(\bx) \bV^T (I - \bx_i \bx_i^T)  \bV \right|_{\bx_{i,j} = \pm \bv_k} =0 .
\]
For the derivatives of the attention matrix, we have
\[
\!\!\!
\pde{A_{ij}(\bx)}{\bx_h} = \begin{cases}
\left( - \frac{\alpha_1^k}{\beta_1^k} + \delta_{jh}\right) \frac{\alpha_1^k}{\beta_1^k} \bv_k^T \bQ^T \bK 
+ \left(( 1 -\gamma_1^k) \frac{\alpha_1^k}{\beta_1^k}   \bv_k^T \bK^T \bQ \right) \delta_{ih} & \text{if } \; \; i  \in \mathcal{V}_1, \; h \in \mathcal{V}_1, \;  j \in \mathcal{V}_1 \\
- \frac{\alpha_1^k}{\beta_1^k}  \frac{\alpha_2^k}{\beta_1^k}  \bv_k^T \bQ^T \bK 
-( 1 +\gamma_1^k) \frac{\alpha_2^k}{\beta_1^k}  \left( \bv_k^T \bK^T \bQ \right) \delta_{ih} & \text{if } \; \; i \in \mathcal{V}_1 ,\;  h \in \mathcal{V}_1 , \; j\in \mathcal{V}_2 \\
- \frac{\alpha_2^k}{\beta_1^k} \frac{\alpha_1^k}{\beta_1^k} \bv_k^T \bQ^T \bK 
& \text{if } \; \; i \in \mathcal{V}_1, \; h\in \mathcal{V}_2, \;  j \in \mathcal{V}_1 \\
\left( - \frac{\alpha_2^k}{\beta_1^k} + \delta_{jh}\right) \frac{\alpha_2^k}{\beta_1^k} \bv_k^T \bQ^T \bK 
& \text{if }  \; \; i \in \mathcal{V}_1, \; h\in \mathcal{V}_2 , \;  j \in \mathcal{V}_2  \\
\left(  \frac{\alpha_2^k}{\beta_2^k} - \delta_{jh}\right) \frac{\alpha_2^k}{\beta_2^k} \bv_k^T \bQ^T \bK 
& \text{if }  \; \; i \in \mathcal{V}_2, \; h \in \mathcal{V}_1, \; j \in \mathcal{V}_1  \\
  \frac{\alpha_2^k}{\beta_2^k}  \frac{\alpha_1^k}{\beta_2^k} \bv_k^T \bQ^T \bK 
& \text{if }  \; \; i \in \mathcal{V}_2, \; h \in \mathcal{V}_1 , \;  j \in \mathcal{V}_2  \\
 \frac{\alpha_1^k}{\beta_2^k}  \frac{\alpha_2^k}{\beta_2^k} \bv_k^T \bQ^T \bK 
+ \left(( 1 +\gamma_2^k) \frac{\alpha_2^k}{\beta_2^k}   \bv_k^T \bK^T \bQ \right) \delta_{ih} & \text{if } \; \; i  \in \mathcal{V}_2, \; h \in \mathcal{V}_2, \;  j \in \mathcal{V}_1  \\
\left(  \frac{\alpha_1^k}{\beta_2^k} - \delta_{jh} \right) \frac{\alpha_1^k}{\beta_2^k} \bv_k^T \bQ^T \bK  - \left(( 1 -\gamma_2^k) \frac{\alpha_1^k}{\beta_2^k}   \bv_k^T \bK^T \bQ \right) \delta_{ih} & \text{if } \; \; i  \in \mathcal{V}_2, \; h \in \mathcal{V}_2, \;  j \in \mathcal{V}_2 \\
\end{cases}
\]
and hence
\[
\sum_j \bx_j  \pde{A_{ij}(\bx)}{\bx_h}  = \begin{cases}
\left(  \frac{\alpha_1^k}{\beta_1^k} (1 - \gamma_1^k ) + (1-(\gamma_1^k)^2 ) \delta_{ih} \right) \bv_k  \bv_k^T \bQ^T \bK  & \text{if } \; \; i  \in \mathcal{V}_1, \; h \in \mathcal{V}_1\\
-  \frac{\alpha_2^k}{\beta_1^k} (1 + \gamma_1^k ) \bv_k \bv_k^T \bQ^T \bK   & \text{if } \; \; i  \in \mathcal{V}_1, \; h \in \mathcal{V}_2\\ 
-\frac{\alpha_2^k}{\beta_2^k} (1 + \gamma_2^k ) \bv_k  \bv_k^T \bQ^T \bK  & \text{if } \; \; i  \in \mathcal{V}_2, \; h \in \mathcal{V}_1\\
\left(  \frac{\alpha_1^k}{\beta_2^k} (1 - \gamma_2^k ) + (1-(\gamma_2^k)^2 )  \delta_{ih} \right)  \bv_k \bv_k^T \bQ^T \bK   & \text{if } \; \; i  \in \mathcal{V}_2, \; h \in \mathcal{V}_2\\
\end{cases}
\]
Concerning the normed terms, 
\[
\left. \| \bx_i  +  \eta \bV \sum_j A_{ij}(\bx) \bx_j\| \right|_{\bx_{i,j} = \pm \bv_k}= \begin{cases}
| 1+ \eta \lambda_k \gamma_1^k | & \text{if } \; \; i  \in \mathcal{V}_1 \\
| 1+ \eta \lambda_k \gamma_2^k | & \text{if } \; \; i  \in \mathcal{V}_2
\end{cases}
\]
and
\beq
\label{eq:Jac-attent1}
\begin{split}
&\left. \pde{\| \bx_i  +  \eta \bV \sum_j A_{ij}(\bx) \bx_j\|}{\bx_h }  \right|_{\bx_{i,h,j} = \pm \bv_k} =  \\
& = \begin{cases} 
\eta \lambda_k \bv_k^T   \left( \frac{\alpha_1^k}{\beta_1^k} \left(  1+ (1 - \gamma_1^k ) \bQ^T \bK \right)  \right.  & \\ \qquad   + \left[ (1-(\gamma_1^k)^2 ) \bQ^T \bK  +  (1-\eta \lambda_k \gamma_1^k ) \gamma_1^k  \right] \delta_{ih}  \Big) &  \text{if } \; \; i  \in \mathcal{V}_1, \; h \in \mathcal{V}_1\\
\eta \lambda_k \bv_k^T  \frac{\alpha_2^k}{\beta_1^k} \left( 1 - (1 + \gamma_1^k ) \bQ^T \bK  \right)    &  \text{if } \; \; i  \in \mathcal{V}_1, \; h \in \mathcal{V}_2\\
- \eta \lambda_k  \bv_k^T  \frac{\alpha_2^k}{\beta_2^k} \left( 1 - (1 + \gamma_2^k ) \bQ^T \bK  \right)    &  \text{if } \; \; i  \in \mathcal{V}_2, \; h \in \mathcal{V}_1\\
- \eta \lambda_k  \bv_k^T \left( \frac{\alpha_1^k}{\beta_2^k}  \left( 1 + (1 - \gamma_2^k )  \bQ^T \bK  \right)  \right. \\
\qquad  + \left[ (1-(\gamma_2^k)^2 ) \bQ^T \bK  + (1-\eta \lambda_k \gamma_2^k ) \gamma_2^k \right] \delta_{ih}  \Big) &  \text{if } \; \; i  \in \mathcal{V}_2, \; h \in \mathcal{V}_2.\\
\end{cases}
\end{split}
\eeq
Combining these terms together, we obtain the two components of $\pde{f_i(\bx)}{\bx_h} $
\[
\begin{split}
& \left.  \frac{I \delta_{ih} + \eta \bV \left( \sum_j  \bx_j  \pde{A_{ij} (\bx) }{\bx_h}  + A_{ih}(\bx) I \right) }{\|  \bx_i  +  \eta \bV \sum_j A_{ij}(\bx) \bx_j \| }\right|_{\bx_{i,j} = \pm \bv_k} = \\
& \qquad =
 \begin{cases}
 \frac{\eta \frac{\alpha_1^k}{\beta_1^k} \left( \bV +  (1-\gamma_1^k)\lambda_k \bv_k \bv_k^T \bQ^T \bK \right)  + \left[ I + \eta  (1-(\gamma_1^k)^2) \lambda_k \bv_k \bv_k^T \bQ^T \bK  \right]\delta_{ih} }
 {| 1+ \eta \lambda_k \gamma_1^k | } &  \text{if } \; \; i  \in \mathcal{V}_1, \; h \in \mathcal{V}_1\\
 \frac{\eta  \frac{\alpha_2^k}{\beta_1^k} \left(\bV - (1+\gamma_1^k) \lambda_k \bv_k \bv_k^T \bQ^T \bK  \right) }
 {| 1+ \eta \lambda_k \gamma_1^k | } &  \text{if } \; \; i  \in \mathcal{V}_1, \; h \in \mathcal{V}_2\\
  \frac{\eta  \frac{\alpha_2^k}{\beta_2^k} \left( \bV - (1+\gamma_2^k) \lambda_k \bv_k \bv_k^T \bQ^T \bK  \right) }
 {| 1+ \eta \lambda_k \gamma_2^k | } &  \text{if } \; \; i  \in \mathcal{V}_2, \; h \in \mathcal{V}_1\\
 \frac{\eta \frac{\alpha_1^k}{\beta_2^k} \left( \bV + (1-\gamma_2^k)\lambda_k \bv_k  \bv_k^T \bQ^T \bK \right) + \left[ I + \eta  (1-(\gamma_2^k)^2) \lambda_k \bv_k \bv_k^T \bQ^T \bK \right]\delta_{ih}}
 {| 1+ \eta \lambda_k \gamma_2^k | } &  \text{if } \; \; i  \in \mathcal{V}_2, \; h \in \mathcal{V}_2\\
 \end{cases}
 \end{split}
 \]
 \[
\begin{split}
 &\left.  \frac{  \left(  \bx_i  + \eta \bV \sum_j A_{ij}(\bx) \bx_j \right) \pde{}{\bx_h} \|  \bx_i  +  \eta \bV \sum_j A_{ij}(\bx) \bx_j \| } {\|  \bx_i  +  \eta \bV \sum_j A_{ij}(\bx) \bx_j \|^2 }\right|_{\bx_{i,j} = \pm \bv_k} = \\
 & \;\; 
 = \begin{cases}
 \frac{\eta \lambda_k \bv_k \bv_k^T \left( 
 \frac{\alpha_1^k}{\beta_1^k} \left( 1+ (1-\gamma_1^k)  \bQ^T \bK \right) 
 + \left[ (1-(\gamma_1^k)^2) \bQ^T \bK  +(1-\eta \lambda_k \gamma_1^k) \gamma_1^k  \right] \delta_{ih} \right) }
{1+ \eta \lambda_k \gamma_1^k} &  \text{if } \; \; i  \in \mathcal{V}_1, \; h \in \mathcal{V}_1
\\
 \frac{\eta \lambda_k \bv_k \bv_k^T \frac{\alpha_2^k}{\beta_1^k} \left( 1- (1+\gamma_1^k)  \bQ^T \bK \right) }
{1+ \eta \lambda_k \gamma_1^k} &  \text{if } \; \; i  \in \mathcal{V}_1, \; h \in \mathcal{V}_2
\\
 \frac{\eta \lambda_k \bv_k \bv_k^T \frac{\alpha_2^k}{\beta_2^k} \left( 1- (1+\gamma_2^k) \bQ^T \bK  \right) }
{1+ \eta \lambda_k \gamma_2^k} &  \text{if } \; \; i  \in \mathcal{V}_2, \; h \in \mathcal{V}_1
\\
\frac{\eta \lambda_k \bv_k \bv_k^T \left( \frac{\alpha_1^k}{\beta_2^k} \left( 1+ (1-\gamma_2^k) \bQ^T \bK \right) 
 + \left[  (1-(\gamma_2^k)^2) \bQ^T \bK  + (1-\eta \lambda_k \gamma_2^k) \gamma_2^k  \right] \delta_{ih} \right) }
{1+ \eta \lambda_k \gamma_2^k} &  \text{if } \; \; i  \in \mathcal{V}_2, \; h \in \mathcal{V}_2
\\
\end{cases}
\end{split}
\]
which lead to the expression for $ J(\pm \bv_k)$. 
\qed

To avoid dealing with absolute values in $ 1 + \eta \lambda_k \gamma_i  $, $ i=1,2$, in $ J(\pm \bv_k)$, it is convenient to restrict the step size $ \eta $ opportunely.
\begin{lemma}
\label{lemm:simplif-eta}
When $ \eta< \min\left( \frac{1}{\lambda_1 } , \, \frac{1}{ | \lambda_d |} \right) $ then $ 1+ \eta \lambda_k \gamma_i^k > 0$, $ i=1,2$. 
\end{lemma}
\proof
Notice first that, since $ |\gamma_i^k |\leq 1 $, $  | \lambda_k| \geq | \lambda_k \gamma_i^k | $, hence $\min\left( \frac{1}{\lambda_1 } , \, \frac{1}{ | \lambda_d |} \right)  \leq \frac{1}{ | \lambda_k| } \leq \frac{1}{ | \lambda_k \gamma_i^k |}$. 
Consequently, $ \eta< \min\left( \frac{1}{\lambda_1 } , \, \frac{1}{ | \lambda_d |} \right) $ implies $ \eta |  \lambda_k \gamma_i^k | \leq 1 $, or $ 1- \eta |  \lambda_k \gamma_i^k | >0$, which in turn implies $ 1 + \eta \lambda_k \gamma_i^k \geq 0 $, $ i=1,2$. 
\qed

This limitation on $ \eta$ implies $ \sigma_1 =\sigma_2 =1 $, which leads to a considerable simplification in the Jacobian terms of \eqref{eq:J_bip}, and therefore in the calculation of its eigenvalues.

\begin{lemma}
\label{lem:lin-stab-self2}
When $ \eta< \min\left( \frac{1}{\lambda_1 } , \, \frac{1}{ | \lambda_d |} \right) $, at a bipartite consensus point $ \pm  \bv_k $, the eigenvalues of the Jacobian $ J(\pm\bv_k) $ are %\rednote{if  $ 1+ \eta \lambda_k \gamma_1^k >0 $ and $ 1+ \eta \lambda_k \gamma_2^k >0 $ only???}
 \benu
 \item 
  $ \frac{ 1 - (1-\eta \lambda_k \gamma_1^k)  \eta  \lambda_k\gamma_1^k  }
 {1+ \eta \lambda_k \gamma_1^k}\  $ of multiplicity $ n_1$, and  $ \frac{ 1 - (1-\eta \lambda_k \gamma_2^k)  \eta  \lambda_k\gamma_2^k  }
 {1+ \eta \lambda_k \gamma_2^k}\  $  of multiplicity $ n_2$.
 \item $ \mu_{j,\pm}^k =  \frac{1}{2} \Big( a_j^k+ d_j^k \pm \sqrt{ (a_j^k-d_j^k)^2 + 4 c_j^kb_j^k } \Big) $ for $ j=1, \ldots, k-1, k+1,\ldots, d $, where $ a_j^k$, $ b_j^k$, $ c_j^k $ and $ d_j^k$ are given by \eqref{eq:a-d}.

 \item   $  \frac{1}{1+ \eta \lambda_k \gamma_1^k} $ and $  \frac{1 }{1+ \eta \lambda_k \gamma_2^k} $, of total multiplicity $ nd -n -2d +2 $.
 \eenu

\end{lemma}
\proof
For $ J(\pm \bv_k) $ there are 3 classes of eigenvectors:
 \benu
 \item %The eigenvector of the first class depend on the sign of $ 1+ \eta \lambda_k \gamma_i$, $i=1, 2$. 
 If $ 1+ \eta \lambda_k \gamma_i^k >0$, $i=1, 2$, then choose $ \bm{p}^\ell = [ \,\underbrace{0 \, \ldots \, 0}_{\ell-1} \, \bv_k^T \, \underbrace{0 \, \ldots \, 0}_{n- \ell} \,]^T$. There are $ n$ such eigenvectors, and they are obviously all orthogonal to each other. For $\ell \leq n_1 $ it is 
  \[
 \begin{cases}
 J_o^{11} \bv_k & =  \frac{
  \eta  \frac{\alpha_1^k}{\beta_1^k} \left( \bV - \lambda_k \bv_k \bv_k^T \right)  
 }
{1+ \eta \lambda_k \gamma_1^k} \bv_k 
=  \frac{ \eta  \frac{\alpha_1^k}{\beta_1^k} \left( \lambda_k \bv_k  - \lambda_k \bv_k \right)}
{1+ \eta \lambda_k \gamma_1^k} =0
\\
   J_o^{21}  \bv_k &  =  \frac{
  \eta  \frac{\alpha_2^k}{\beta_2^k}  \left( \bV - \lambda_k \bv_k \bv_k^T \right)  
}
{1+ \eta \lambda_k \gamma_2^k} \bv_k = \frac{
  \eta  \frac{\alpha_2^k}{\beta_2^k}   \left( \lambda_k \bv_k  - \lambda_k \bv_k \right)  
}
{1+ \eta \lambda_k \gamma_2^k} =0
\\
 J_d^1 \bv_k &  = \frac{ I -(1-\eta \lambda_k \gamma_1^k)  \eta  \lambda_k \gamma_1^k \bv_k \bv_k^T }
 {1+ \eta \lambda_k \gamma_1^k} \bv_k 
 = \frac{ 1 -  (1-\eta \lambda_k \gamma_1^k) \eta  \lambda_k \gamma_1^k }
 {1+ \eta \lambda_k \gamma_1^k}\bv_k 
 \\
   \end{cases}
\]
meaning that 
\[
J(\pm \bv_k ) \bp =\frac{ 1 -(1-\eta \lambda_k \gamma_1^k) \eta \lambda_k  \gamma_1^k  }
 {1+ \eta \lambda_k \gamma_1^k}\ \bp
\]
i.e., $ \frac{ 1 - (1-\eta \lambda_k \gamma_1^k)  \eta  \lambda_k\gamma_1^k  }
 {1+ \eta \lambda_k \gamma_1^k}\  $ is an eigenvalue of $ J(\pm \bv_k)$ of multiplicity $ n_1$.
For $ \ell> n_1 $, similar calculations give
\[
 J_o^{12} \bv_k =0, \quad  J_o^{22} \bv_k =0, \quad J_d^2 \bv_k = \frac{ 1 -  (1-\eta \lambda_k \gamma_2^k) \eta  \lambda_k \gamma_2^k }
 {1+ \eta \lambda_k \gamma_2^k}\bv_k 
 \]
 i.e., $ \frac{ 1 - (1-\eta \lambda_k \gamma_2^k)  \eta  \lambda_k\gamma_2^k  }
 {1+ \eta \lambda_k \gamma_2^k}\  $ is an eigenvalue of $ J(\pm \bv_k)$ of multiplicity $ n_2$.

\item Second class: Consider $ d-1$  vectors  $ \bm{q}^h $ s.t. $ \bm{q}^h =[\, \underbrace{ (\bm{w}_1^h)^T \, \ldots  \,  (\bm{w}_1^h)^T}_{n_1 \; \text{times}} \; \underbrace{(\bm{w}_2^h)^T \, \ldots  \, (\bm{w}_2^h)^T}_{n_2 \; \text{times}} ]^T $, with 
$ \bm{w}_1^h = \sum_{j=1}^d \xi_j^{h,1} \bv_j $ and $ \bm{w}_2^h = \sum_{j=1}^d \xi_j^{h,2} \bv_j $.
Choose $ \bw_i^h $ s.t. $ \xi_k^{h,i}=0$, so that $ \bv_k^T \bw_i^h =0$.
We have 
\[
\begin{split} 
J_o^{11} \bm{w}_1^h & = \frac{\eta  \frac{\alpha_1^k}{\beta_1^k}   \sum_{j\neq k } \lambda_j \xi_j^{h,1} \bv_j }{1+ \eta \lambda_k \gamma_1^k} \\
J_o^{12} \bm{w}_2^h & = \frac{\eta  \frac{\alpha_2^k}{\beta_1^k}  \sum_{j\neq k } \lambda_j \xi_j^{h,2} \bv_j }{1+ \eta \lambda_k \gamma_1^k} \\
J_o^{21} \bm{w}_1^h & = \frac{\eta  \frac{\alpha_2^k}{\beta_2^k}  \sum_{j\neq k } \lambda_j \xi_j^{h,1} \bv_j }{1+ \eta \lambda_k \gamma_2^k} \\
J_o^{22} \bm{w}_2^h & = \frac{\eta  \frac{\alpha_1^k}{\beta_2^k}  \sum_{j\neq k } \lambda_j \xi_j^{h,2} \bv_j }{1+ \eta \lambda_k \gamma_2^k} \\
\end{split}
\]
and
\[
\begin{split} 
J_d^{1} \bm{w}_1^h & = \frac{ \bw_1^h }{1+ \eta \lambda_k \gamma_1^k} \\
J_d^{2} \bm{w}_2^h & = \frac{ \bw_2^h  }{1+ \eta \lambda_k \gamma_2^k}
\end{split}
\]
The $ \bm{q}^h $ are obtained solving the algebraic equation  $ J(\pm \bv_k) \bm{q}^h = \mu \bm{q}^h $ where also the eigenvalue $ \mu $ is an unknown.
Expanding we obtain a block of $ n_1 $ identical equations
\[
\begin{split}
 \frac{  \sum_{j\neq k }  \left(  \eta n_ 1 \frac{\alpha_1^k}{\beta_1^k} \xi_j^{h,1}  +  \eta  n_2 \frac{\alpha_2^k}{\beta_1^k}   \xi_j^{h,2} \right)  \lambda_j  \bv_j +  \sum_{j\neq k }  \xi_j^{h,1}   \bv_j }{1+ \eta \lambda_k \gamma_1^k} 
  = \mu  \sum_{j\neq k }  \xi_j^{h,1}   \bv_j 
 \end{split}
 \]
 and another of $ n_2 $ identical equations
 \[
 \begin{split}
 \frac{  \sum_{j\neq k }  \left( \eta n_ 1 \frac{\alpha_2^k}{\beta_2^k}  \xi_j^{h,1}  +  \eta  n_2 \frac{\alpha_1^k}{\beta_2^k}  \xi_j^{h,2} \right)  \lambda_j  \bv_j +  \sum_{j\neq k }  \xi_j^{h,2}   \bv_j }{1+ \eta \lambda_k \gamma_2^k} 
  = \mu  \sum_{j\neq k }  \xi_j^{h,2}   \bv_j  \end{split}
 \]
 whose solution provides both the desired eigenvector $ \bm{w}_i^h $ and the associated eigenvalues $ \mu$. 
Projecting these equations along the eigenvector $ \bv_j $ and rearranging
\[
\begin{split}
\Big(  \overbrace{ 
 \frac{ 1+ \eta n_ 1 \frac{\alpha_1^k}{\beta_1^k}  \lambda_j  }{1+ \eta \lambda_k \gamma_1^k} }^{=a_j^k} - \mu \Big)  \xi_j^{h,1} + 
\Big(  \overbrace{ 
 \frac{ \eta n_2 \frac{\alpha_2^k}{\beta_1^k}  \lambda_j  }{1+ \eta \lambda_k \gamma_1^k} }^{=b_j^k} \Big)  \xi_j^{h,2} =0 \\
\Big( \underbrace{ \frac{ \eta n_1 \frac{\alpha_2^k}{\beta_2^k}   \lambda_j  }{1+ \eta \lambda_k \gamma_2^k}  }_{=c_j^k} \Big) \xi^{h,1}_j  
+ \Big( \underbrace{ \frac{  1+  \eta n_ 2 \frac{\alpha_1^k}{\beta_2^k}   \lambda_j  }{1+ \eta \lambda_k \gamma_2^k  } }_{=d_j^k} - \mu  \Big)  \xi^{h,2}_j  & = 0.
\end{split}
\]
Notice that $ a_j^k$, $b_j^k $, $ c_j^k$ and $ d_j^k $ are independent of the index $h$, hence, to avoid repeated identical algebraic equations, we can take $ h=j$ and drop one index in the $ \xi_j^{h,i}$ variables: $ \xi_j^{h,i} = \xi_j^i $, obtaining
\[
\begin{split}
(a_j^k-\mu) \xi^1_j  +b_j^k  \xi^2_j & =0 \\
c_j^k \xi^1_j  + (d_j^k-\mu)   \xi^2_j  & =0 ,
\end{split}
\]
which leads to the formula for the eigenvalues $ \mu_{j, \pm}^k = \frac{1}{2} \Big( a_j^k+ d_j^k \pm \sqrt{ (a_j^k-d_j^k)^2 + 4 c_j^kb_j^k } \Big) $. 
%Since $ b_j^kc_j^k>0 $ \rednote{check!!! It may not be....}, the two solutions are always real. 
The relationship between the components of the two vectors $ \bm{w}_\ell^j $ is then $  \xi^2_j = -\frac{a_j^k- \mu_{j, \pm}^k}{b_j^k} \xi^1_j$.
Notice that for each $ \bm{q}^j $ there are two eigenvalues $ \mu_{j, \pm}^k$, for a total of $ 2 (d-1)$ eigenvalues in this class.

\item 
The third class contains the $ nd - n - 2d+2 $ remaining eigenvectors subdivided into two groups: $ \bm{r}_1 = \begin{bmatrix} \epsilon_1 \bv_h^T & \ldots &  \epsilon_{n_1} \bv_h^T & 0 & \ldots & 0 \end{bmatrix}^T$ and $ \bm{r}_2 = \begin{bmatrix} 0 & \ldots & 0  & \varepsilon_1 \bv_h^T & \ldots & \varepsilon_{n_2}  \bv_h^T \end{bmatrix}^T$ where $ \bv_h $ is an eigenvector of $\bV$, $h \neq k $, and $ \sum_{j=1}^{n_1} \epsilon_j =0$, $ \sum_{j=1}^{n_2} \varepsilon_j =0$.
By construction, $ \bv_k^T \bv_h =0 $ and $ (\bm{w}_\ell^j)^T \bv_h=0 $ for all  $ j=1, \ldots, k-1, k+1, \ldots, d $ and $ \ell =1,2$. 
Computing: 
\[
J(\pm \bv_k ) \bm{r}_1
= \begin{bmatrix}  
 \frac{\eta \frac{\alpha_1^k}{\beta_1^k} \lambda_h  \left(\sum_{j=1}^{n_1} \epsilon_j \right) \bv_h +  \epsilon_1 \bv_h}{ 1+ \eta \lambda_k \gamma_1^k} \\ 
\vdots \\ 
 \frac{\eta \frac{\alpha_1^k}{\beta_1^k} \lambda_h  \left(\sum_{j=1}^{n_1} \epsilon_j \right) \bv_h + \epsilon_{n_1} \bv_h}{ 1+ \eta \lambda_k \gamma_1^k} \\ 
\frac{\eta \frac{\alpha_2^k}{\beta_2^k} \lambda_h   \left(\sum_{j=1}^{n_1} \epsilon_j \right) \bv_h }{ 1+ \eta \lambda_k \gamma_2^k}  \\ 
\vdots \\
 \frac{\eta \frac{\alpha_2^k}{\beta_2^k} \lambda_h  \left(\sum_{j=1}^{n_1} \epsilon_j \right) \bv_h }{ 1+ \eta \lambda_k \gamma_2^k} 
 \end{bmatrix} = \frac{1}{1+ \eta \lambda_k \gamma_1^k}  \bm{r}_1
 \]  
and, similarly, $ J(\pm \bv_k ) \bm{r}_2 = \frac{1}{1+ \eta \lambda_k \gamma_2^k} \bm{r}_2$. 
\eenu
\qed

\begin{remark}
As for the multiagent Oja system, in Lemma~\ref{lem:lin-stab-self2} the first class of eigenvectors is ``extrinsic'' and the other two ``intrinsic'', i.e., they belong to the tangent space $ T_{\bx} (\mathbb{S}^{d-1})^n $. Only the last two classes count when deciding the local asumptotic stability properties of the bipartite equilibria. 
\end{remark}

\begin{lemma}
\label{lem:polyg-self}
The polygonal equilibria are all unstable for \eqref{eq:update}.
\end{lemma}
\proof
The idea of the proof is similar to that used in Lemma~\ref{lem:moja-polygonal-stab} for the multiagent Oja flow. 
A polygonal equilibrium point $ \bx =\bm{s} =\begin{bmatrix} \bm{s}_1 & \ldots & \bm{s}_n \end{bmatrix} \in (\mathbb{S}^{d-1})^n$ is s.t. $ \bV  \sum_{j=1}^n A_{ij}(\bm{s}) \bm{s}_j  =0$. Let us compute the linearization of \eqref{eq:update} at $\bm{s}$, obtained perturbing $ \bm{s}$ with a perturbation $ \bw = [ \bw_1 \ldots \bw_n] $ with  $  \bs_i^T \bw_i =0$.  For the non-normalized next iterate it is
\[
\begin{split}
\tilde{\bx}_i = & \bs_i+\bw_i + \frac{\eta}{n} \bV\sum_j A_{ij}(\bs + \bw) (\bs_j + \bw_j)  \\
= & \bs_i + \eta \underbrace{\bV\sum_j A(\bs) \bs_j}_{=0} + \bw_i + \eta  \bV \sum_j A(\bs) \bw_j +  \eta  \bV \sum_j \sum_h \pde{A_{ij}(\bs)}{\bw_h} \bw_h \bs_j + h.o.t.
\end{split}
\]
where the vectors $ \pde{A_{ij}(\bm{s})}{\bw_h} $ are computed in \eqref{eq:partial-A}.
Denote $ \psi^{ij} (\bw) := \sum_h\pde{A_{ij}(\bm{s})}{\bw_h}  \bw_h$ and observe that $ \psi^{ij} (\bw) $ is a sum of bilinear forms $ \bs_\ell B_k \bw_h $ for some matrices $ B_k $ (see \eqref{eq:partial-A} for their specific expressions). 
Since $ \| \bm{s}_\ell \|=1$ and $  \bw_h $ is small enough, from the matrix Cauchy-Schwartz inequality, for each of these bilinear forms it holds $ - \| B_k\|_2 \leq \bm{s}_\ell B_k \bw_h \leq \| B_k \|_2 $, where $ \|B_k \|_2 $ is  independent of $ \bw_h$.

Denoting $ \bu_i := \bw_i + \eta \bV  \sum_j \left( \psi_{ij}(\bw) \bs_j + A_{ij}(\bs) \bw_j \right) $, the norm of $ \tilde{\bx}_i $ is
\[
\| \tilde{\bx}_i \| = \left( 1+ \| \bu_i \|^2 + 2 \eta \bs_i^T \bV \sum_j \left( \psi_{ij}(\bw) \bs_j + A_{ij}(\bs) \bw_j \right)  \right)^{\frac{1}{2}}
\]
where we have used $ \bs_i^T \bw_i =0 $.
Observe that $ \| \bu_i \|^2 $ is composed entirely of terms of order 2 or higher in $ \bw_i $, hence, using the binomial series expansion ($ (1 + \xi)^{-\frac{1}{2}}\approx 1 - \frac{\xi}{2}  + h.o.t. $), 
\[
\begin{split}
\bx_i^+ &= \frac{\tilde{\bx}_i}{\| \tilde{\bx}_i \|}
\approx (\bs_i + \bu_i) \left(1- \frac{\| \bu_i \|^2}{2} - \eta \bs_i^T \bV \sum_j \left( \psi_{ij}(\bw) \bs_j + A_{ij}(\bs) \bw_j \right) \right)  + h.o.t. \\
&\approx  \bs_i - \eta \bs_i^T \bV \sum_j \left( \psi_{ij}(\bw) \bs_j + A_{ij}(\bs) \bw_j \right) \bs_i  + \bu_i   + h.o.t.\\
& = \bs_i - \eta  \bs_i  \bs_i^T \bV \sum_j \left( \psi_{ij}(\bw) \bs_j + A_{ij}(\bs) \bw_j \right)+ \bw_i + \eta \bV \sum_j \left( \psi_{ij}(\bw) \bs_j + A_{ij}(\bs) \bw_j \right) + h.o.t.\\
& = \bs_i + \bw_i +\eta (I - \bs_i  \bs_i^T ) \bV \sum_j \left( \psi_{ij}(\bw) \bs_j + A_{ij}(\bs) \bw_j \right).
\end{split}
\]
The linearized equation is then 
\[
\bw_i^+ = \bw_i +\eta (I - \bs_i  \bs_i^T ) \bV \sum_j \left( \psi_{ij}(\bw) \bs_j + A_{ij}(\bs) \bw_j \right).
\]
Expanding all $ \bw_i $ and $ \bs_i $ in the basis of eigenvectors $ \bv_1, \ldots, \bv_d $, and denoting $ \bw_i = \sum_\ell \xi^i_\ell \bv_\ell $ and $ \bs_i = \sum_\ell \zeta^i_\ell \bv_\ell $, 
\[
\bw_i^+ = \sum_\ell (\xi^i_\ell)^+ \bv_\ell =  \sum_\ell \xi^i_\ell \bv_\ell +\eta \Big( I - \sum_k \zeta^i_k \bv_k \sum_\ell \zeta^i_\ell \bv_\ell^T \Big)  \sum_{j, \ell} \lambda_\ell (\psi_{ij}(\bw) \zeta^j_\ell + A_{ij}(\bs) \xi^j_\ell )  \bv_\ell .
\]

Assuming that the perturbation $ \bw$ is aligned with $ \bv_1$, i.e., for all $i$, $ \xi^i_1 \neq 0 $ and $ \xi^i_k=0$ for $ k=2, \ldots, d$, then, projecting along $ \bv_1$, yields
\[
{\xi^i_1}^+ =\xi^i_1 + \eta \left( 1- (\zeta^i_1)^2 \right) \lambda_1 \Big( \sum_j \psi_{ij}(\bw) \zeta^j_1 +  \sum_j A_{ij}(\bm{s}) \xi^j_1 \Big) .
\]
Denoting  $ \bm{\xi}_1 = \begin{bmatrix} \xi^1_1 & \ldots & \xi^n_1 \end{bmatrix}^T$
and $  \bm{\zeta}_1 = \begin{bmatrix} \zeta^1_1 & \ldots & \zeta^n_1 \end{bmatrix}^T$ the collection of the $ \xi_1^i $ and $ \zeta_1^i $ components of all agents, the previous ODEs can be expressed in vector form as
\beq
{\bm{\xi}_1}^+ =\bm{\xi}_1 + \eta  \lambda_1 \left( I- Z_1^2 \right) \left( \Psi \bm{\zeta}_1 + A(\bm{s} ) \bm{\xi}_1 \right)
\label{eq:ode-eta-self}
\eeq
where $ Z_1 = \diag(\bm{\zeta}_1)$, $ \Psi =[\psi_{ij}]$ is a matrix with lower and upper bounds independent of $ \bw$, and $ \bm{\zeta}_1 $ is fixed. $ A(\bm{s} ) $ is a row stochastic matrix,  $ \rho(A(\bm{s} )) =1 $ and $ A(\bm{s} ) \mathds{1} = \mathds{1}$. 
Since $ |\zeta^i_1 |< 1$ because  $ \bm{s} $ is not an eigenvector of $\bV$ and $ \lambda_1 >0$, it is $\rho(I +\eta \lambda_1(I-Z_1^2) A(\bs))>1 $, and the system \eqref{eq:ode-eta-self} diverges when $ \bm{\xi}_1 = \epsilon \mathds{1}$ for some scalar $ \epsilon$. Hence the polygonal equilibrium $ \bm{s}$ is unstable. 
\qed

\begin{remark}
As for the multiagent Oja system, the perturbation used in Lemma~\ref{lem:polyg-self} is``intrinsic'', i.e.,  $ \bw_i \in T_{\bs_i} \mathbb{S}^{d-1} $ for all $i$. 
\end{remark}

\noindent {\bf Proof of Theorem~\ref{thm:stab-main-self}.}
The proof is the direct combination of Lemmas~\ref{lem:lin-stab-self1},~\ref{lem:lin-stab-self2} and~\ref{lem:polyg-self}, with the only observation that in the second class of eigenvalues of the bipartite consensus case of Lemma~\ref{lem:lin-stab-self2}, using the Schur-Cohen formula for stability of a $ 2 \times 2 $ matrix in discrete time, the condition $ | \mu_{j, \pm}^k | <1$ becomes $  a_j^k d_j^k - b_j^k c_j^k < 1 $ and $ a_j^k + d_j^k <1+ a_j^k d_j^k - b_j^k c_j^k $, $\, \forall \, j=1, \ldots, k-1, k+1,\ldots,  d $.
\qed

\begin{remark}
\label{rem:bip-in-practice-singlehead}
Similarly to the continuous-time case analyzed e.g. in \cite{altafini2025multistability}, bipartite consensus equilibria can be locally stable for the single-head discrete-time transformer in \eqref{eq:update}. A simulation example is shown in Fig.~\ref{fig:3Dexamples-bip}.
Notice that under the assumption of Theorem~\ref{thm:stab-main-self}, from the third condition of the theorem ($ \frac{1}{1+ \eta \lambda_k \gamma_i^k} <1$), it is straightforward to deduce that if $ \lambda_k <0 $ is an eigenvalue of $ \bV $, then the associated bipartite equilibria $ \bx_i \pm \bv_k $ are unstable, because it is $ 1+ \eta \lambda_k \gamma_i^k <1 $ for all admissible values of $ \eta< \min\left( \frac{1}{\lambda_1 } , \, \frac{1}{ | \lambda_d |} \right) $.
\end{remark}

\subsection{Extensions}
\subsubsection{Multi-head case}
\label{sec:ext-multi-head}
The multi-head case corresponds to the time-invariant dynamical system \eqref{eq:update-multihead-auton} of the paper, repeated here for convenience:
\beq
\label{eq:update-multihead-auton-bis}
\bx_i(t+1) = \frac{\bx_i (t) +\eta F_{O} \sum_{h=1}^H \bar{F}_{V}^{(h)}  \sum_j A_{ij}^{(h)}(\bx(t)) \bx_j (t)  }{\| \bx_i (t) +\eta F_{O} \sum_{h=1}^H \bar{F}_{V}^{(h)}  \sum_j A_{ij}^{(h)}(\bx(t)) \bx_j  (t) \|}, \qquad \begin{array}{l} i=1, \ldots, n, \\ t=1, \ldots, T\end{array}
\end{equation}
Recall that in this case what matters is the ``layer matrix'' $ F= F_{O} \sum_{h=1}^H \bar{F}_{V}^{(h)} $. 
For the system \eqref{eq:update-multihead-auton-bis}, in the main paper we stated Theorem~\ref{thm:symm-case1}, which we can now prove as an easy consequence of Theorem~\ref{thm:stab-main-self}.
% \begin{theorem}
% \label{thm:symm-case1-bis}
% Consider the system~\eqref{eq:update-multihead-auton}. 
% Assume the attention matrices $ A^{(h)}$, $ h=1, \ldots, H$, are full and the matrix $ F $ is symmetric with $ \lambda_1 > \lambda_i$, $ i=2,\ldots, d$. Assume furhter that  $ \eta<  \frac{2}{ |\lambda_1 + \lambda_d |}  $ whenever $ - \lambda_d > \lambda_1 $. Then the consensus point $ \bx_i =\bv_1 $ $\forall \, i=1, \ldots, n$ (or $ \bx_i =-\bv_1 $ $\forall \, i=1, \ldots, n$) is a locally asymptotically stable equilibrium point for the system~\eqref{eq:update-multihead-auton} when $ T\to \infty$. 
% \end{theorem}

\noindent {\bf Proof of Theorem~\ref{thm:symm-case1} (in the paper).} Under the assumption of the theorem, the system~\eqref{eq:update-multihead-auton-bis} has the consensus points at the eigenvectors of $F$ as equilibria, as can be verified by an argument similar to Lemma~\ref{lem:self-equil1}. In particular, when $ \bx_i$ approaches a consensus point, the attention matrices of \eqref{eq:update-multihead-auton-bis} still approach a uniform matrix: $ A^{(h)}(\bx) \to \frac{1}{n}$, $ h=1, \ldots, H$. Therefore we can use the same argument used in Theorem~\ref{thm:stab-main-self}: at consensus, the system \eqref{eq:update-multihead-auton-bis} has the Jacobian $ J(\bv_k) $ of the multiagent Oja system \eqref{eq:update-moja}, whose eigenvalues were computed in Lemma~\ref{lem:moja-F-eigenvalues}. Consensus aligned with the principal eigenvector $ \bv_1 $ is locally asymptotically stable, while all other consensus points at the other eigenvectors of $F$ are unstable. 
\qed

While the picture for consensus is easily obtainable from that of the single-head case, what happens to the bipartite consensus points is less clear. 
We currently have no (analytical nor numerical) evidence of such points being stable equilibria for \eqref{eq:update-multihead-auton-bis}.

\subsubsection{Longer step size}
To check stability of the consensus equilibria a condition on the step size is needed only when $ \lambda_d < 0$, and  $ | \lambda_d |> \lambda_1 $, i.e., when the spectral radius of $F$ (or $ \bV$ for the single-head case) is $ | \lambda_d | $, not $ \lambda_1 $. 
Extra conditions on $ \eta $ are required only for checking the stability of the bipartite consensus points. 
These extra conditions are only sufficient, and meant to simplify the expression of the eigenvalues of the Jacobian $ J(\pm \bv_k)$, see Lemma~\ref{lemm:simplif-eta}.
As mentioned above, locally asymptotically stable bipartite consensus equilibria appear to be rare, even in the single-head case.

\subsubsection{Asymmetric value matrix}
As mentioned in Section~\ref{sec:shared-param} of the paper, what really matters to capture the asymptotic behavior of the self-attention dynamics (and to establish the analogy with the power method) is that the dominant eigenvalue of $F$, $ \lambda_1 $, is real, while $ \lambda_2, \ldots, \lambda_d $ can be any.
This together with the strict dominance of $ \lambda_1 $ is the essence of the Perron-Frobenius theorem extended beyond nonnegative matrices. 
In fact, when the dominance is strict (and the spectral radius is equal to $ \lambda_1 $) not much changes in the behavior of the single-head system~\eqref{eq:update} and neither of its multi-head version~\eqref{eq:update-multihead-auton} by considering an $ F$ which is not symmetric.
A numerical analysis of the asymmetric $F$ case is carried out in the main paper when dealing with the ALBERT transformer.  

\subsection{Small-scale numerical examples}
\label{sec:extra-examples}
In Fig.~\ref{fig:3Dexamples} we show a numerical simulation of the dynamics of the single-head self-attention system \eqref{eq:update} and of its multi-head version \eqref{eq:update-multihead-auton-bis} when $ d=3 $ and $ n=10$.
In this case $ \bx_i\in \mathbb{S}^2 $, i.e., we can represent trajectories of the self-attention dynamics on a sphere in $ \mathbb{R}^3$. 
Both trajectories in Fig.~\ref{fig:3Dexamples} converge to consensus, aligned with the principal eigenvector $ \bv_1 $ of $ F_V $ and $ F $ respectively, see Fig.~\ref{fig:3Dexamples}. 
The dynamical behavior of the two systems is nearly identical and very similar to that of the ALBERT transformer.

In Fig.~\ref{fig:3Dexamples-bip}, for the single-head case,  we show instead an example of a bipartite consensus equilibrium which is locally asymptotically stable, as checked analytically using Theorem~\ref{thm:stab-main-self}. In Fig.~\ref{fig:3Dexamples-bip}(a), a small perturbation added to a bipartite equilibrium point aligned with $ \bv_1 $ is reabsorbed. However, the basin of attraction is very small: as soon as a slightly larger perturbation is applied (Fig.~\ref{fig:3Dexamples-bip}(b)), the tokens on one of the two antipodal points migrate to the other point, and consensus is achieved (still aligned with $ \bv_1 $, as expected from Theorem~\ref{thm:stab-main-self}. 
The (very) small basin of attraction of the bipartite stable equilibria appears to be a constant characteristic of these equilibria in our simulations. Analytical calculations of the basin of attraction are prohibitive for this system.

\begin{figure*}[ht]
     \centering    
      \subfigure[single head]{
        \includegraphics[trim=2cm 0.5cm 2cm 1cm, clip=true,width=0.4\textwidth]{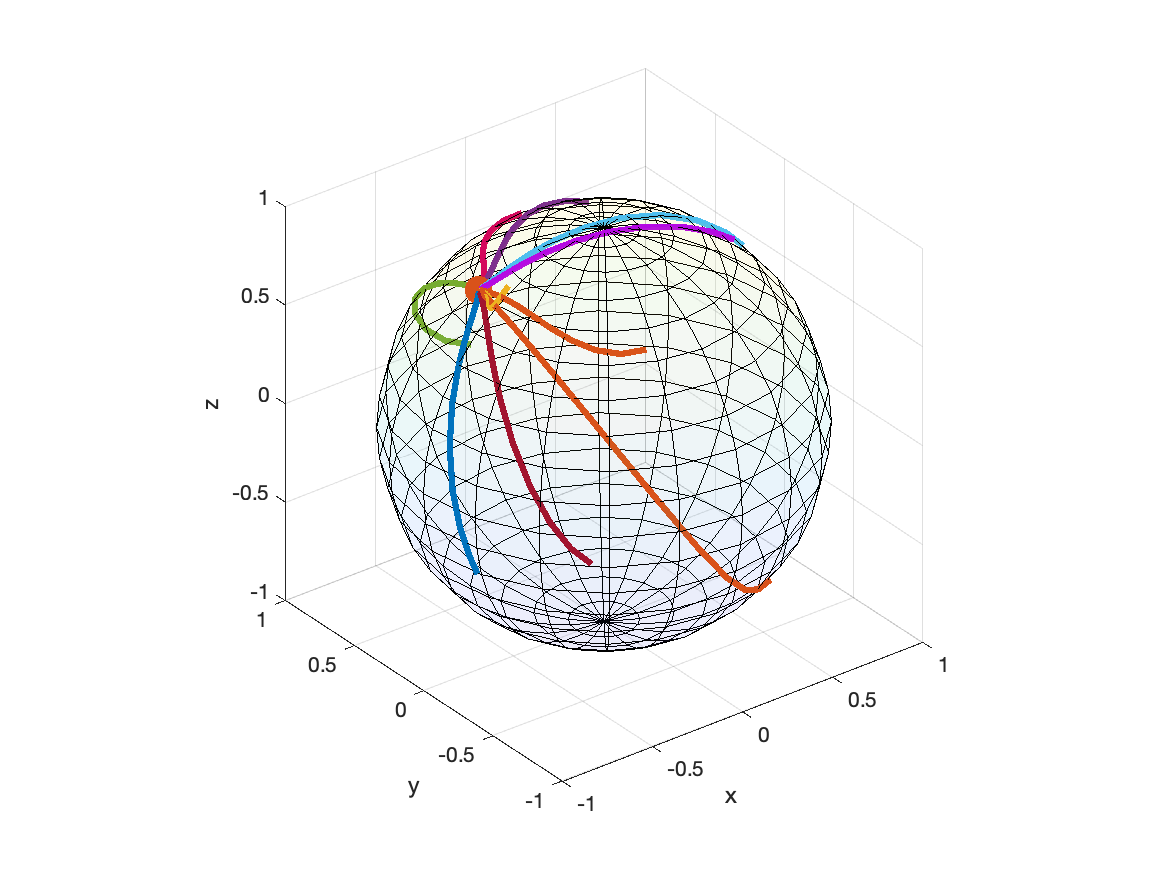}
        \includegraphics[trim=0cm 0cm 0cm 0cm, clip=true,width=0.4\textwidth]{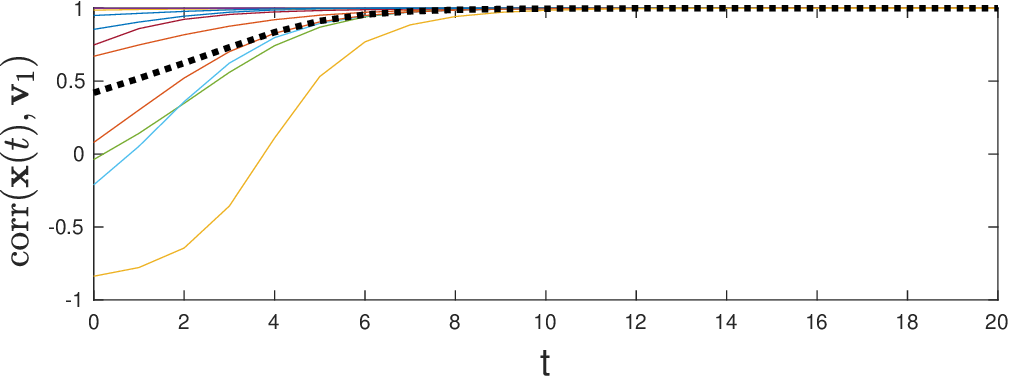}}
   \subfigure[multihead]{
        \includegraphics[trim=2cm 0.5cm 2cm 1cm, clip=true,width=0.4\textwidth]{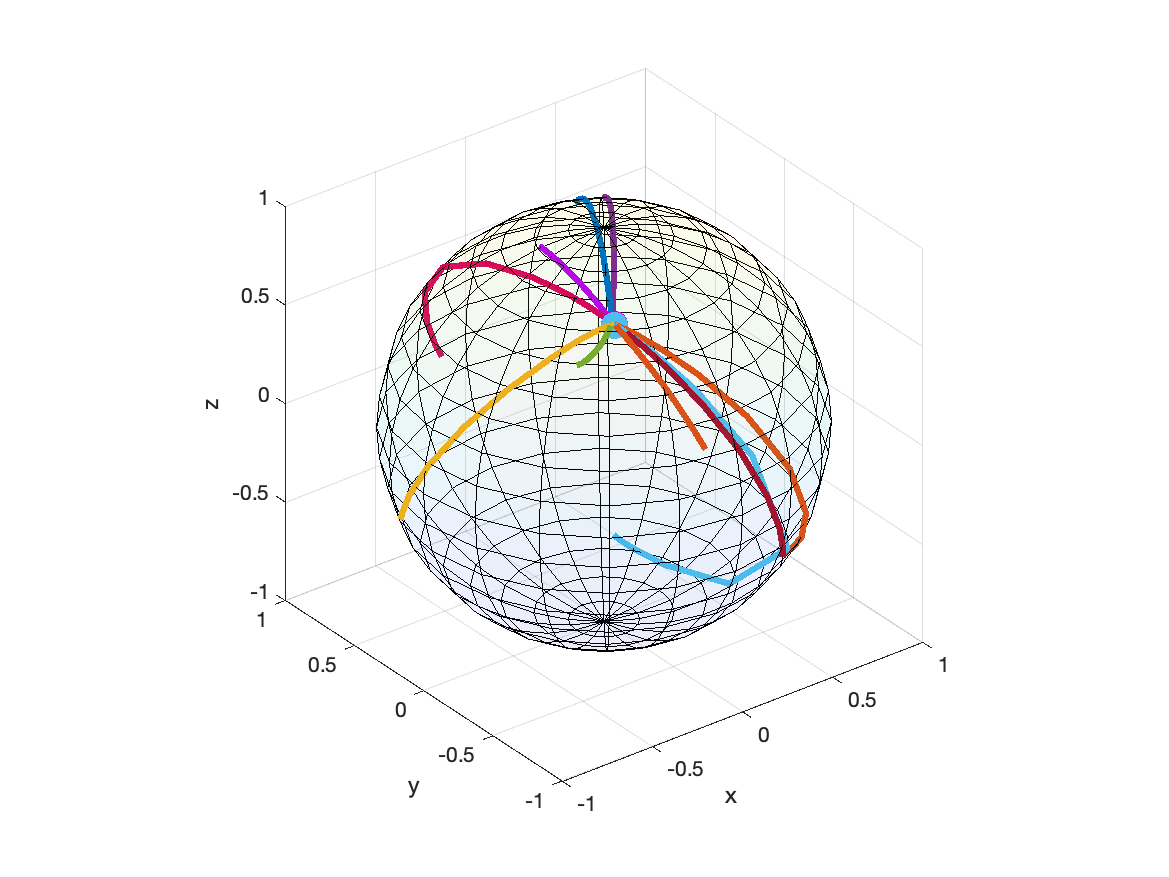}
      \includegraphics[trim=0cm 0cm 0cm 0cm, clip=true,width=0.4\textwidth]{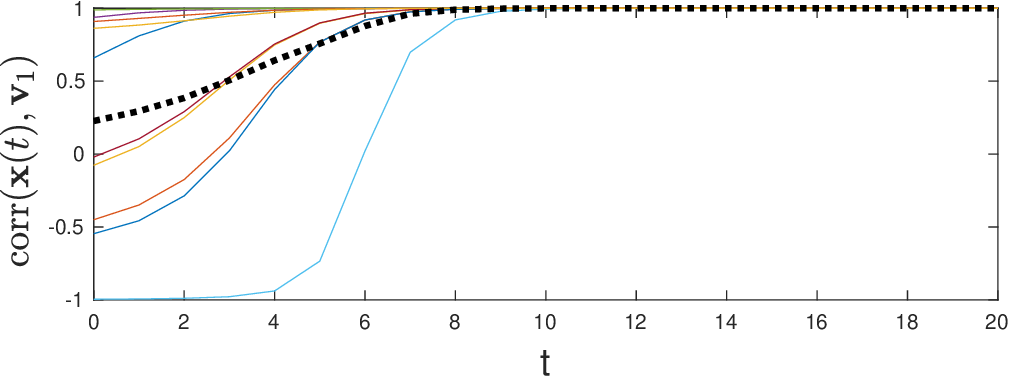}}
%   \subfigure[]{
%        \includegraphics[trim=0cm 0cm 0cm 0cm, clip=true,width=0.4\textwidth]{Figures/example3D_limitcycle1a.eps}}
        \caption{Example with $ d=3$ and $ n=10$. (a): single-head self-attention dynamics; (b): multi-head self attention dynamics (with 5 heads). Left panel: the trajectory of the 10 tokens is shown on the sphere $ \mathbb{S}^2 $. The solid dot is the endpoint of a trajectory, i.e., the consensus point reached by all tokens. Right panel: the correlation $ {\rm corr}(\bx_i, \bv_1) $ tends to 1 as consensus is achieved (all tokens become aligned with the principal eigenvector $ \bv_1 $). The black dotted line is the mean correlation over the 10 tokens.}
        \label{fig:3Dexamples}
\end{figure*}

\begin{figure*}[ht]
     \centering    
   \subfigure[perturbation preserves bipartition]{
        \includegraphics[trim=2cm 0.5cm 2cm 1cm, clip=true,width=0.4\textwidth]{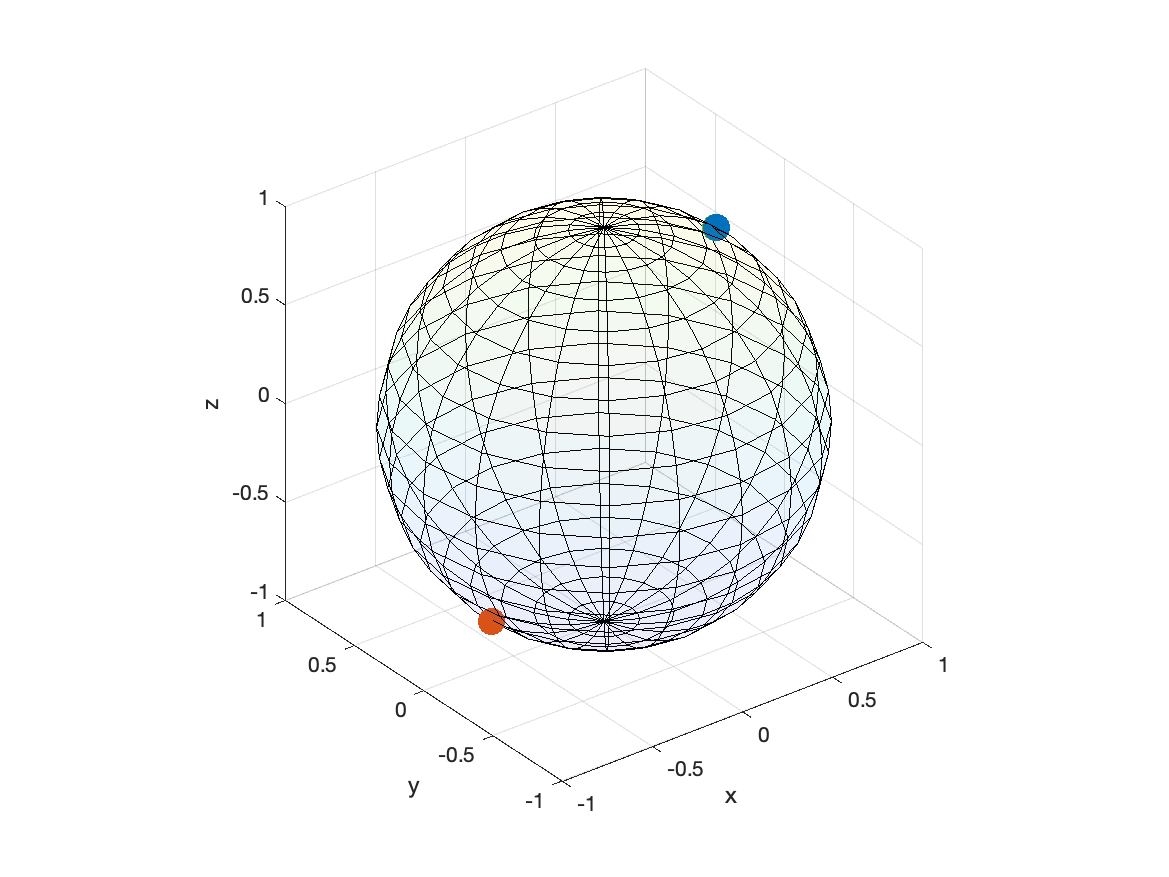}
      \includegraphics[trim=0cm 0cm 0cm 0cm, clip=true,width=0.4\textwidth]{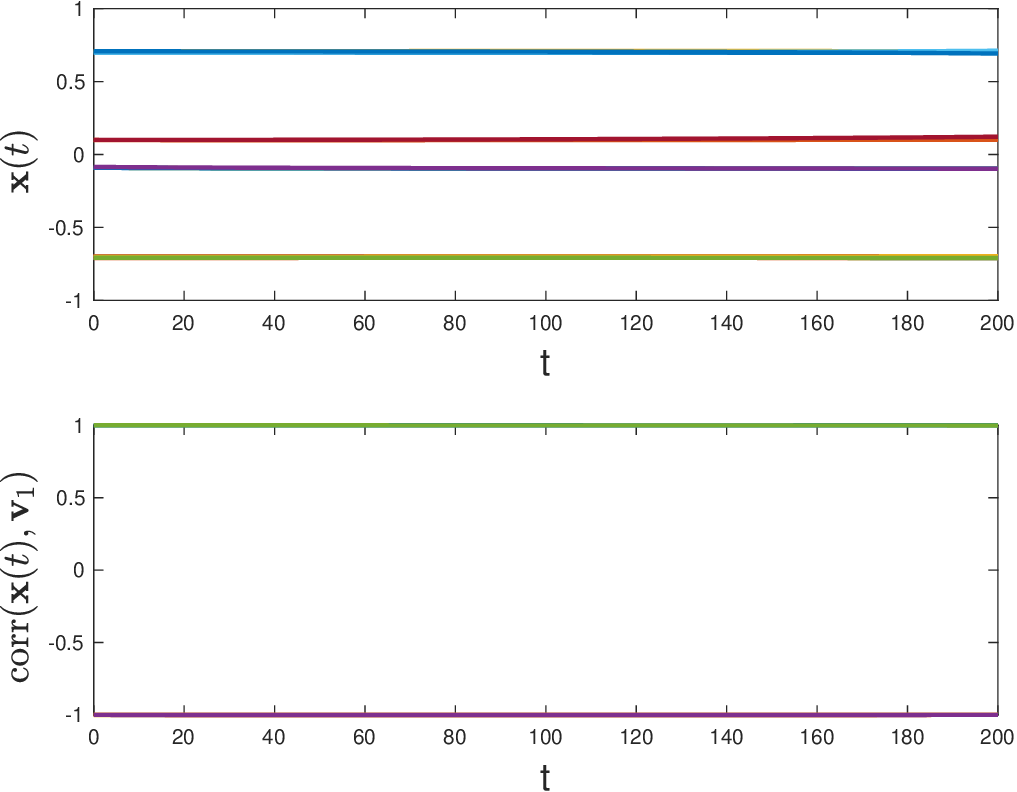}}
      \subfigure[perturbation leads to consensus]{
        \includegraphics[trim=2cm 0.5cm 2cm 1cm, clip=true,width=0.4\textwidth]{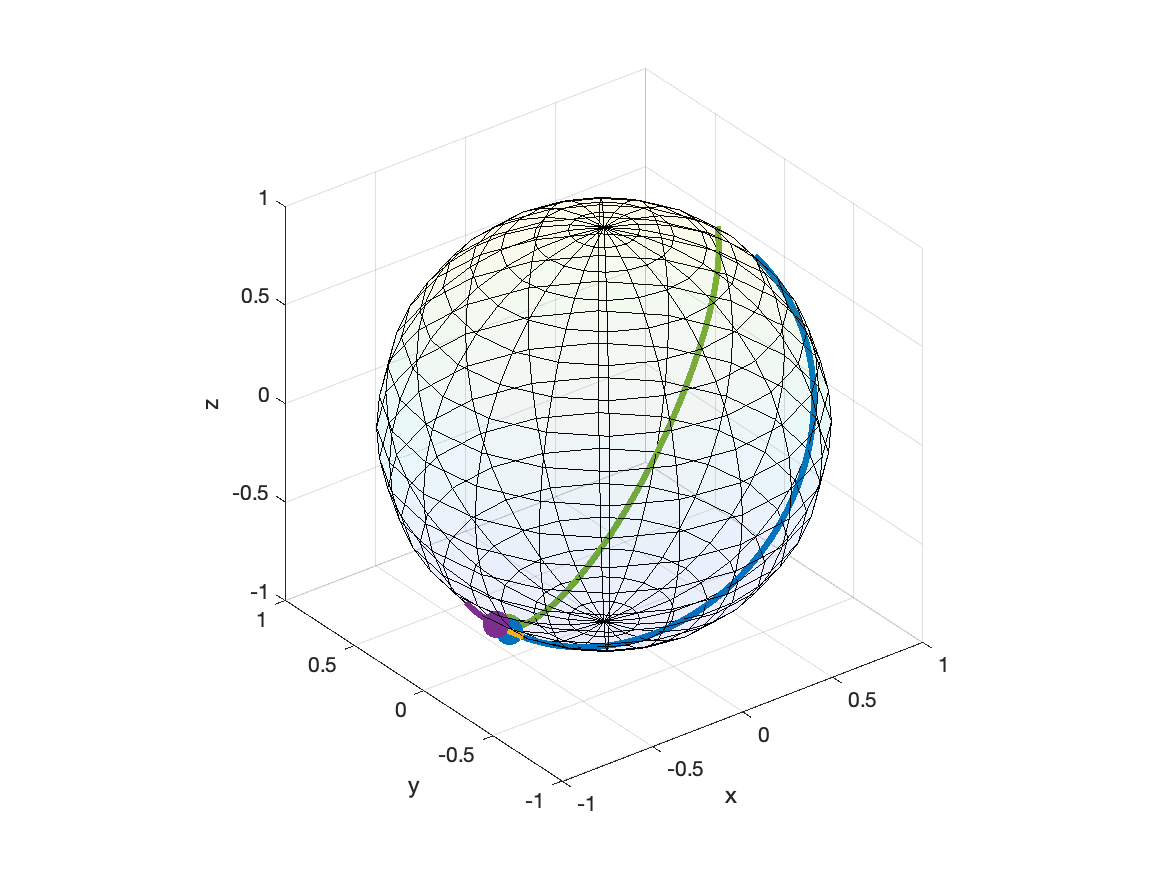}
        \includegraphics[trim=0cm 0cm 0cm 0cm, clip=true,width=0.4\textwidth]{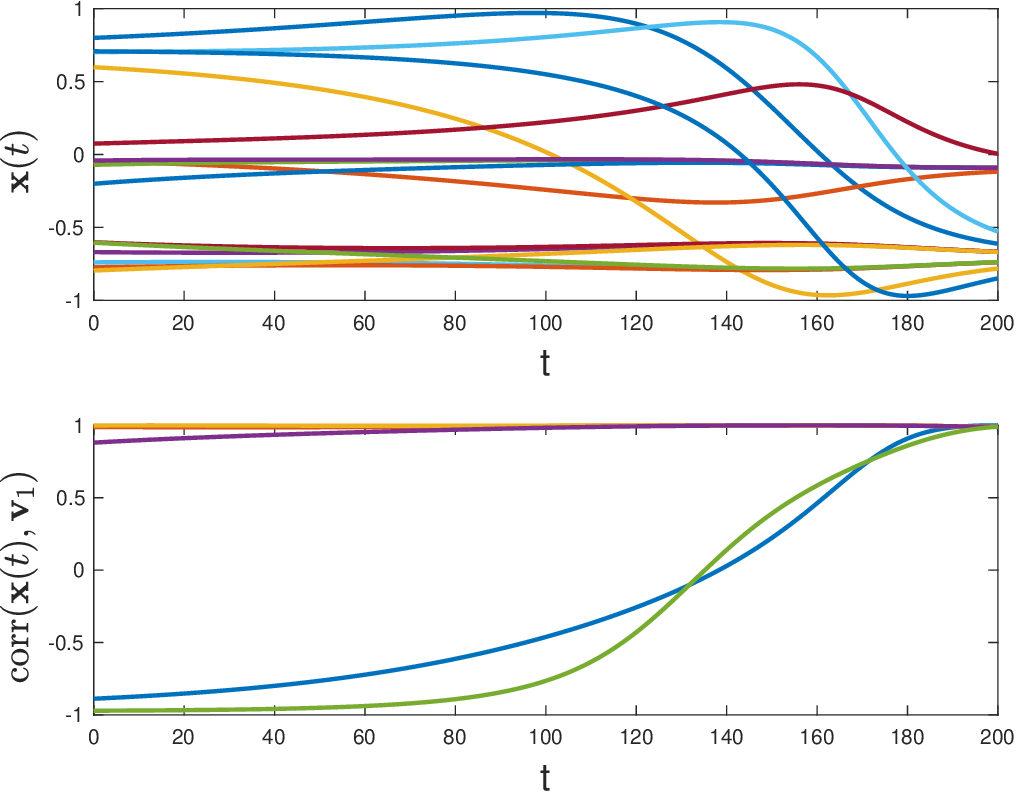}}
%   \subfigure[]{
%        \includegraphics[trim=0cm 0cm 0cm 0cm, clip=true,width=0.4\textwidth]{Figures/example3D_limitcycle1a.eps}}
        \caption{Example of bipartite consensus with $ d=3$ and $ n=10$. (a): when a very small perturbation is applied to a locally asymptotically stable bipartite consensus point (aligned with $ \bv_1 $), the system returns to the bipartite equilibrium. (b): when a larger perturbation is applied to the same bipartite consensus point, the trajectory escape towards one of the two the consensus points associated to the principal eigenvector $ \bv_1 $. The two tokens in blue and green ``migrate'' to the antipodal point, thereby achieving consensus.
        }
        \label{fig:3Dexamples-bip}
\end{figure*}

\clearpage

\end{document}